\documentclass{article}

    \PassOptionsToPackage{numbers, compress}{natbib}

 \usepackage[preprint]{neurips_2026}

\usepackage{float}
\usepackage{subcaption}
\usepackage{graphicx}
\usepackage{enumitem}
\usepackage[colorinlistoftodos]{todonotes}
\usepackage{algpseudocode}
\usepackage{xcolor}
\usepackage{amssymb}
\usepackage{bm}
\usepackage{bigints}
\usepackage{mathtools}
\usepackage{xhfill}

\usepackage{hyperref}

\usepackage{mathtools}

\usepackage{amsmath, amssymb, amsthm}

\newtheorem{proposition}{Proposition}

\newcommand{\R}{\mathbb{R}}

\newcommand{\NN}{\mathcal{N}}

\renewcommand{\P}{\operatorname{\mathbb{P}}}
\newcommand{\E}{\operatorname{\mathbb{E}}}
\newcommand{\var}{\operatorname{var}}

\newcommand{\Tr}{\text{Tr}}

\newcommand{\e}{e}

\newcommand{\vct}[1]{\boldsymbol{#1}}
\newcommand{\mtx}[1]{\boldsymbol{#1}}

\newcommand{\bmtx}[1]{\mathbf{#1}}

\DeclareMathOperator*{\argmin}{\text{arg~min}}
\DeclareMathOperator*{\argmax}{\text{arg~max}}

\usepackage{xhfill}

\newcommand{\bI}{\bmtx{I}}

\newcommand{\va}{\vct{a}}
\newcommand{\vb}{\vct{b}}

\newcommand{\vm}{\vct{m}}

\newcommand{\vp}{\vct{p}}
\newcommand{\vq}{\vct{q}}

\newcommand{\vu}{\vct{u}}
\newcommand{\vv}{\vct{v}}
\newcommand{\vw}{\vct{w}}
\newcommand{\vx}{\vct{x}}

\newcommand{\vz}{\vct{z}}

\newcommand{\vmu}{\vct{\mu}}

\newcommand{\vzero}{\vct{0}}

\newcommand{\mH}{\mtx{H}}

\newcommand{\mM}{\mtx{M}}

\newcommand{\mSigma}{\mtx{\Sigma}}

\newcommand{\mzero}{{\bf 0}}

\newcommand{\gradp}{\nabla_{\vp}}
\newcommand{\gradq}{\nabla_{\vq}}

\newcommand{\hessp}{\nabla^2_{\vp}}
\newcommand{\hessq}{\nabla^2_{\vq}}

\usepackage[capitalize,noabbrev]{cleveref}

\title{Active Query Synthesis for Preference Learning}

\author{
  Namrata Nadagouda \\
  Georgia Institute of Technology \\
  \And
  Nauman Ahad \\
  Gauss Labs \\
  \AND
  Maegan Tucker \\
  Georgia Institute of Technology \\
  \And
  Mark A. Davenport \\
  Georgia Institute of Technology 
}

\begin{document}
\maketitle

\begin{center}
    \today
\end{center}

\begin{abstract}
Efficient learning of user preferences is crucial for many modern decision making systems but typically requires costly labeled data. Active learning reduces this cost, yet standard methods are computationally expensive due to pool-based evaluation. Further, most methods assume all query feedback is equally reliable, ignoring that pairwise queries between nearly identical or entirely dissimilar items yield ambiguous, low-confidence responses. To address the issue of feedback reliability, we introduce a novel \textbf{confidence aware response model} that explicitly accounts for these ambiguous comparisons. To overcome the computational bottleneck of pool-based evaluation, we propose an active query synthesis framework, \textbf{Info-Synth} that generates optimal queries by maximizing a mutual information-based objective within a continuous space. Moreover, we propose two strategies, \textbf{Pair M-dist} and \textbf{Pair Opt-dist}, that extend \textbf{Info-Synth} to select effective queries even when restricted to finite query pools. We demonstrate our framework's versatility and performance across synthetic preference learning, constrained text summary datasets, and subjective, continuous-space controller gain tuning for a simulated mobile robot.
\end{abstract}

\section{Introduction} \label{sec:intro}
Training machine learning models often requires large labeled datasets that are expensive~\cite{tajbakhsh2021guest, kim2025n} or difficult to obtain~\cite{bellamy2022batched, bi2025comprehensive}. Active learning~\cite{settles2009active} helps mitigate this bottleneck by iteratively selecting the most informative~\cite{biswas2023active, warmuth2001active} samples to label. But the standard pool-based approaches remain computationally expensive because they require large sets of unlabeled candidates. Instead, directly synthesizing informative queries offers a significantly more efficient alternative~\cite{wang2015active}. This approach reduces computational costs and naturally supports optimization over continuous spaces~\cite{sadigh2017active}, potentially improving performance by generating data points that do not currently exist in the dataset.

We present this work through the lens of preference learning~\cite{chu2005preference} which we frame as localizing a user within an item embedding space, using responses to pairwise comparison queries (``Which among A and B do you prefer?”~\cite{david1963method}.) Our goal is to actively select queries to estimate the user's preference with minimal queries. To achieve this, practical querying must balance informativeness with response reliability. As illustrated in Fig.~\ref{fig:query_illustrations}, comparisons between points that are nearly identical (Fig.~\ref{fig:points_too_close}) or entirely dissimilar (Fig.~\ref{fig:points_too_far}) are inherently ambiguous and yield low-confidence responses, in contrast to ideal queries (Fig.~\ref{fig:ideal_query}).  To address this, we introduce a novel \textit{confidence aware response model} that explicitly accounts for this response reliability during query selection.

\begin{figure}[ht]
\vspace{-3mm}
    \hfill
    \begin{minipage}{0.23\linewidth}
        \centering
        \includegraphics[width=\textwidth]{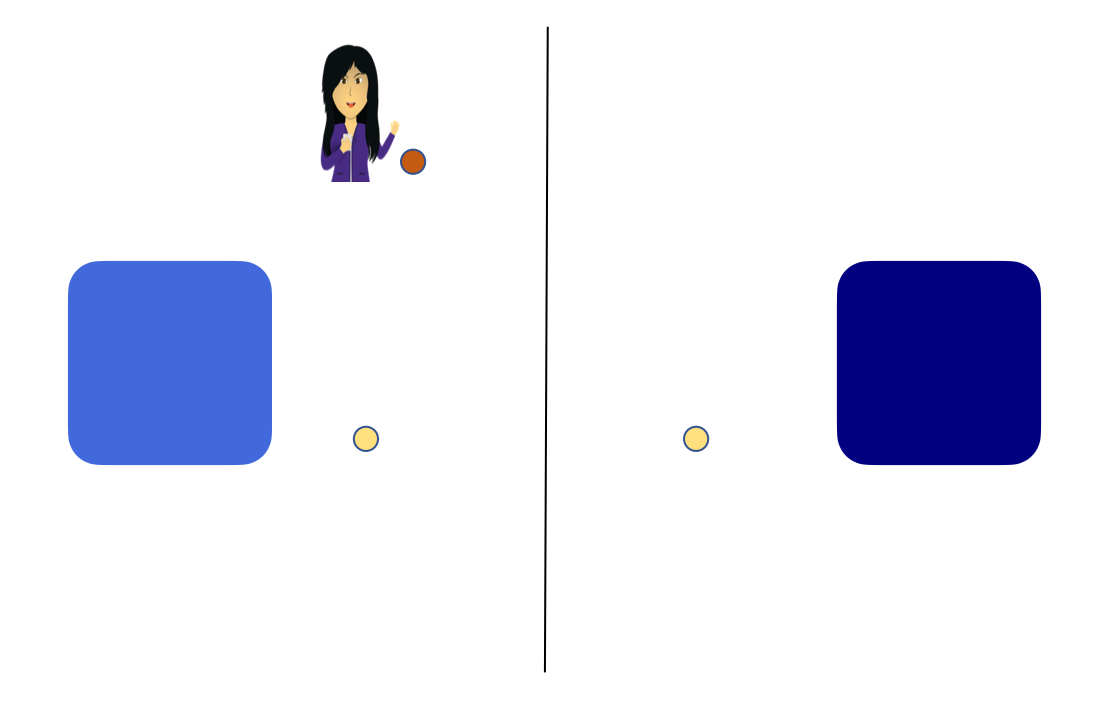}
        \subcaption{\footnotesize Ideal query}
        \label{fig:ideal_query}
    \end{minipage}
    \hfill
    \begin{minipage}{0.23\linewidth}
        \centering
        \includegraphics[width=0.8\textwidth]{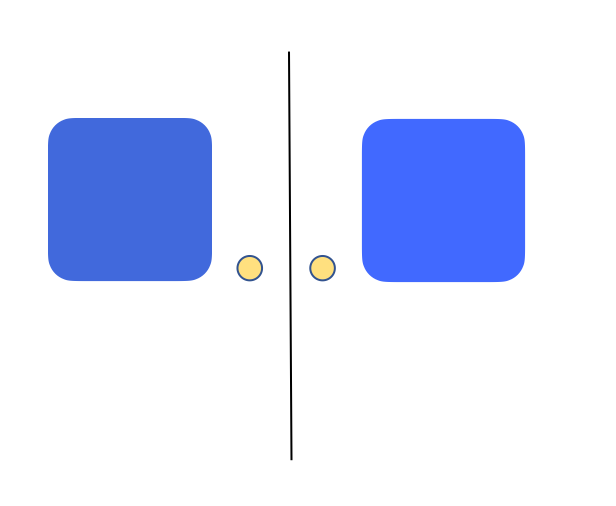}
        \subcaption{\footnotesize Points are too close}
        \label{fig:points_too_close}
    \end{minipage}
    \hfill
    \begin{minipage}{0.33\linewidth}
        \centering
        \includegraphics[width=\textwidth]{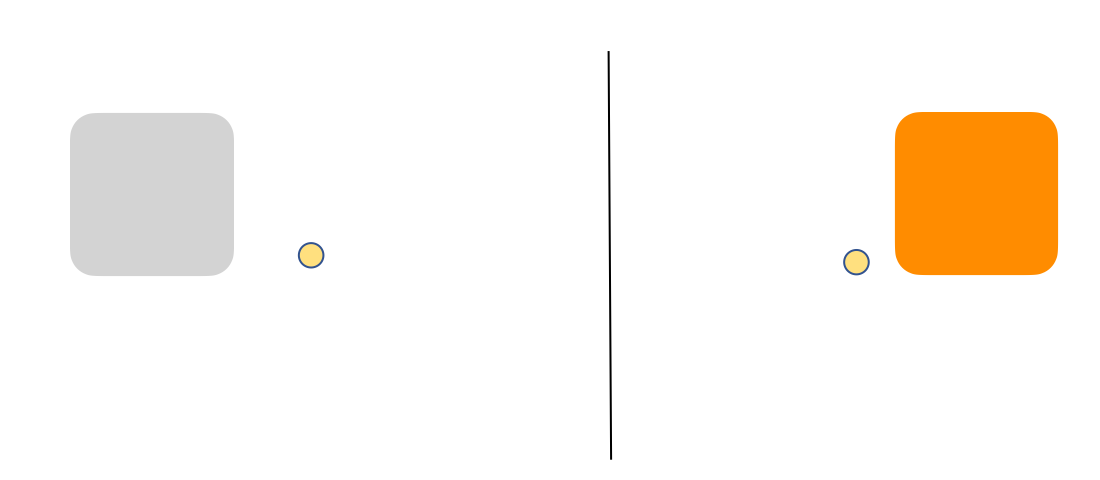}
        \subcaption{\footnotesize Points are too far away}
        \label{fig:points_too_far}
    \end{minipage}
    \hfill
    \vskip 0pt
\caption{\footnotesize Illustrations of pairwise comparison queries based on intra-query distances. (a) An ideal query balances similarity and distinctness, enabling reliable preference selection. Conversely, queries between items that are (b) nearly identical or (c) entirely dissimilar are inherently ambiguous and yield unreliable, low-confidence responses.} 
\label{fig:query_illustrations}
\vspace{-3mm}
\end{figure}

\begin{figure}[ht]
    \centering
    \includegraphics[width=\linewidth]{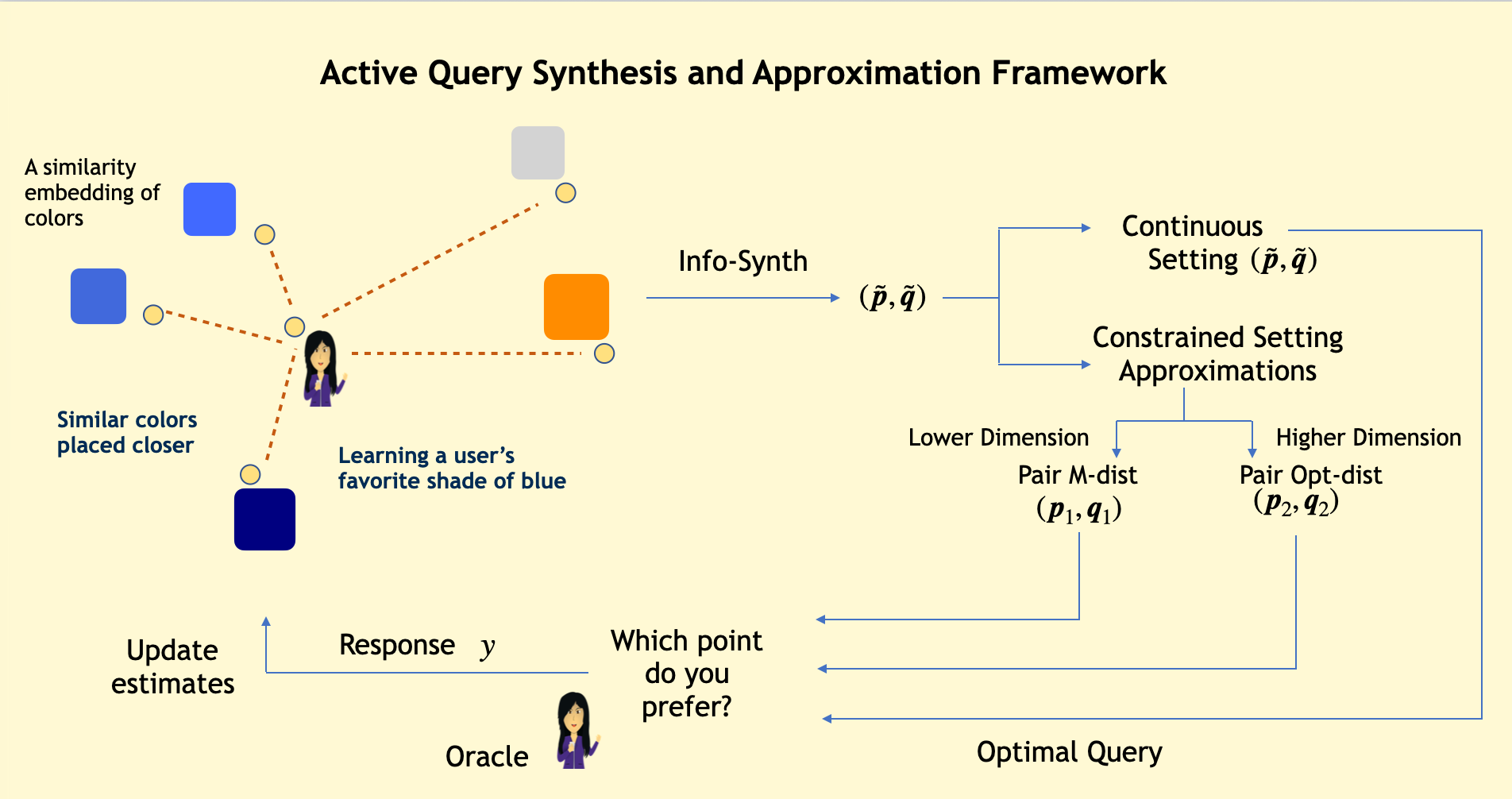}
    \caption{\footnotesize Visualization of the active query synthesis and approximation framework, using a color similarity embedding to estimate a user's preferred shade of blue. \textit{Info-Synth} first generates an optimal continuous query $(\tilde{\vp}, \tilde{\vq})$. In the continuous setting, this query is used directly. In the constrained setting, it is approximated for a fixed dataset using either \textit{Pair M-dist} $(\vp_1, \vq_1)$ or \textit{Pair Opt-dist} $(\vp_2, \vq_2)$, depending on dimensionality. The oracle's response ($y$) to the selected query is then used to update the user preference estimates.}
    \label{fig:framework_fig}
    \vspace{-1\baselineskip}
\end{figure}

We develop an active query synthesis method, \textit{Info-Synth} for the proposed response model. It generates informative queries within a continuous space by maximizing the mutual information between the underlying preference parameters and query responses. As illustrated in Fig.~\ref{fig:framework_fig}, when we can obtain a response directly to the synthesized query $(\tilde{\vp}, \tilde{\vq})$, we refer to it as the \textit{continuous setting}. However, when restricted to a fixed dataset (the \textit{constrained setting} branch), plainly mapping this continuous query to the nearest available points fails to account for anisotropic geometry and uncertainty. To address this, we propose two dataset-specific approximation strategies to find allowable pairs: \textit{Pair M-dist} $(\vp_1, \vq_1)$, which uses a \textit{Mahalanobis} distance metric for lower-dimensional datasets, and \textit{Pair Opt-dist} $(\vp_2, \vq_2)$, which uses an optimality criteria based heuristic for higher-dimensional ones. 

Beyond preference learning, parameter tuning such as adjusting controller gains for a robot, is a distinct and equally critical challenge. In cyber-physical systems, \emph{optimal} performance lacks a rigid objective function and relies heavily on an engineer's subjective intuition to resolve qualitative trade-offs such as deciding whether a trajectory is ‘too jittery’ or
if a robot ‘cuts corners’ too aggressively. Additionally, evaluation of candidates via simulations can also be expensive~\cite{bergstra2012random}. To address these issues, active preference learning has emerged as a potential tool for systematic gain tuning using human feedback~\cite{csomay2022learning}, translating human expertise into a structured parameter search without requiring a predefined reward function. Furthermore, gain tuning is also an ideal application for our continuous query space. Since a robot's behavioral sensitivity to parameter changes is highly nonlinear, sampling directly from a continuous space allows our method to dynamically identify the most informative and perceptible comparisons.

To empirically validate our framework, we evaluate our methods across three distinct settings: (i) preference learning on synthetic data in both continuous and constrained settings (ii) constrained preference learning on a text summary dataset restricted to allowable pairs, and (iii) continuous space controller gain tuning for a simulated mobile robot. To summarize, below are the contributions of this work: 
\begin{enumerate}
    \item A novel \textit{confidence aware response model} for pairwise comparisons that encourages the selection of reliable, high-confidence queries (Sec.~\ref{sec:modified_response_model})
    \item \textit{Info-Synth}, an active query synthesis framework designed for this response model to generate optimal continuous queries (Sec.~\ref{sec:query_synthesis})
    \item \textit{Pair M-dist}, a Mahalanobis distance-based method to approximate synthesized queries within lower-dimensional datasets (Sec.~\ref{sec:query_approximation})
    \item \textit{Pair Opt-dist}, an optimality-driven heuristic for query approximation in higher-dimensional datasets (Sec.~\ref{sec:query_approximation})
    \item Empirical demonstrations of our framework's performance on synthetic data, real-world preference learning, and practical robotic gain tuning tasks (Sec. ~\ref{sec:experiments})
\end{enumerate}

\section{Background} \label{sec:background}
We wish to estimate the preferences of a user $\vw \in \R^d$ over items that may be drawn from a continuous subspace of $\mathbb{R}^d$ or restricted to a fixed dataset $D \subset \mathbb{R}^d$. We use the \textit{ideal point model} \cite{coombs1950psychological} to model the user's preferences and wish to estimate a point $\hat{\vw}$ in $\R^d$ that best explains the preferences. We seek to find a solution by actively asking pairwise comparison queries and updating our estimate based on the responses obtained. We use a Bayesian framework to estimate the user (details in Sec.~\ref{sec:supp-background}). 

Suppose that the response to a query, $Y$ is a binary random variable where $Y = 1$ denotes item $\vp$ is preferred to $\vq$. A popular model based on the \textit{Bradley-Terry} model \cite{bradley1952rank}, used to account for the noisy responses is
\begin{equation*}
    P(Y = 1 | W = \vw) = \Phi \left(k( \|\vw - \vq\|_2^2 - \|\vw - \vp\|_2^2) \right)
\end{equation*}
which can also be written as
\begin{equation*}
    P(Y | \vw) = \Phi \left( k( \va^\top \vw - \tau) \right) = \Phi \left( k \|\va\|( \hat{\va}^\top \vw - \hat{\tau}) \right) \\
\end{equation*}
where \\
$\va = 2(\vp - \vq)$ and $\tau = \|\vp\|^2 - \|\vq\|^2$ represent the normal vector and threshold of a hyperplane $h$ bisecting items $\vp$ and $\vq$ with $h = (\va, \tau): \va^\top \vm - \tau = 0$ where $\vm \in \R^d$,  \\
$\|\va\| = 2\|\vp - \vq\|$, $\hat{\va} = \va / \|\va\|$, $\hat{\tau} = \tau / \|\va\|$, \\
$k$ represents the noise parameter, and \\
$\Phi(x) = 1/(1 + \e^{-x})$ is the logistic function.

A common setting for the noise parameter is using a constant value with $k = k_0$ which results in the model favoring queries that are far from each other. Other settings include a normalized noise with $k = k_0 \|\va\|^{-1}$ where the response probabilities depend only on the hyperplane and not the query distances, a decaying noise with $k = k_0 \exp (-\|\va\|)$ \cite{canal2019active}, and a normalized noise with probit response model \cite{chumbalov2020scalable}, highlighting different assumptions about how user uncertainty depends on query geometry. In practice, as illustrated in Fig.~\ref{fig:query_illustrations} humans typically are most confident in answering pairwise comparisons when the points are located at an ideal distance from each other. When the points are located too close to each other, it might be hard for the user to select one of them and the response can be ambiguous. When the points are too far from each other, they might lose a common ground of comparison and again result in an unreliable response.

\subsection{Other related work}
Related efforts have also explored adding additional feedback modalities beyond binary comparisons, such as coactive feedback, which allows users to iteratively improve proposed solutions \cite{tucker2020preference} and ordinal labels where users assign items to ordered preference categories (e.g., poor to excellent), yielding richer information about relative quality than binary comparisons alone \cite{li2021roial}. These additional feedback modalities have been shown to accelerate preference learning under structured noise assumptions but impose a higher cognitive load for the human providing the feedback. 
The method proposed in \cite{biyik2019asking} employs a specific information gain criterion that balances the uncertainty of the query and the human ability to provide a response to encourage selection of high confidence queries. In this work, we instead propose a modified response probability model that not only doesn't increase the cognitive load for responding but is also amenable to any active learning algorithm.

The application of preference based adaptive learning for tuning parameters is commonly used when it is difficult to obtain a quantitative objective function value for the desired performance. Many works have demonstrated that Bayesian optimization can efficiently learn system parameters using only human preference feedback over presented alternatives~\cite{brochu2010bayesian, gonzalez2017preferential}. In robotics, similar ideas have been applied to learn reward functions~\cite{sadigh2017active, biyik2018batch} based on the responses obtained to pairwise comparisons of trajectories using an active query synthesis approach. Bayesian Optimization has also been applied for controller tuning by adapting the framework to handle human preference feedback~\cite{coutinho2024human, csomay2022learning}, enforce safety constraints~\cite{berkenkamp2016safe}, or optimize efficiency through multi-fidelity data sources~\cite{marco2017virtual}.

\section{Modified Response Model} \label{sec:modified_response_model}
We propose a novel \textit{confidence aware response model} to encourage selection of queries with high confidence responses.  To achieve this, we model the noise variance as a function of the query distances which is consistent with frameworks explored in heteroskedastic regression \cite{chaudhuri2017active, das2023near}. Specifically, we assume:
\begin{equation*}
    d_1 = \|\vw - \vp\|_2^2 + e_1 \quad \text{and} \quad d_2 = \|\vw - \vq\|_2^2 + e_2
\end{equation*}
where $e_1$ and $e_2$ are noises with standard deviations $\sigma_0 \|\vw - \vp\|_2^2$ and $\sigma_0 \|\vw - \vq\|_2^2$  respectively.  \\
\\
The resulting probability model is as below:
\begin{equation*} \label{eq:conf_aware_resp_model}
    P(Y | \vw) = \Phi \left( \frac{\quad \|\vw - \vq\|_2^2 - \|\vw - \vp\|_2^2}{ \sigma_0 \sqrt{ \|\vw - \vq\|_2^4 +  \|\vw - \vp\|_2^4}} \right) = \Phi \left( f(\vw) \right)
\end{equation*}
\begin{equation} \label{eq:func_f_def}
    \text{with} \quad f(\vw) = \frac{\quad \|\vw - \vq\|_2^2 - \|\vw - \vp\|_2^2}{ \sigma_0\sqrt{ \|\vw - \vq\|_2^4 +  \|\vw - \vp\|_2^4}}.
\end{equation}
It can also be written as
\begin{equation}\label{eq:prob_model}
    P(Y| \vw) = \Phi \left( \frac{\quad \va^\top (\vw - \vb)}{ \sigma_0 \sqrt{2 \left( \|\vw - \vb\|_2^2 + \frac{\|\va\|_2^2}{16} \right)^2 + \frac{1}{2} \left( \va^\top (\vw - \vb) \right)^2}} \right)
\end{equation}
$\va$ represents the normal vector of the separating hyperplane as before, \\
$\vb = \frac{\vp + \vq}{2}$ represents the midpoint of $\vp$ and $\vq$, 
and \\
$\Phi (x)$ is the CDF of a symmetric noise distribution (i.e., $\Phi(-x) = 1 - \Phi(x)$ and $\Phi'(-x) = \Phi'(x)$). 

This modified response model, as explained in Sec.~\ref{sec:query_synthesis}, enables us to select queries with higher confidence responses which are typically pairs of points that are at an ideal distance from each other, which is neither too close nor too far.

\section{Query Synthesis} \label{sec:query_synthesis}
We wish to select queries that result in maximal improvement of our estimate of the user point $W$. The information gain of a query is evaluated by computing the Mutual Information (MI) between the user point and the response to the query. This can be written as:
\begin{equation*}
    I(W ; Y | (\vp,\vq)) = H(W) - H(W | Y, (\vp, \vq)).
\end{equation*}
Using the symmetry of mutual information, we can write
\begin{equation}\label{eq:infogain_pair}
    I(W ; Y | (\vp,\vq)) = H(Y | (\vp,\vq)) - H(Y | W, (\vp,\vq)). 
\end{equation}
Since evaluating all candidate queries is computationally expensive, we instead directly optimize for the most informative query $(\tilde{\vp}, \tilde{\vq})$. Further, since every pairwise comparison query can also be characterized by its corresponding parameters as derived in \eqref{eq:prob_model}, we compute the optimal parameters $(\tilde{\va}, \tilde{\vb})$ that maximize the information gain. 

\subsection*{Reparametrizing the MI expression}
Re-writing \eqref{eq:infogain_pair} by conditioning on the hyperplane normal vector and query midpoint instead, we get
\begin{equation}\label{eq:infogain_hyperplane}
    I(W ; Y | (\va, \vb)) = H(Y | (\va, \vb)) - H(Y | W, (\va, \vb)).
\end{equation}

We now optimize the information gain objective in \eqref{eq:infogain_hyperplane} by separately analyzing the two terms. The first term $H(Y | (\va, \vb))$ is the marginal entropy of the binary response $Y$. Since binary entropy is maximized when the two outcomes are equiprobable, this term is maximized when the outcomes are equiprobable. Under a Gaussian posterior over W, this condition induces a symmetry constraint that places the query midpoint $\vb$ at the posterior mean, $\vmu$. The second term $\underset{W}{\E} [ H(Y | W, (\va, \vb)) ]$ measures the expected conditional entropy of $Y$ given $W$. It is minimized when the conditional probability $P(Y|\vw)] = \Phi(f(\vw))$ is as certain as possible. This leads to an optimization over the hyperplane normal vector $\va$, which results in $\va$ aligning with the principal eigenvector of the posterior covariance $\mSigma$. Below, we provide a detailed derivation of these results.

\subsection{Optimizing the first term}

The first term in \eqref{eq:infogain_hyperplane} can be written as
\begin{equation*}
    H(Y | (\va, \vb)) = H \left( \underset{W}{\E} \left[ P(Y = 1 | (\va, \vb), W) \right] \right) = H (\E[\Phi(f(W))]).
\end{equation*}
We know that entropy for a binary random variable is maximum when both the outcomes are equi-probable, i.e. when 
\begin{equation*}
    \E [P(Y = 1)] = \E[\Phi(f(W))] = 0.5.
\end{equation*}
Assuming $W \sim \NN(\vmu, \mSigma)$, the equi-probability criterion results (more details in Sec.~\ref{sec:supp-query_details}) in 
\begin{equation}\label{eqn:term_I_cond}
    \vmu = \frac{\tilde{\vp} + \tilde{\vq}}{2} = \tilde{\vb}.
\end{equation}
Under this condition, the expected probability is exactly $0.5$.
\begin{enumerate}
    \item Antisymmetry

    Let $W = \vmu + \Delta$ where $\Delta \sim \NN(\vzero, \mSigma)$. Substituting this into $f(W)$ yields 
    \[f(\vmu + \Delta) = \frac{\va^\top \Delta}{\sigma_0 \sqrt{2 \left( \|\Delta\|_2^2 + \frac{\|\va\|_2^2}{16} \right)^2 + \frac{1}{2} \left( \va^\top \Delta \right)^2}}.\]
    Since the numerator is linear (odd) and the denominator is quadratic (even) in $\Delta$, $f(W)$ is odd about $\vmu$, i.e.
    \begin{equation*}
        f(\vmu - \Delta) = -f(\vmu + \Delta).
    \end{equation*}
    
    \item Distributional Symmetry

    Since the PDF $p(\Delta)$ is symmetric and $f$ is odd, the random variable $Y = f(W)$ is symmetrically distributed around zero. This implies $p(x) = p(-x)$ and $\E[f(W)] = 0$.

    \item Expected Probability

    Any link function  satisfying $\Phi(-x) = 1 - \Phi(x)$ results in
    \begin{equation*}
        \E[\Phi(f)] = \int_{0}^{\infty} [\Phi(f) + \Phi(-f)] p(f) \, df = \int_{0}^{\infty} (1) p(f) \, df = 0.5.
    \end{equation*}
     
\end{enumerate}

\subsection{Optimizing the second term}

The optimal value of $\va$, $\tilde{\va}$ can be computed as the minimizer of the second term in equation \eqref{eq:infogain_hyperplane} illustrated as below:
\begin{equation*}
    \tilde{\va} = \underset{\va}{\argmin} \quad H(Y | W, (\va, \vb)) = \underset{\va}{\argmin} \quad \underset{W}{\E} [ H(Y | W, (\va, \vb)) ] = \underset{\va}{\argmin} \quad \underset{W}{\E} [H(\Phi (f(W)))].
\end{equation*}
The above entropy is minimized when $\Phi (f(W))$ is close to either $0$ or $1$. This corresponds to maximizing $|f(W)|$ and since $W \sim \NN(\vmu, \mSigma)$, we maximize $\E[|f(W)|]$. We have
\[\tilde{\va} = \underset{\va}{\argmax~} \E [|f(W)|] = \underset{\va}{\argmax~} \E \left[ \frac{|\va^\top (W - \vb)|}{\sigma_0 \sqrt{2 \left( \|W - \vb\|_2^2 + \frac{\|\va\|_2^2}{16} \right)^2 + \frac{1}{2} \left( \va^\top (W - \vb) \right)^2}} \right].
\]
From \eqref{eqn:term_I_cond}, we have $\tilde{\vb} = \vmu$
which results in 
\begin{equation*}
    \tilde{\va} = \underset{\va}{\argmax~} \E_{W \sim \mathcal{N}(\vmu,\mSigma)} \left[ \frac{|\va^\top (W - \vmu)|}{\sigma_0 \sqrt{2 \left( \|W - \vmu\|_2^2 + \frac{\|\va\|_2^2}{16} \right)^2 + \frac{1}{2} \left( \va^\top (W - \vmu) \right)^2}} \right].
\end{equation*}
Let $\vp = \vmu + \tilde{\vx}$ and $\tilde{\vq} = \vmu - \tilde{\vx}$. Further, let $S = W - \vmu \sim \NN(0, \mSigma)$. We now have
\begin{equation*}
    \tilde{\vx} = \underset{\vx}{\argmax~} \E_{S \sim \mathcal{N}(\vzero,\mSigma)} \left[ \frac{4|\vx^\top S|}{\sigma_0 \sqrt{2 \left( \|S\|_2^2 + \|\vx\|_2^2 \right)^2 + 2 \left( 2\vx^\top S \right)^2}} \right] = \underset{\vx}{\argmax~} F(\vx).
\end{equation*}
Setting $\vx = r \vu$, we obtain
\begin{equation}\label{eqn:f_obj}
    \tilde{r}, \tilde{\vu} = \underset{r, \vv}{\argmax~} \E \left[ \frac{4 r|\vu^\top S|}{\sigma_0 \sqrt{2 \left( \|S\|_2^2 + r^2 \right)^2 + 8 r^2 \left( \vu^\top S \right)^2}} \right].
\end{equation}
The ratio is maximized for $\vu = \vv_1$ where $\vv_1$ is the principal eigenvector of $\mSigma$. Optimizing \eqref{eqn:f_obj} for $\tilde{r}$ results in it being independent of $\sigma_0$. However, in practice we find that the optimal magnitude does depend on $\sigma_0$. Hence, the optimal magnitude is obtained by one-dimensional optimization as below: 
\begin{equation*}
    \tilde{r} = \underset{r}{\argmin} ~\underset{S}{\E} \left[ H\left( \Phi \left( \frac{4 r ~\vv_1^\top S}
    {\sigma_0 \sqrt{2\left(\|S\|_2^2 + r^2\right)^2 + 8r^2 \left(\vv_1^\top S\right)^2}} \right) \right) \right]
\end{equation*}
with scope for varying with $\sigma_0$, since $\sigma_0$ enters nonlinearly inside the link function - entropy composition. The resulting normal vector of the optimal hyperplane is $\tilde{\va} = 4 \tilde{r} \vv_1$. Full derivation for this method, referred to as \textit{Info-Synth}, is in Sec.~\ref{sec:opt_second_term}. The cost of this synthesis is $O(S^*d^2)$ since it involves the estimation of the covariance matrix, $\mSigma$ using $S$ Monte Carlo (MC) samples.\\ 
\\
Intuitively, in \eqref{eqn:f_obj} choosing a larger value of $r$ can make both the numerator and denominator larger. Thus, we want to select a moderate value of $r$. This moderate value of the magnitude is what encourages the selection of pairs of points that are at an ideal distance, neither too close nor too far from each other, thus resulting in selection of higher response-confidence queries.

\section{Query Approximation} \label{sec:query_approximation}
If we have the flexibility to query any pair of points, we can query the synthesized pair of points directly. Else, we need to approximate them with the best points in the dataset. Most works \cite{chumbalov2020scalable, wang2015active} on synthesizing active queries select the nearest neighbors in the dataset as the closest approximations. This works well for a single point but for a pair of points, the resulting geometry could bear no association with the synthesized optimal pair. We propose two new methods for selecting the best approximate points that can be applied to datasets depending on their dimensionality. Both employ a filter and refine strategy, where the queries are first filtered for informativeness based on approximate criteria and then refined using MI evaluations.

\subsection{Pair M-dist}

The first method is a Mahalanobis distance based framework. Euclidean distance based methods rely on distances in the data space and do not explicitly account for the local geometry of the mutual information objective. Consequently, points that are close in Euclidean distance may incur very different losses in informativeness. To address this, we introduce a Mahalanobis distance based selection that uses a Taylor series approximation of the mutual information. Here, the negative Hessian induces a geometry-aware metric for selecting candidate points with minimal loss relative to the optimal query. \\
\\
Let $I(\vz) := I(\vp,\vq) := I(Y; W | \vp,\vq)$ with $\vz^*$ being the maximizer. \\
\\
The Taylor series expansion of $I(\vz)$ around $\vz^*$ is:
\begin{equation*}
    I(\vz) \approx I(\vz^*) + \nabla I(\vz^*)^\top (\vz - \vz^*) + \frac{1}{2} (\vz - \vz^*)^\top \nabla_{\vz}^2 I(\vz^*) (\vz - \vz^*).
\end{equation*}
Since $\vz^*$ is the maximizer, the gradient $\nabla I(\vz^*) = 0$. Thus, the approximation simplifies to:
\begin{equation} \label{eq:taylor_exp}
    I(\vz) \approx I(\vz^*) + \frac{1}{2} (\vz - \vz^*)^\top \mH(\vz^*) (\vz - \vz^*)
\end{equation}
where $\mH(\vz^*) = \nabla_{\vz}^2 I(\vz^*)$ is the Hessian matrix. \\
\\
Consider the difference:
\begin{equation*}
    I(\vz^*) - I(\vz) = \frac{1}{2} (\vz^* - \vz)^\top (-\mH(\vz^*)) (\vz^* - \vz)
\end{equation*}
which can also be written as
\begin{equation*}
    \Delta I(\vz) = \frac{1}{2} (\Delta \vz)^\top (-\mH(\vz^*)) (\Delta \vz) = \frac{1}{2} \|\Delta \vz\|_{\mM}^2.
\end{equation*}
where $\Delta I(\vz) = I(\vz^*) - I(\vz)$, $\Delta \vz = \vz^* - \vz$ and $\mM = -\mH(\vz^*)$. \\
\\
The detailed derivations for the Hessian can be found in the appendix (Sec.~\ref{sec:supp-mahalanobis_distance}).\\
\\
First, we find the nearest $\alpha$ fraction of pairs to the optimal pair of points by computing the Mahalanobis distances and then find the most informative pair among them via MI evaluations for these pairs.  The computational cost of this method is $O((S + P)^*d^2 + \alpha^*P^*S^*d)$ where $P$ is the total number of pairs and $d$ is the dimension of the dataset. $O(S^*d^2)$ represent the cost of optimal query synthesis and computing the Hessian using Monte Carlo estimation. $O(P^*d^2)$ represents the cost of computing the Mahalanobis distances while $O(\alpha^*P^*S^*d)$ is the cost of MI evaluation for $\alpha^*P$ pairs at $O(S^*d)$ per pair.
 
\subsection{Pair Opt-dist}

Since computing the hessian is very expensive in higher dimensions, we use a more computationally friendly heuristic for the filtering in higher-dimensional datasets. This heuristic (more details in Sec.~\ref{sec:supp-pair_opt_dist_heuristic}) which measures deviation from the optimality criteria of $\tilde{\vb} = \vmu$ and $\tilde{\va} = 4\tilde{r} \vv_1$, is given by
\begin{equation}
    \eta(\vp,\vq,\lambda) = \sum_{i=1}^{d} \frac{(\vb_i - \vmu_i)^2}{\mSigma_{ii}} + \lambda \|(\vp - \vq) - 2 \tilde{r} \vv_1 \|_2^2.
\end{equation}
We refer to this method as \textit{Pair Opt-Dist}, where the queries are filtered by selecting pairs that minimize $\eta(\vp,\vq,\lambda)$ and then refined via MI evaluation on a $\gamma$ fraction of the queries. The total cost of this method is $O(S^*d^2 + P^*d + \gamma^*P^*S^*d)$.

\section{Experiments} \label{sec:experiments}

\subsection{Preference Learning on Synthetic Datasets} \label{sec:synth_data_exp}
We run experiments on synthetic datasets (details in Sec.~\ref{sec:supp-synth_data_exp}) for both continuous and constrained settings. We compare the performance of our methods to multiple baselines. \textit{Gauss Search} - \textit{Synthesis} and \textit{Discrete} - refer to the active query synthesis method in \cite{chumbalov2020scalable} and its discrete approximation respectively. \textit{Active Discrete} refers to the pool based method of evaluating candidate queries for optimality based on mutual information. \textit{Random} - \textit{Synthesis} and \textit{Discrete} - refer to the methods where a pair of points are sampled at random over the continuous space and selected at random from the candidate pairs respectively. (Method names inspired by \cite{sadigh2017active}.) The nearest point approximation method is referred to as \textit{NN Approx} and the $k$-NN based method as \textit{k-NN Approx} (details in Sec.~\ref{sec:supp-synth_data_exp})).

\textit{Active Discrete} is the most computationally expensive method with a cost of $O(P^*S^*d)$ for $P$ queries while \textit{k-NN Approx} has a cost of $O(N^*d + \beta^*P^*S^*d)$ with cost $O(N^*d)$ to compute the nearest neighbors and $O(\beta^*P^*S^*d)$ for MI evaluation of $\beta$ fraction of the total pairs.

\begin{figure*}[t]
    \hfill
    \begin{minipage}{0.24\linewidth}
        \centering
        \includegraphics[width=\textwidth]{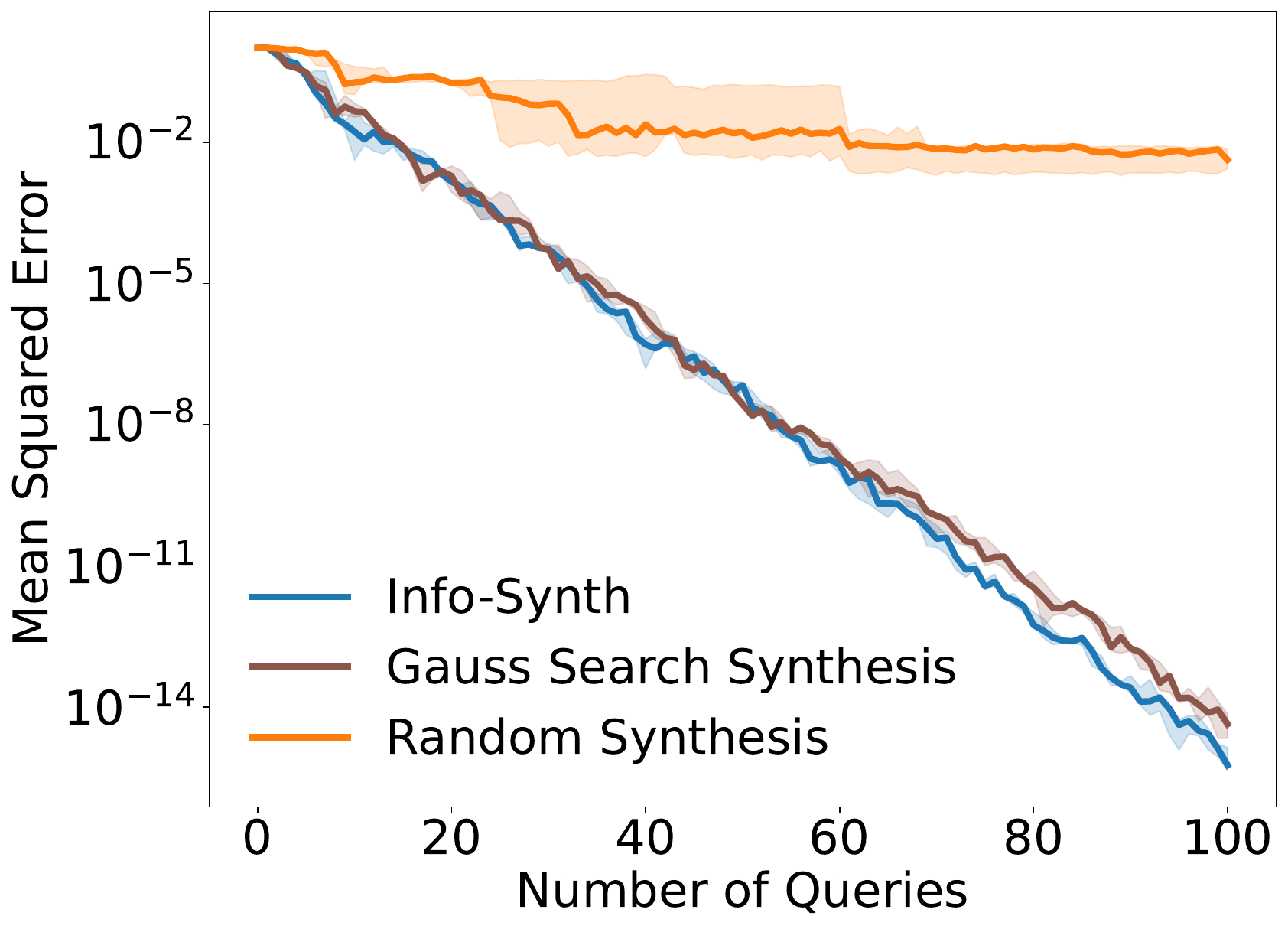}
        \subcaption{\footnotesize $\sigma_0 = 0.001$}
    \end{minipage}
    \hfill
    \begin{minipage}{0.24\linewidth}
        \centering
        \includegraphics[width=\textwidth]{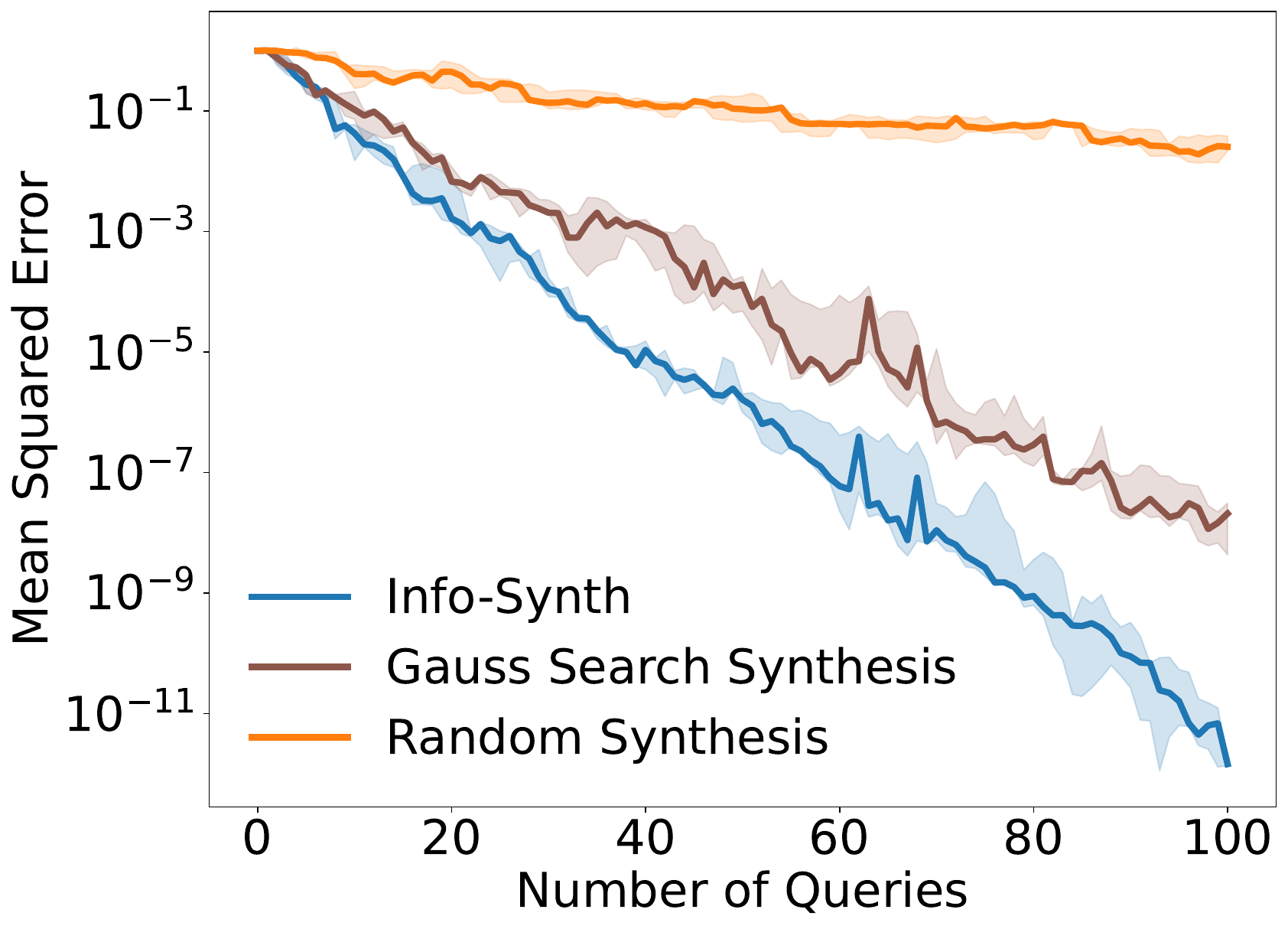}
        \subcaption{\footnotesize $\sigma_0 = 0.1$}
    \end{minipage}
    \hfill
    \begin{minipage}{0.24\linewidth}
        \centering
        \includegraphics[width=\textwidth]{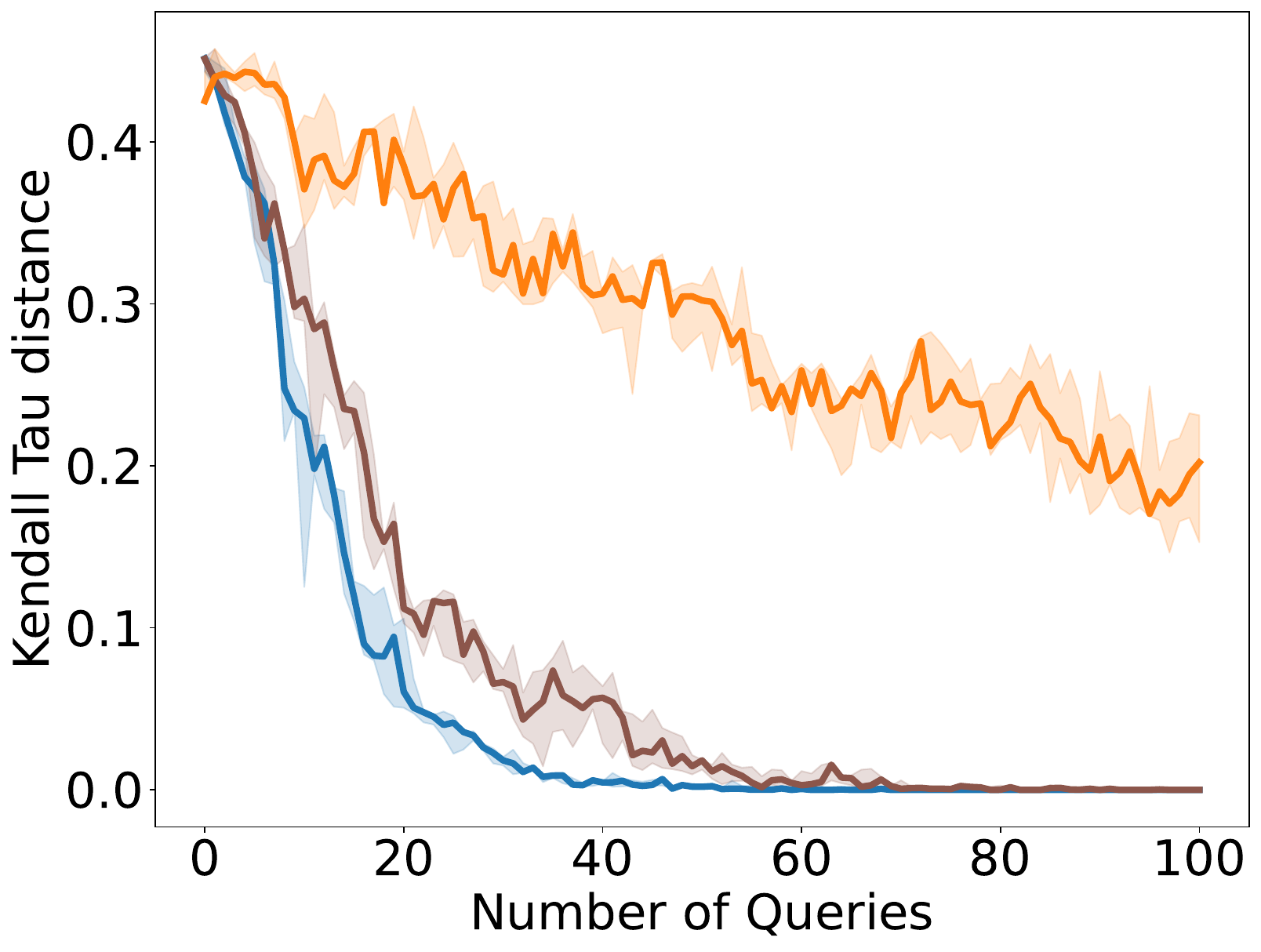}
        \subcaption{\footnotesize $\sigma_0 = 0.1$}
    \end{minipage}
    \hfill
    \begin{minipage}{0.24\linewidth}
        \centering
        \includegraphics[width=\textwidth]{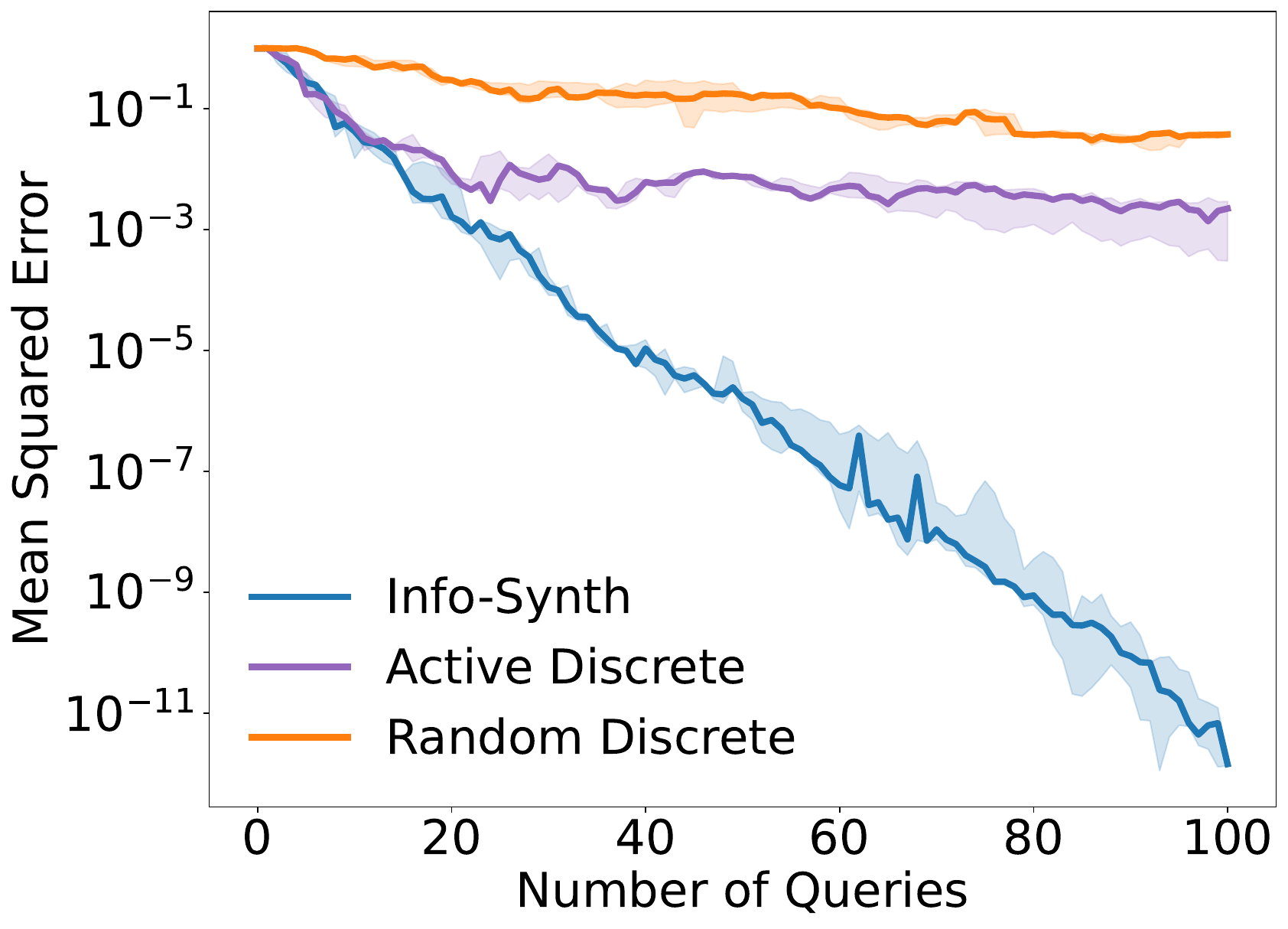}
        \subcaption{\footnotesize $\sigma_0 = 0.1$}
    \end{minipage}
    \hfill
\caption{\footnotesize Query synthesis performance comparison between different AL methods and Random on synthetic datasets. The plots correspond to datasets in $4$D comparing different synthesis methods in ((a), (b), (c)) and with $500$ points comparing synthesis with discrete methods in (d). In the MSE plots, the $y$-axis corresponds to the MSE between the true point and the estimated point. In the Kendall Tau distance plots, the $y$-axis corresponds to the KT distance computed between the estimated \textit{user-item} distances against the ground truth \textit{user-item} distances. For both metrics, lower values indicate a better performance.}
\label{fig:syn_data_expt_1}
\vspace{-1\baselineskip}
\end{figure*}

\begin{figure*}[t]
    \hfill
    \begin{minipage}{0.31\linewidth}
        \centering
        \includegraphics[width=0.8\textwidth]{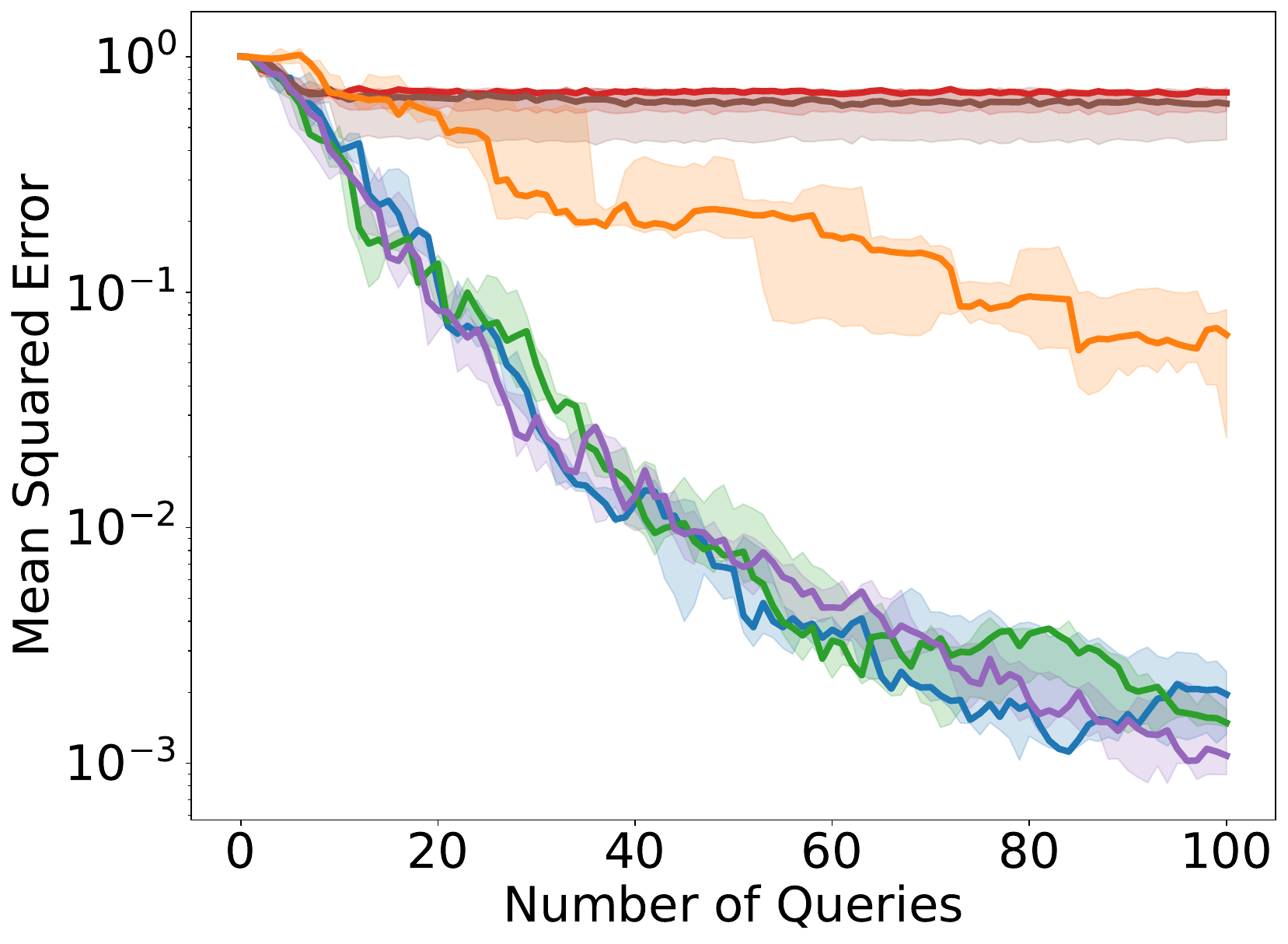}
    \end{minipage}
    \hfill
    \begin{minipage}{0.31\linewidth}
        \centering
        \includegraphics[width=1.1\textwidth]{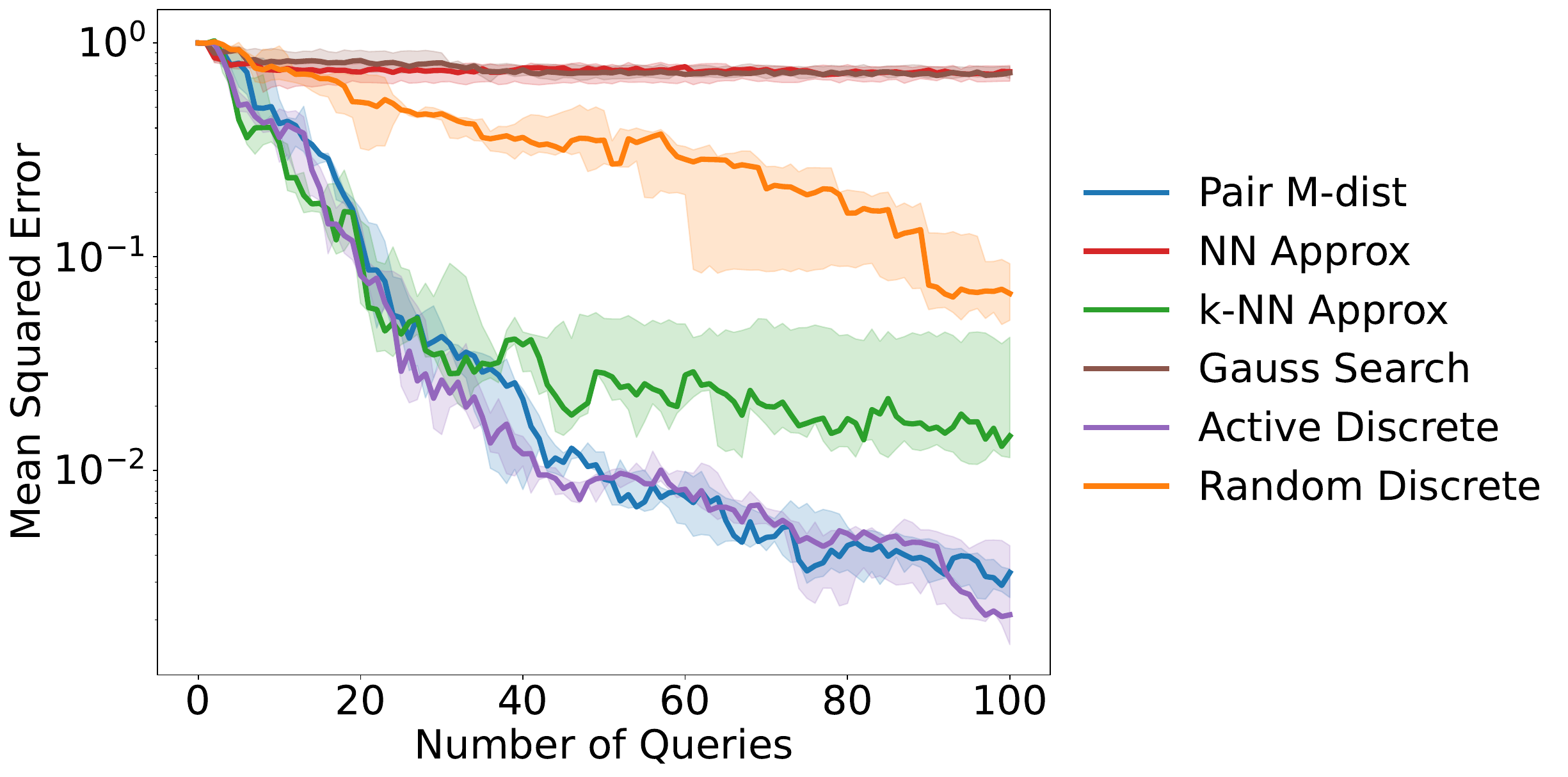}
    \end{minipage}
    \hfill
    \begin{minipage}{0.31\linewidth}
        \centering
        \includegraphics[width=0.8\textwidth]{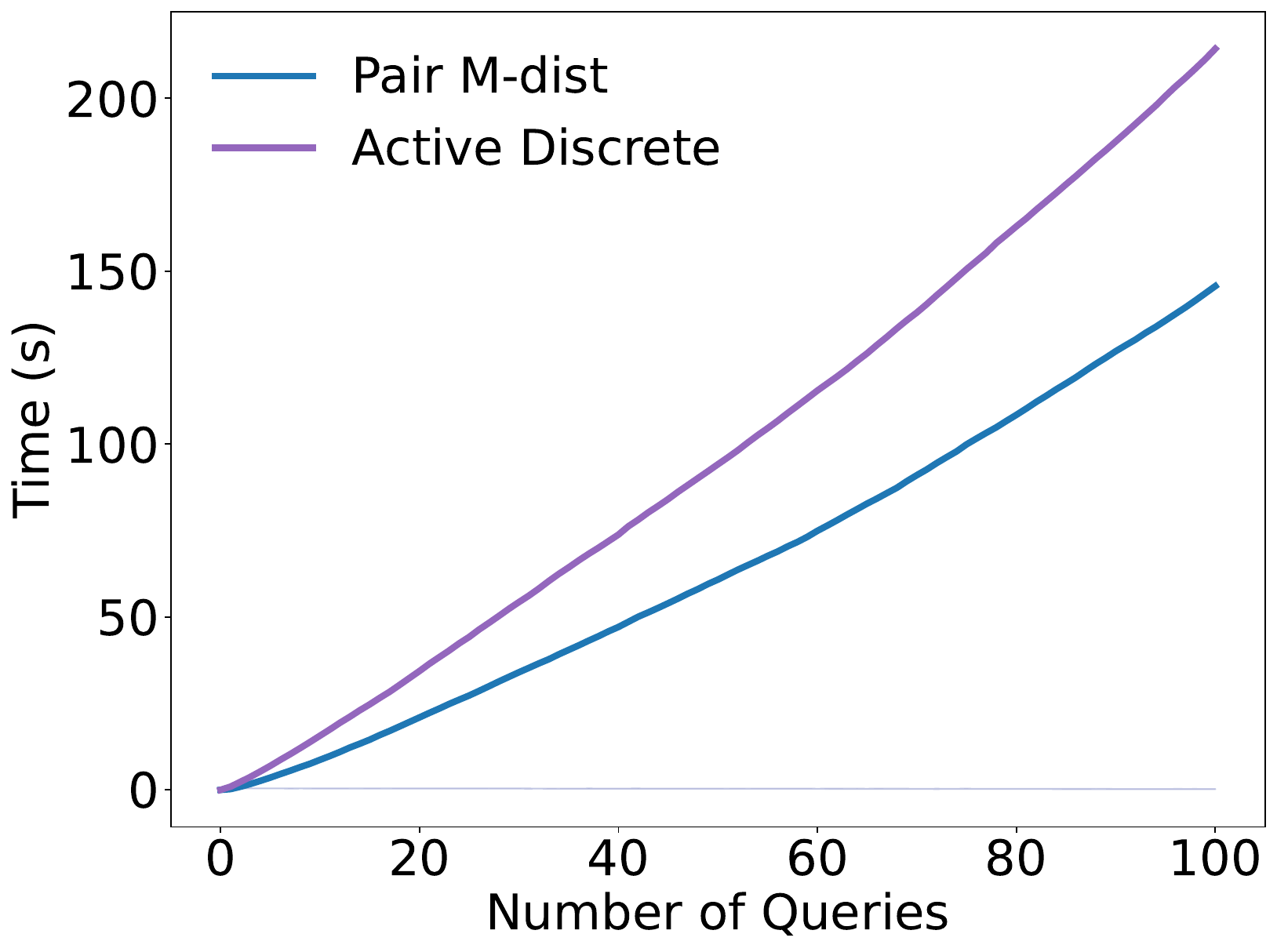}
    \end{minipage}
    \hfill
\caption{\footnotesize Comparison of different methods of discrete approximation of the optimal synthetic query and other baseline methods on a synthetic dataset in $10$ D with $500$ points (left), $100$ points (center) and time taken for query selection (right) for $100$ points. }
\label{fig:syn_data_expt_2}
\vspace{-1\baselineskip}
\end{figure*}

As illustrated in Fig.~\ref{fig:syn_data_expt_1}, \textit{Info-Synth} outperforms the baseline \textit{Gauss Search Synthesis} and \textit{Random Synthesis}. While Gauss Search Synthesis has a very similar performance at a lower noise level of $\sigma_0 = 0.001$, we can clearly see the difference when $\sigma_0 = 0.1$. Although \textit{Gauss Search Synthesis} is not a perfect baseline since it is not designed for our specific probability model, we use it as a baseline due to lack of relevant baselines. Further, comparison of \textit{Info-Synth} with the discrete methods \textit{Active Discrete} and \textit{Random Discrete} shows the benefits of optimizing in a continuous space. 

For the discrete approximations, we observe in Fig.~\ref{fig:syn_data_expt_2} that even though \textit{k-NN Approx} performs as well as \textit{Pair M-dist} and \textit{Active Discrete} when there are $500$ points, it saturates beyond a certain point for $100$ points. Our method achieves the same performance as \textit{Active Discrete}, while being faster as illustrated in the plot on the right. The fractions $\alpha$ and $\beta$ are selected such that the computational costs of \textit{Pair M-dist} and \textit{k-NN Approx} are approximately equal.

\subsection{Preference Learning on the Reddit Summary Dataset} \label{sec:reddit_summary_exp}
\begin{figure}[t]
    \hfill
    \begin{minipage}{0.24\linewidth}
        \centering
        \includegraphics[width=\textwidth]{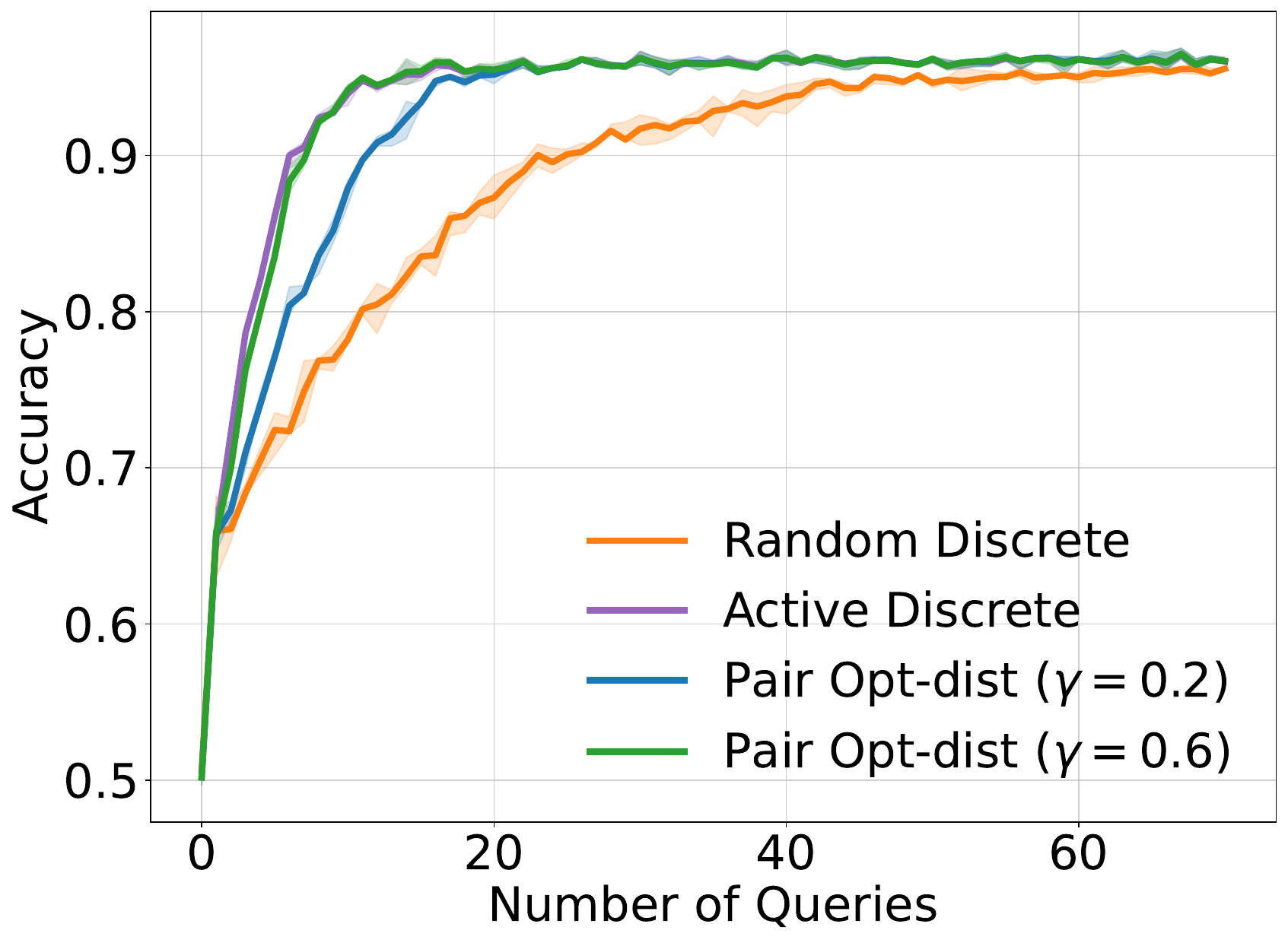}
          \subcaption{\footnotesize$\sigma = 0.1$}
          \label{fig:reddit-performance_accuracy_comparison_sigma0p1}
    \end{minipage}
    \hfill
    \begin{minipage}{0.24\linewidth}
        \centering
        \includegraphics[width=\textwidth]{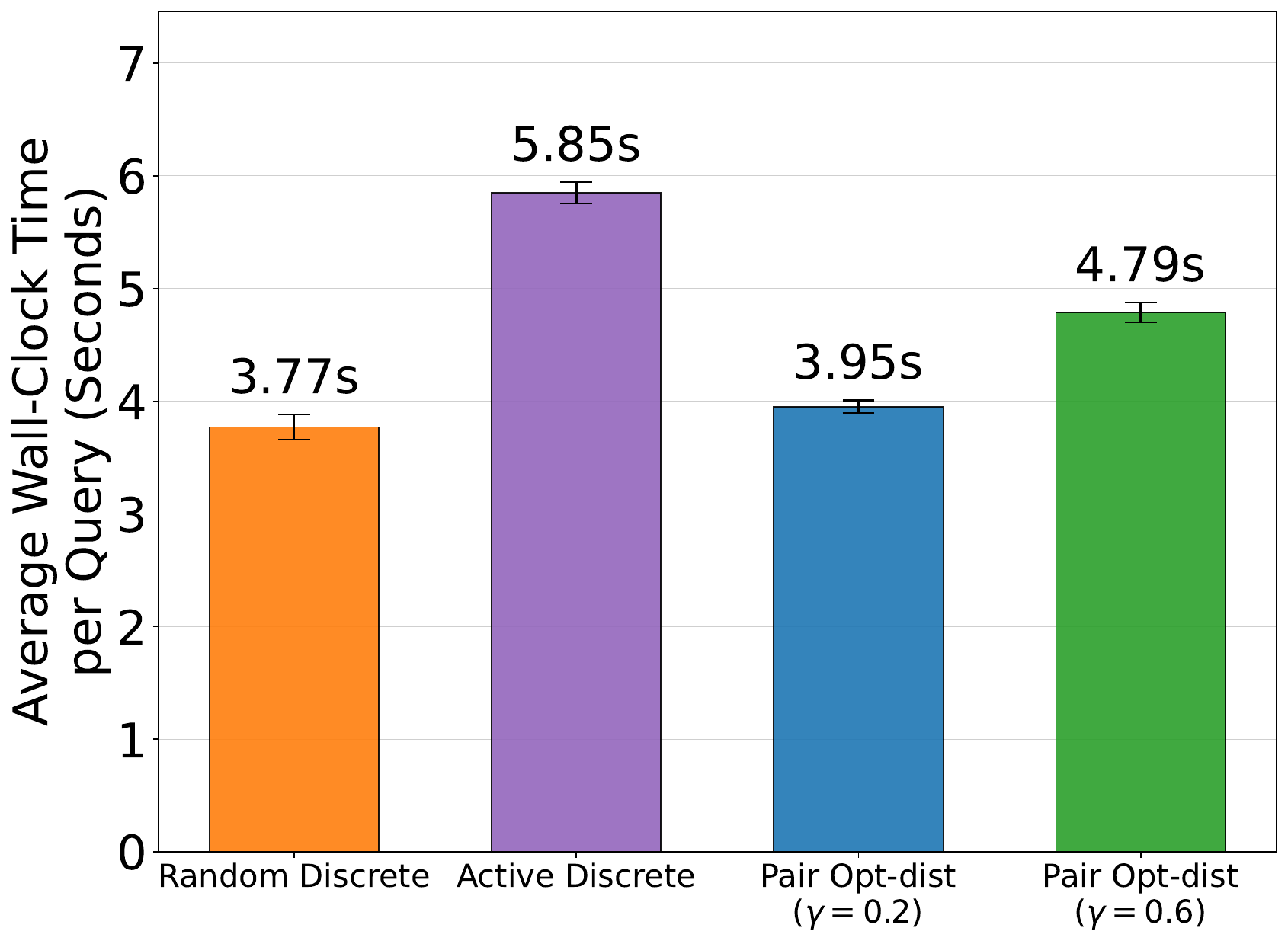}
      \subcaption{\footnotesize $\sigma = 0.1$}
      \label{fig:reddit-performance_time_comparison_sigma0p1}
    \end{minipage}
    \begin{minipage}{0.24\linewidth}
        \centering
        \includegraphics[width=\textwidth]{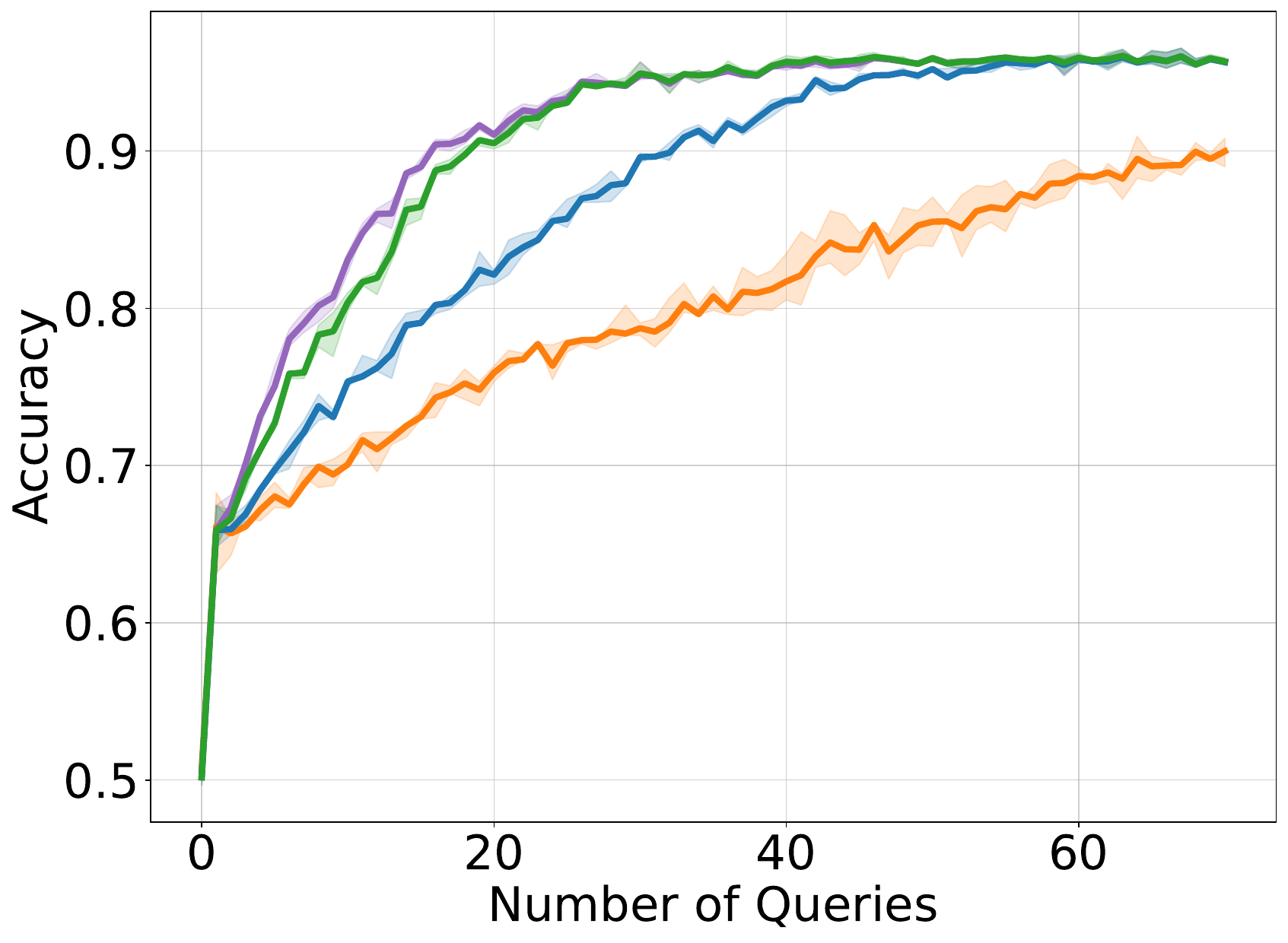}
      \subcaption{\footnotesize $\sigma = 0.3$}
      \label{fig:reddit-performance_accuracy_comparison_sigma0p3}
    \end{minipage}
    \begin{minipage}{0.24\linewidth}
        \centering
        \includegraphics[width=\textwidth]{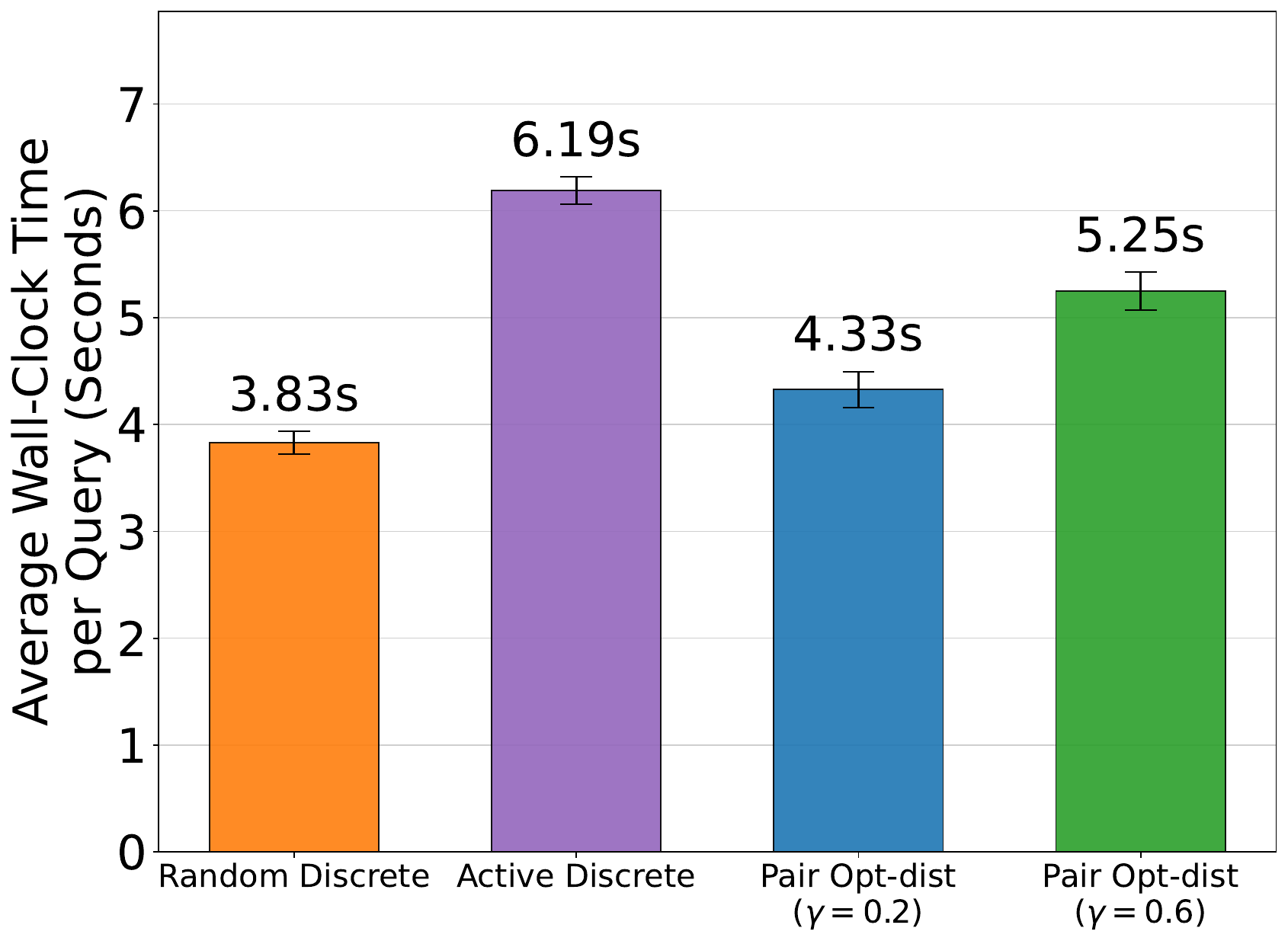}
      \subcaption{\footnotesize $\sigma = 0.3$}
      \label{fig:reddit-performance_time_comparison_sigma0p3}
    \end{minipage}
    \hfill
    \vspace{-2mm}
      \caption{\footnotesize  Performance analysis on the Reddit Summary TL;DR dataset for different AL methods. Our proposed approximation method, \textit{Pair Opt-dist} is shown for two different filtering levels of $\gamma=0.6$ (green)  and $\gamma=0.2$ (blue). Here $\gamma$ represents total fraction of queries used for selection. (a) and (b) show the preference prediction accuracy  and average query selection time for $\sigma=0.1$ while (c) and (d) show these results for $\sigma=0.3$.}
      \label{fig:reddit_summary_expt}
      \vspace{-1\baselineskip}
\end{figure}

We also conducted experiments on the Reddit Summary TL;DR dataset where we learn user preferences for text summaries  through active querying  \cite{stiennon2020learning,li2024personalized}. Users are provided queries consisting of a paragraph and its two candidate summaries for which they provide a preference label. To obtain summary embeddings, we adopt the same methodology as \cite{chen2025pal} where the original text and its summary are concatenated and then passed to the pretrained OPT-350M transformer model that is appended with a pretrained frozen embedding space projection head MLP \cite{zhang2022opt}. The output embedding of this MLP head, corresponding to the last token of this concatenated sentence, serves as the summary representation embedding. This process is repeated for the second summary candidate, resulting in two distinct embeddings corresponding to the same paragraph. Further, these embeddings are projected to a space of 128 dimensions using PCA.

For the experiments, we actively query responses to the summary pairs and estimate the user as a point in the same embedding space. We selected a user at random from this dataset and compared the performance of different AL methods for query selection. We use the \textit{Pair Opt-dist} approximation methodology for query selection. Figure \ref{fig:reddit_summary_expt} shows how our proposed approximation scheme can achieve a performance comparable to \textit{Active Discrete} with a much shorter query selection time. By adjusting $\gamma$, our framework offers a flexible trade-off: reducing $\gamma$ significantly accelerates the average query selection time, providing a valuable alternative for scenarios where computational speed is prioritized over strict sample efficiency.

\subsection{Gain Tuning} \label{sec:gain_tuning_exp}
\begin{figure}[t]
    \hfill
    \begin{minipage}{0.24\linewidth}
        \centering
        \includegraphics[width=\textwidth]{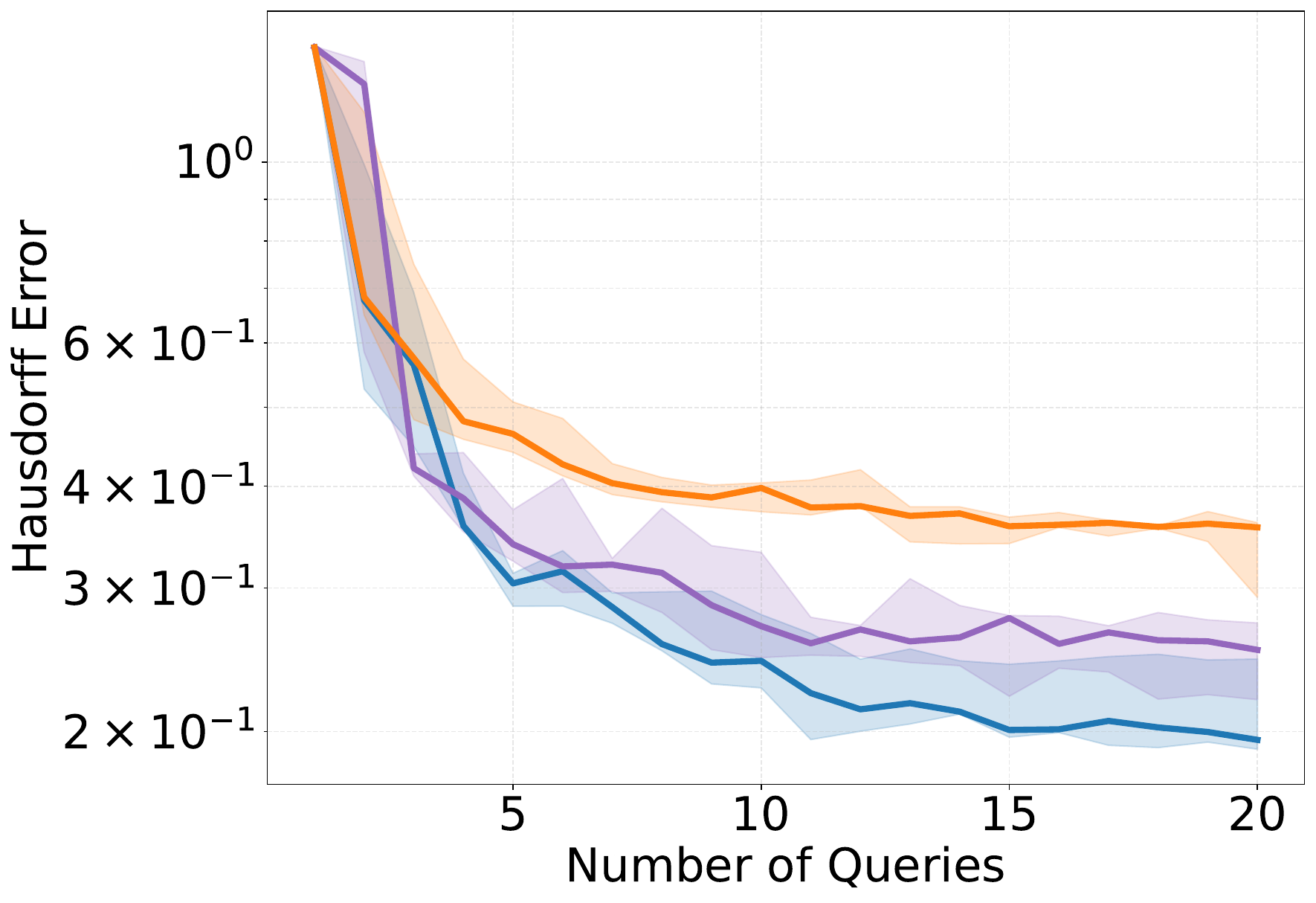}
      \subcaption{}
    \end{minipage}
    \hfill
    \begin{minipage}{0.24\linewidth}
        \centering
        \includegraphics[width=\textwidth]{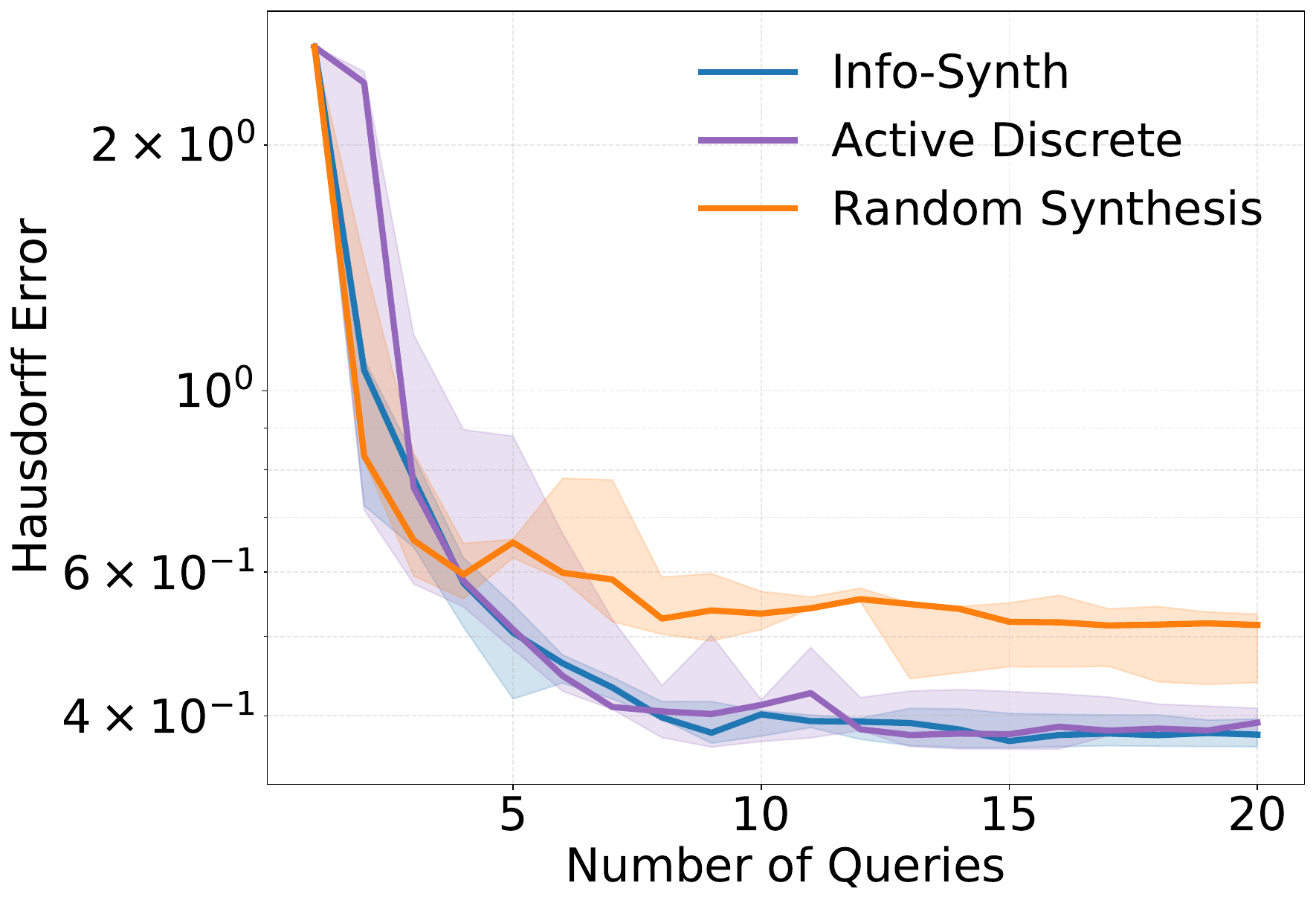}
      \subcaption{}
    \end{minipage}
    \hfill
    \begin{minipage}{0.24\linewidth}
        \centering
        \includegraphics[width=\textwidth]{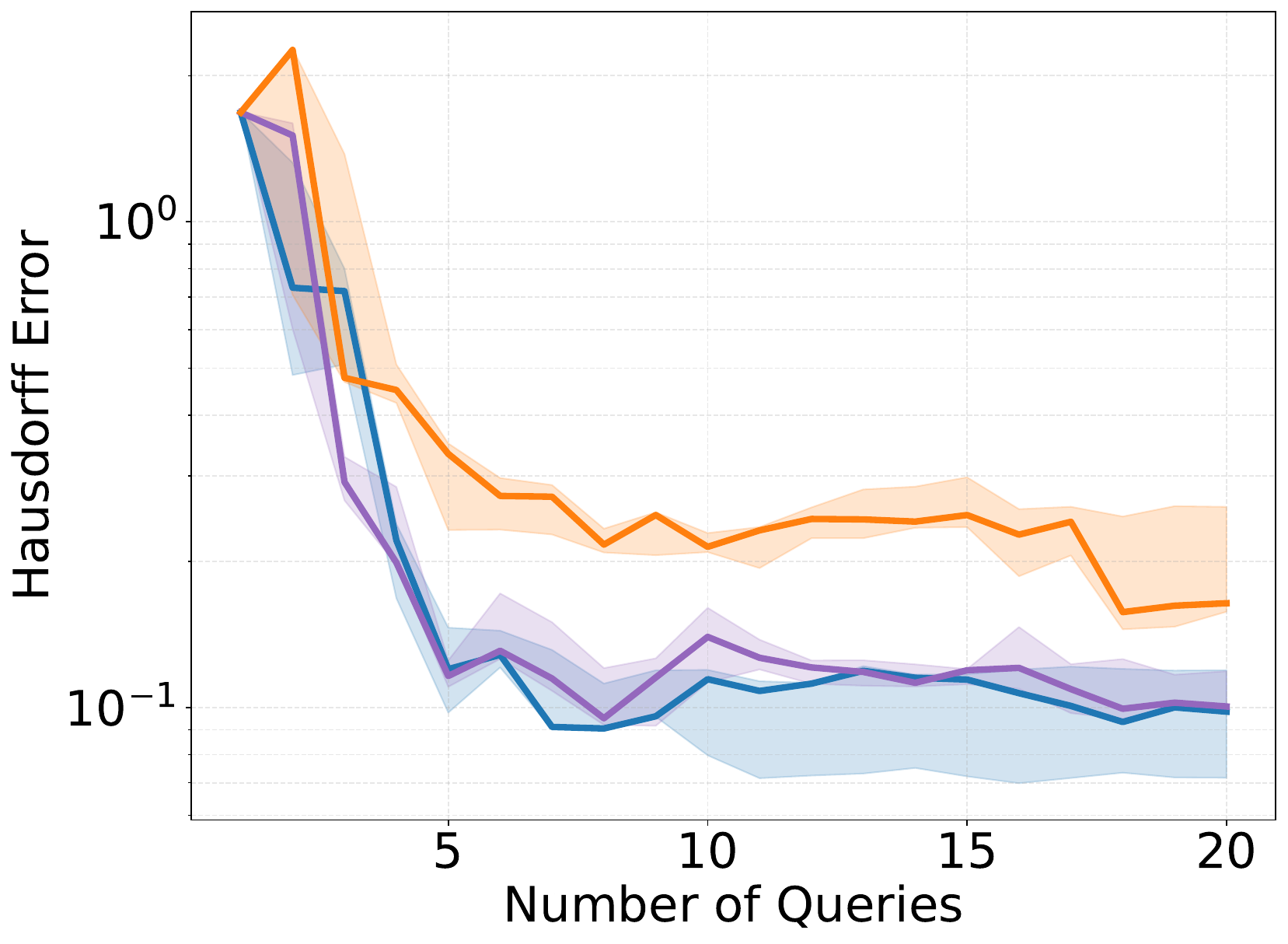}
          \subcaption{}
    \end{minipage}
    \hfill
    \begin{minipage}{0.24\linewidth}
        \centering
        \includegraphics[width=\textwidth]{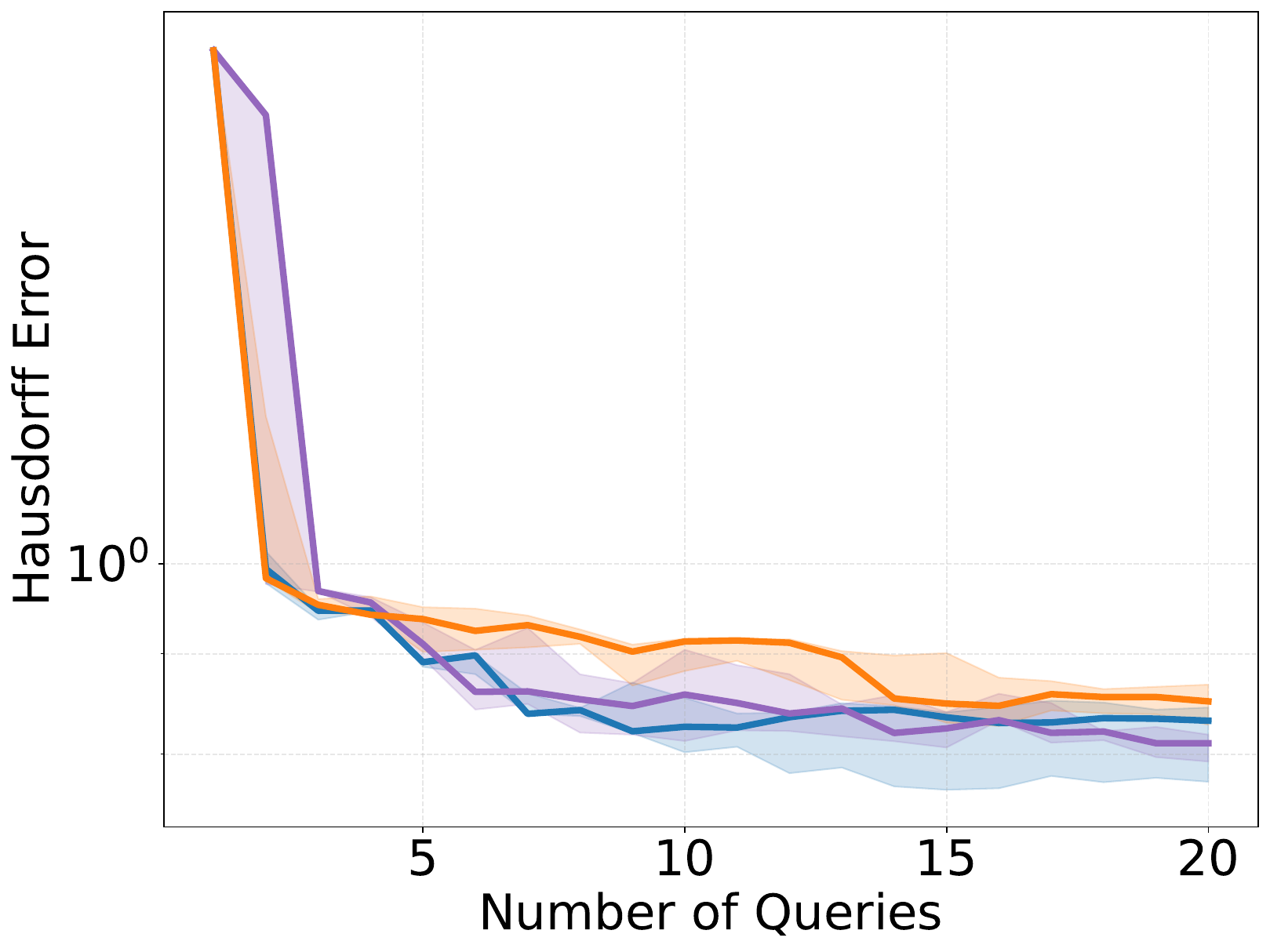}
      \subcaption{}
    \end{minipage}
    \hfill
    \vspace{-2mm}
      \caption{\footnotesize Trajectory tracking error comparison for different experiments. The plots represent the performance with error aggregation over different initial states for high curvature (a) and standard sinusoidal (b) trajectories, and error aggregation over different trajectories with an initial heading error (c) and lateral error (d).}
      \label{fig:gain_tuning_expt}
      \vspace{-1\baselineskip}
\end{figure}

In this section, we apply our active query methodology to learn the parameters of a trajectory-tracking controller for a unicyle robot. The detailed experimental setup is discussed in Sec.~\ref{sec:supp-gain_tuning_exp} and here, we discuss the results. We evaluate the performance and sample-efficiency of our proposed method by plotting the tracking error convergence over time (number of queries). We compare our method \textit{Info-Synth} to a pool based active method, \textit{Active Discrete} that samples points in the parameter space and evaluates different candidates for optimality to choose the best one. The baseline \textit{Random Synthesis} queries a randomly sampled gain in the continuous parameter space. We consider performance over different trajectories and initial states. 

As illustrated in Fig.~\ref{fig:gain_tuning_expt}, we observe that our method achieves the best performance in all the experiments, thus showcasing the power of both active learning and query synthesis in a continuous space. \textit{Active Discrete} performs consistently better than \textit{Random Synthesis}, even achieving a similar performance as our method in some cases. \textit{Info-Synth} wins even when its performance is similar to \textit{Active Discrete} as the latter involves MI evaluations of multiple candidate queries.

\section{Conclusion} \label{sec:conclusion}
In this work, we addressed the computational inefficiencies of standard pool-based active learning and the inherent unreliability of ambiguous pairwise comparisons. We introduced a novel \textit{confidence aware response model} alongside \textit{Info-Synth}, an active framework for synthesizing optimal queries in continuous spaces. For constrained datasets, we propose two approximation strategies, \textit{Pair M-dist} and \textit{Pair Opt-dist}. Evaluated across synthetic preference learning, text summaries, and robotic gain tuning, our framework dynamically identifies informative, high-confidence queries, thereby offering an efficient and structured solution for complex preference learning and subjective parameter tuning.

\bibliography{references}

@article{coombs1950psychological,
  title={Psychological scaling without a unit of measurement.},
  author={Coombs, Clyde H},
  journal={Psychological review},
  volume={57},
  number={3},
  pages={145},
  year={1950},
  publisher={American Psychological Association}
}

@article{settles2009active,
  title={Active learning literature survey},
  author={Settles, Burr},
  year={2009},
  publisher={University of Wisconsin-Madison Department of Computer Sciences}
}

@article{tajbakhsh2021guest,
  title={Guest editorial annotation-efficient deep learning: the holy grail of medical imaging},
  author={Tajbakhsh, Nima and Roth, Holger and Terzopoulos, Demetri and Liang, Jianming},
  journal={IEEE transactions on medical imaging},
  volume={40},
  number={10},
  pages={2526--2533},
  year={2021},
  publisher={IEEE}
}

@article{kim2025n,
  title={N Segment: Label-specific Deformations for Remote Sensing Image Segmentation},
  author={Kim, Yechan and Yoon, DongHo and Kim, SooYeon and Jeon, Moongu},
  journal={IEEE Geoscience and Remote Sensing Letters},
  year={2025},
  publisher={IEEE}
}

@article{bellamy2022batched,
  title={Batched Bayesian optimization for drug design in noisy environments},
  author={Bellamy, Hugo and Rehim, Abbi Abdel and Orhobor, Oghenejokpeme I and King, Ross},
  journal={Journal of Chemical Information and Modeling},
  volume={62},
  number={17},
  pages={3970--3981},
  year={2022},
  publisher={ACS Publications}
}

@article{bi2025comprehensive,
  title={A comprehensive benchmark of active learning strategies with AutoML for small-sample regression in materials science},
  author={Bi, Jinghou and Xu, Yuanhao and Conrad, Felix and Wiemer, Hajo and Ihlenfeldt, Steffen},
  journal={Scientific Reports},
  volume={15},
  number={1},
  pages={37167},
  year={2025},
  publisher={Nature Publishing Group UK London}
}

@article{warmuth2001active,
  title={Active learning in the drug discovery process},
  author={Warmuth, Manfred KK and R{\"a}tsch, Gunnar and Mathieson, Michael and Liao, Jun and Lemmen, Christian},
  journal={Advances in Neural information processing systems},
  volume={14},
  year={2001}
}

@incollection{biswas2023active,
  title={Active learning on medical image},
  author={Biswas, Angona and Abdullah Al, Nasim Md and Ali, Md Shahin and Hossain, Ismail and Ullah, Md Azim and Talukder, Sajedul},
  booktitle={Data Driven Approaches on Medical Imaging},
  pages={51--67},
  year={2023},
  publisher={Springer}
}

@article{bradley1952rank,
  title={Rank analysis of incomplete block designs: I. The method of paired comparisons},
  author={Bradley, Ralph Allan and Terry, Milton E},
  journal={Biometrika},
  volume={39},
  number={3/4},
  pages={324--345},
  year={1952},
  publisher={JSTOR}
}

@book{david1963method,
  title={The method of paired comparisons},
  author={David, Herbert Aron},
  volume={12},
  year={1963},
  publisher={London}
}

@inproceedings{chu2005preference,
  title={Preference learning with Gaussian processes},
  author={Chu, Wei and Ghahramani, Zoubin},
  booktitle={Proceedings of the 22nd international conference on Machine learning},
  pages={137--144},
  year={2005}
}

@inproceedings{canal2019active,
  title={Active embedding search via noisy paired comparisons},
  author={Canal, Gregory and Massimino, Andy and Davenport, Mark and Rozell, Christopher},
  booktitle={International Conference on Machine Learning},
  pages={902--911},
  year={2019},
  organization={PMLR}
}

@inproceedings{chumbalov2020scalable,
  title={Scalable and Efficient Comparison-based Search without Features},
  author={Chumbalov, Daniyar and Maystre, Lucas and Grossglauser, Matthias},
  booktitle={International Conference on Machine Learning},
  pages={1995--2005},
  year={2020},
  organization={PMLR}
}

@article{wang2015active,
  title={Active learning via query synthesis and nearest neighbour search},
  author={Wang, Liantao and Hu, Xuelei and Yuan, Bo and Lu, Jianfeng},
  journal={Neurocomputing},
  volume={147},
  pages={426--434},
  year={2015},
  publisher={Elsevier}
}

@inproceedings{chen2025pal,
  title={Pal: Sample-efficient personalized reward modeling for pluralistic alignment},
  author={Chen, Daiwei and Chen, Yi and Rege, Aniket and Wang, Zhi and Vinayak, Ramya Korlakai},
  booktitle={The Thirteenth International Conference on Learning Representations},
  year={2025}
}

@article{stiennon2020learning,
  title={Learning to summarize with human feedback},
  author={Stiennon, Nisan and Ouyang, Long and Wu, Jeffrey and Ziegler, Daniel and Lowe, Ryan and Voss, Chelsea and Radford, Alec and Amodei, Dario and Christiano, Paul F},
  journal={Advances in neural information processing systems},
  volume={33},
  pages={3008--3021},
  year={2020}
}

@article{zhang2022opt,
  title={Opt: Open pre-trained transformer language models},
  author={Zhang, Susan and Roller, Stephen and Goyal, Naman and Artetxe, Mikel and Chen, Moya and Chen, Shuohui and Dewan, Christopher and Diab, Mona and Li, Xian and Lin, Xi Victoria and others},
  journal={arXiv preprint arXiv:2205.01068},
  year={2022}
}

@article{li2024personalized,
  title={Personalized language modeling from personalized human feedback},
  author={Li, Xinyu and Zhou, Ruiyang and Lipton, Zachary C and Leqi, Liu},
  journal={arXiv preprint arXiv:2402.05133},
  year={2024}
}

@inproceedings{das2023near,
  title={Near optimal heteroscedastic regression with symbiotic learning},
  author={Das, Aniket and Nagaraj, Dheeraj M and Netrapalli, Praneeth and Baby, Dheeraj},
  booktitle={The Thirty Sixth Annual Conference on Learning Theory},
  pages={3696--3757},
  year={2023},
  organization={PMLR}
}

@inproceedings{chaudhuri2017active,
  title={Active heteroscedastic regression},
  author={Chaudhuri, Kamalika and Jain, Prateek and Natarajan, Nagarajan},
  booktitle={International Conference on Machine Learning},
  pages={694--702},
  year={2017},
  organization={PMLR}
}

@article{bergstra2012random,
  title={Random search for hyper-parameter optimization.},
  author={Bergstra, James and Bengio, Yoshua},
  journal={Journal of machine learning research},
  volume={13},
  number={2},
  year={2012}
}

@inproceedings{brochu2010bayesian,
  title={A Bayesian interactive optimization approach to procedural animation design},
  author={Brochu, Eric and Brochu, Tyson and De Freitas, Nando},
  booktitle={Proceedings of the 2010 ACM SIGGRAPH/Eurographics Symposium on Computer Animation},
  pages={103--112},
  year={2010}
}

@inproceedings{gonzalez2017preferential,
  title={Preferential bayesian optimization},
  author={Gonz{\'a}lez, Javier and Dai, Zhenwen and Damianou, Andreas and Lawrence, Neil D},
  booktitle={International Conference on Machine Learning},
  pages={1282--1291},
  year={2017},
  organization={PMLR}
}

@inproceedings{sadigh2017active,
title={Active preference-based learning of reward functions},
author={Sadigh, Dorsa and Dragan, Anca D and Sastry, Shankar and Seshia, Sanjit A},
booktitle={Proceedings of Robotics: Science and Systems (RSS)},
year={2017}
}

@inproceedings{biyik2018batch,
  title={Batch active preference-based learning of reward functions},
  author={Biyik, Erdem and Sadigh, Dorsa},
  booktitle={Conference on robot learning},
  pages={519--528},
  year={2018},
  organization={PMLR}
}

@article{biyik2019asking,
  title={Asking easy questions: A user-friendly approach to active reward learning},
  author={B{\i}y{\i}k, Erdem and Palan, Malayandi and Landolfi, Nicholas C and Losey, Dylan P and Sadigh, Dorsa},
  journal={arXiv preprint arXiv:1910.04365},
  year={2019}
}

@inproceedings{tucker2020preference,
  title={Preference-based learning for exoskeleton gait optimization},
  author={Tucker, Maegan and Novoseller, Ellen and Kann, Claudia and Sui, Yanan and Yue, Yisong and Burdick, Joel W and Ames, Aaron D},
  booktitle={2020 IEEE international conference on robotics and automation (ICRA)},
  pages={2351--2357},
  year={2020},
  organization={IEEE}
}

@inproceedings{li2021roial,
  title={Roial: Region of interest active learning for characterizing exoskeleton gait preference landscapes},
  author={Li, Kejun and Tucker, Maegan and B{\i}y{\i}k, Erdem and Novoseller, Ellen and Burdick, Joel W and Sui, Yanan and Sadigh, Dorsa and Yue, Yisong and Ames, Aaron D},
  booktitle={2021 IEEE International Conference on Robotics and Automation (ICRA)},
  pages={3212--3218},
  year={2021},
  organization={IEEE}
}

@inproceedings{berkenkamp2016safe,
title={Safe controller optimization for quadrotors with gaussian processes},
author={Berkenkamp, Felix and Schoellig, Angela P and Krause, Andreas},
booktitle={2016 IEEE International Conference on Robotics and Automation (ICRA)},
pages={491--496},
year={2016},
organization={IEEE}
}

@inproceedings{marco2017virtual,
  title={Virtual vs. real: Trading off simulations and physical experiments in reinforcement learning with Bayesian optimization},
  author={Marco, Alonso and Berkenkamp, Felix and Hennig, Philipp and Schoellig, Angela P and Krause, Andreas and Schaal, Stefan and Trimpe, Sebastian},
  booktitle={2017 IEEE International Conference on Robotics and Automation (ICRA)},
  pages={1557--1563},
  year={2017},
  organization={IEEE}
}

@inproceedings{csomay2022learning,
  title={Learning controller gains on bipedal walking robots via user preferences},
  author={Csomay-Shanklin, Noel and Tucker, Maegan and Dai, Min and Reher, Jenna and Ames, Aaron D},
  booktitle={2022 International Conference on Robotics and Automation (ICRA)},
  pages={10405--10411},
  year={2022},
  organization={IEEE}
}

@article{coutinho2024human,
  title={Human-in-the-loop controller tuning using Preferential Bayesian Optimization},
  author={Coutinho, Joao PL and Castillo, Ivan and Reis, Marco S},
  journal={IFAC-PapersOnLine},
  volume={58},
  number={14},
  pages={13--18},
  year={2024},
  publisher={Elsevier}
}

@inproceedings{kanayama1990stable,
  title={A stable tracking control method for an autonomous mobile robot},
  author={Kanayama, Yutaka and Kimura, Yoshihiko and Miyazaki, Fumio and Noguchi, Tetsuo},
  booktitle={Proceedings., IEEE International Conference on Robotics and Automation},
  pages={384--389},
  year={1990},
  organization={IEEE}
}

@article{farouki2012bernstein,
  title={The Bernstein polynomial basis: A centennial retrospective},
  author={Farouki, Rida T},
  journal={Computer Aided Geometric Design},
  volume={29},
  number={6},
  pages={379--419},
  year={2012},
  publisher={Elsevier}
}

@article{carpenter2017stan,
  title={Stan: A probabilistic programming language},
  author={Carpenter, Bob and Gelman, Andrew and Hoffman, Matthew D and Lee, Daniel and Goodrich, Ben and Betancourt, Michael and Brubaker, Marcus and Guo, Jiqiang and Li, Peter and Riddell, Allen},
  journal={Journal of statistical software},
  volume={76},
  pages={1--32},
  year={2017}
}
\bibliographystyle{unsrtnat}


\appendix

\section{Problem setup} \label{sec:supp-background}
To estimate the preferences of a user $w \in \mathbb{R}^d$, we assume all query items are embedded in the same $d$-dimensional space. These items may be drawn from a continuous subspace of $\mathbb{R}^d$ or restricted to a fixed dataset $D \subset \mathbb{R}^d$. Furthermore, we model preference based on distance, assuming the user prefers items that are closer to them. A query consists of presenting a pair of items $(\vp,\vq)$ to the user $\vw$ and asking which among the two they prefer. We alternate between obtaining a response $Y \in \{0,1\}$ to each query and updating our estimate of the user point. The goal is to choose these queries adaptively.

Suppose that the user is denoted by the random variable $W \in \R^d$ with a prior density of $p_0(W)$. The posterior distribution after obtaining responses to $n$ queries is given by
\begin{equation*}\label{eq:posterior} 
    p(W) \propto p_0(W) \prod_{j=1}^{n}p(Y_j|W) 
\end{equation*}
where $\{ Y_j\}_{j=1}^{n}$ denote the set of all responses obtained for $n$ queries involving pairs of items - $\{(\vp_1, \vq_1), (\vp_2, \vq_2), \dots, (\vp_n, \vq_n)\}$ respectively. The user point is estimated as the mean of the posterior distribution
\[ \hat{\vw} = \E [W | Y^n]. \]
In the experiments, we use MCMC sampling to generate samples from the posterior distribution.

\section{Probability Model and Query Synthesis} \label{sec:supp-query_details}
\subsection{Derivation of the Anisotropic Denominator}

We define the points $\vp$ and $\vq$ in terms of their midpoint $\vb$ and the hyperplane normal vector $\va$. Following the definitions in Sec.~\ref{sec:modified_response_model}, we have $\va = 2(\vp - \vq)$ and $\vb = \frac{\vp + \vq}{2}$. This implies
\begin{equation*}
    \vp = \vb + \frac{\va}{4}, \quad \vq = \vb - \frac{\va}{4}
\end{equation*}
Let the squared distances from the user point $\vw$ to items $\vp$ and $\vq$ be expanded around the midpoint $\vb$
\begin{align*}
    \|\vw - \vp\|_2^2 &= \left\|(\vw - \vb) - \frac{\va}{4}\right\|_2^2 = \|\vw - \vb\|_2^2 + \frac{\|\va\|_2^2}{16} - \frac{\va^\top (\vw - \vb)}{2} \\
    \|\vw - \vq\|_2^2 &= \left\|(\vw - \vb) + \frac{\va}{4}\right\|_2^2 = \|\vw - \vb\|_2^2 + \frac{\|\va\|_2^2}{16} + \frac{\va^\top (\vw - \vb)}{2}
\end{align*}
To simplify the notation, let
\begin{equation*}
    U = \|\vw - \vb\|_2^2 + \frac{\|\va\|_2^2}{16}, \quad V = \frac{\va^\top (\vw - \vb)}{2}
\end{equation*}
Substituting these into the denominator expression $D(\vw) = \sigma_0 \sqrt{\|\vw - \vp\|_2^4 + \|\vw - \vq\|_2^4}$
\begin{align*}
    D(\vw) &= \sigma_0 \sqrt{(U - V)^2 + (U + V)^2} \\
    D(\vw) &= \sigma_0 \sqrt{(U^2 + V^2 - 2UV) + (U^2 + V^2 + 2UV)} \\
    D(\vw) &= \sigma_0 \sqrt{2U^2 + 2V^2}
\end{align*}
Substituting the full expressions for $U$ and $V$ back into the result yields the final non-isotropic form
\begin{equation*}
    D(\vw) = \sigma_0 \sqrt{2 \left( \|\vw - \vb\|_2^2 + \frac{\|\va\|_2^2}{16} \right)^2 + \frac{1}{2} \left( \va^\top (\vw - \vb) \right)^2}
\end{equation*}

\subsection{Approximating distribution of $f(W)$ as a Gaussian}

Let $W \sim \NN(\vmu, \mSigma)$ and $f(W) = \frac{N(W)}{D(W)}$ where $N(W) = \va^\top W - \tau$ and \\
$D(W) =  \sigma_0 \sqrt{2 \left( \|W - \vb\|_2^2 + \frac{\|\va\|_2^2}{16} \right)^2 + \frac{1}{2} \left( \va^\top (\vw - \vb) \right)^2}$. \\
\\
We use moment matching for the approximation: 
\begin{equation*}
    f(W) \sim \NN(\mu_f, \sigma_f^2) \text{ where } \mu_f = \E[f(W)] \text{ and } \sigma_f^2 = \var(f(W)).
\end{equation*}
The expected probability can be expressed as 
\[ \E[\Phi(f(W))] \approx \Phi \left( \frac{\mu_f}{g(\sigma_f)} \right) \]
where $\Phi(x)$ is the noise CDF with the symmetry $\Phi(x) + \Phi(-x) = 1$ and $g(\cdot)$ is an appropriate function of $\sigma_f$ depending on the specific $\Phi$. \\
\\
From the equiprobable outcome condition, we have $\E[\Phi(f(W))] = 0.5$.
A key characteristic of symmetric link functions such as $\Phi$ is that
\begin{equation*}
    \Phi(0) = 0.5 \quad \text{and} \quad \Phi(x) + \Phi(-x) = 1.
\end{equation*}
From this, it follows that for any symmetric link function
\begin{equation*}
    \E[\Phi(f(W))] = 0.5 \iff \mu_f = 0.
\end{equation*}
For $\E[f(W)]$ to be $0$, the numerator $N(W)$ needs to be an odd function about $\vmu$ and the denominator $D(W)$ an even function about $\vmu$. We examine both the terms by substituting $W = \vmu + \Delta$. \\
\\
We have, 
\[ N(\vmu + \Delta) = \va^\top (\vmu + \Delta) - \tau = (\va^\top \vmu - \tau) + \va^\top \Delta. \]
For $N(W)$ to be odd about $\vmu$, the constant term should be $0$. We have
\begin{align*}
    \va^\top \vmu &- \tau = 0 \\
    \text{i.e.} ~2(\vp - \vq)^\top \vmu &- \left( \|\vp\|^2 - \|\vq\|^2 \right) = 0 \\
    \implies \vmu &= \frac{\vp + \vq}{2} = \vb.
\end{align*}
When this holds, $N(\vmu + \Delta) = \va^\top \Delta$ which satisfies
\begin{equation*}
    N(\vmu - \Delta) = -N(\vmu + \Delta).
\end{equation*}
Further, we have
\begin{equation*}
    D(W) = \sigma_0 \sqrt{2 \left( \|W - \vmu\|_2^2 + \frac{\|\va\|_2^2}{16} \right)^2 + \frac{1}{2} \left( \va^\top (W - \vmu) \right)^2}.
\end{equation*}
Under the condition, we substitute $W = \vmu + \Delta$ to obtain
\begin{equation*}
    D(\vmu + \Delta) = \sigma_0 \sqrt{2 \left( \|\Delta\|_2^2 + \frac{\|\va\|_2^2}{16} \right)^2 + \frac{1}{2} \left( \va^\top \Delta \right)^2},
\end{equation*}
thus making the denominator an even function of $\Delta$ where
\begin{equation*}
    D(\vmu - \Delta) = D(\vmu + \Delta).
\end{equation*}
For $\vp = \vmu + \vx$ and $\vq = \vmu - \vx$, the approximation of the expectation becomes exact due to symmetry as derived in the main paper.

\subsection{Optimizing expected entropy} \label{sec:opt_second_term}

\begin{proposition}
Let $S \sim \NN(\vzero, \mSigma)$ with $\mSigma \in \mathbb{R}^{d \times d}$ positive definite,
and let $\vv_1$ denote the principal eigenvector of $\mSigma$ with eigenvalue $\lambda_1$.
The maximizer of
\begin{equation}
    F(\vx) = \E_{S \sim \NN(\vzero,\mSigma)} \left[ \frac{4|\vx^\top S|}{\sigma_0\sqrt{2\left(\|S\|_2^2+\|\vx\|_2^2\right)^2 + 8\left(\vx^\top S\right)^2}} \right]
\end{equation}
has the form $\tilde{\vx} = \tilde{r} \vv_1$, where
\begin{equation*}
    \tilde{r} = \underset{r}{\arg \max} ~\underset{S}{\E} \left[ H\left( \Phi \left( \frac{4 r ~\vv_1^\top S}
    {\sigma_0 \sqrt{2\left(\|S\|_2^2 + r^2\right)^2 + 8r^2 \left(\vv_1^\top S\right)^2}} \right) \right) \right].
\end{equation*}
\end{proposition}

\begin{proof}

\textbf{Step 1: Optimal direction.}
We write $\vx = r \vu$ with $r > 0$ and $\|\vu\|_2 = 1$.
For a fixed realization of $S$, let $a := \|S\|_2^2$ and $t := (\vu^\top S)^2$.
Let the integrand be
\begin{equation*}
    f(\vu, S) = \frac{4 r\sqrt{t}} {\sigma_0 \sqrt{2(a + r^2)^2 + 8r^2 t}}.
\end{equation*}
Treating $a$ and $r$ as constants, differentiate $f^2$ with respect to $t$:
\begin{equation*}
    \frac{d}{dt}\left[\frac{16 t}{2(a+r^2)^2 + 8r^2 t}\right] = \frac{32(a+r^2)^2}{\left[2(a+r^2)^2 + 8r^2 t\right]^2} > 0.
\end{equation*}
Hence $f(\vu, S)$ is strictly increasing in $|\vu^\top S|$ for every fixed $S$ and $r > 0$.
Since expectation preserves pointwise monotone ordering,
\begin{equation*}
    \underset{\|\vu\|_2=1}{\arg\max} \E[f(\vu,S)] = \underset{\|\vu\|_2=1}{\arg\max} \E \left[\varphi \left(|\vu^\top S|\right)\right]
\end{equation*}
for some increasing function $\varphi$. For $S \sim \NN(\vzero,\mSigma)$,
\begin{equation*}
    \vu^\top S \sim \NN(\vzero,\, \vu^\top \mSigma \vu), \quad\text{so}\quad |\vu^\top S| \stackrel{d}{=} \sqrt{\vu^\top \mSigma \vu} |C|, \quad C \sim \NN(0,1).
\end{equation*}
Therefore,
\begin{equation*}
    \E \left[\varphi \left(|\vu^\top S|\right)\right] = \E \left[\varphi \left(\sqrt{\vu^\top \mSigma \vu}\,|C|\right)\right],
\end{equation*}
which is increasing in $\vu^\top \mSigma \vu$. The term $\vu^\top \mSigma \vu$ is maximized over the unit sphere at $\vu = \vv_1$, so
\begin{equation*}
    \vu^\star = \vv_1.
\end{equation*}

\textbf{Step 2: Optimal magnitude via 1D optimization.}
With $\vu = \vv_1$ fixed, define $Z \sim \NN(0, \bI_d)$ via the spectral decomposition $S = U\Lambda^{1/2}Z$, so that
\begin{equation*}
    \vv_1^\top S = \sqrt{\lambda_1}\,Z_1, \qquad \|S\|_2^2 =: Q = \sum_{i=1}^d \lambda_i Z_i^2.
\end{equation*}
The reduced 1D objective is
\begin{equation*}
    \phi(r) = \frac{4r\sqrt{\lambda_1}}{\sigma_0} \mathbb{E} \left[ \frac{|Z_1|}{\sqrt{2(Q + r^2)^2 + 8r^2\lambda_1 Z_1^2}} \right].
\end{equation*}
Computing the optimal magnitude via 1D numerical optimization of $\arg \max ~\phi(r)$ results in $\tilde{r}$ being independent of $\sigma_0$. In practice, we observe that the optimal magnitude varies with $\sigma_0$ and hence we compute it as
\begin{equation}
    \tilde{r} = \underset{r}{\argmax} ~\underset{S}{\E} \left[  H\left( \Phi (f(r \vv_1, S)) \right) \right]
\end{equation}
initialized at $r_0 = \sqrt{\operatorname{tr}(\Sigma)}$. \\
\\
\textbf{Conclusion.}
Combining both steps, the maximizer of $F(\vx)$ is
\begin{equation*}
    \boxed{\tilde{\vx} = \tilde{r} \vv_1,}
\end{equation*}
where $\tilde{r}$ is computed via 1D optimization.

\end{proof}

In order to get an intuition for the optimal magnitude, let's try to compute its approximate value. We have $\E[|\vx^\top S|] = \sqrt{\frac{2 \lambda_1}{\pi}} r $, and $\E[\|S\|^2 + \|\vx\|^2] = \text{Tr}(\Sigma) + r^2 $, and thus
\begin{equation*}
    \tilde{|\vx|} \approx \underset{|\vx|}{\argmax~} \frac{\E [4|\vx^\top S|]}{\E \left[\sigma_0 \sqrt{2 \left( \|S\|^2 + r^2 \right)^2 + 8 \left( \vx^\top S \right)^2}\right]} \approx \underset{r}{\argmax~} \frac{4 \sqrt{\frac{2 \lambda_1}{\pi}} r}{\sqrt{2 \left( \text{Tr}(\Sigma) + r^2 \right)^2 + \frac{16 \lambda_1}{\pi}r^2 }} = m(r)
\end{equation*}
where $\lambda_1$ is the principal eigenvalue of $\mSigma$. \\
\\
To find the optimal value of $r$, we differentiate $m^2$ with respect to $r$ and equate it to 0 to obtain 
\begin{equation*}
    r = \|\tilde{\vx}\| = \sqrt{\text{Tr}(\Sigma)}.
\end{equation*}
The magnitude $\|\vx\|$ now depends on the Tr$(\Sigma)$ which represents the total variance across all dimensions of the random vector $W$. As the posterior shrinks, so does $\|\vx\|$ and hence, the selected points move closer to $\vmu$. 
In practice, we observe that the optimal magnitude varies with $\sigma_0$ and thus, use 1D optimization to compute its value. The principal eigenvector $\vv_1$ is still the optimal direction. 

\section{Calculating Mahalanobis distances} \label{sec:supp-mahalanobis_distance}
Consider 
\[ I(W; Y | (\vp,\vq)) = H(Y | (\vp,\vq)) - H(Y | W, (\vp,\vq)). \]
\\
For $\pi = \E_W [\P(Y=1|W, (\vp,\vq))]$, we have
\begin{align*}
    H(Y | (\vp,\vq)) &= H \left(\E_W [Y|W, (\vp,\vq)] \right) \\
    &= H(\pi) \\
    &= -\pi \log \pi - (1-\pi) \log(1-\pi).
\end{align*}
Further, for $p = \P(Y=1|W, (\vp,\vq))$, we have
\begin{align*}
    -H(Y|W, (\vp,\vq)) &= -\E_W [H(Y|W, (\vp,\vq))] \\
    &= -\E_W [-p \log p - (1-p) \log(1-p)] \\
    &= \E_W [p \log p + (1-p) \log(1-p)].
\end{align*}
The objective function can now be written as
\begin{equation}
    I(\vp,\vq) = H(\pi) + \E_W \left[ g(f(\vw)) \right]
\end{equation}
where
\begin{enumerate}
    \item 
    \begin{equation*}
        f(\vw) = \frac{\quad \|\vw - \vq\|_2^2 - \|\vw - \vp\|_2^2}{ \sigma_0\sqrt{ \|\vw - \vq\|_2^4 +  \|\vw - \vp\|_2^4}}
    \end{equation*}
    \item The link function and entropy-related term $g$, where $\Phi(x)$ is noise distribution CDF
    \begin{equation*}
        g(f) = \Phi(f) \log(\Phi(f)) + \Phi(-f) \log(\Phi(-f))
    \end{equation*}
\end{enumerate}

\subsection{Gradient Derivation}
The gradient with respect to $\vp$ is obtained via the chain rule
\begin{align*}
    \gradp I(\vp,\vq) &= \frac{dH}{d\pi} \gradp \pi + \gradp \left(\E_W \left[g(f(\vw)) \right] \right) \\
    &= \frac{dH}{d\pi} \gradp \pi + \E_{W} \left[ \frac{dg}{df} \gradp f(\vw) \right]
\end{align*}

\subsubsection*{Derivation of $\frac{dH}{d\pi}$}
\begin{equation*}
    \frac{dH}{d\pi} = \log \left( \frac{1 - \pi}{\pi} \right)
\end{equation*}

\subsubsection*{Derivation of $\gradp \pi$}
\begin{equation*}
    \gradp \pi = \E_{W} \left[ \Phi'(f) \gradp f(\vw) \right]
\end{equation*}

\subsubsection*{Derivation of $\frac{dg}{df}$}
We use the property $\Phi'(-f) = \Phi'(f)$.
\begin{align*}
    \frac{dg}{df} &= \left[ \Phi'(f)\log(\Phi(f)) + \Phi(f)\frac{\Phi'(f)}{\Phi(f)} \right] - \left[ \Phi'(-f)\log(\Phi(-f)) + \Phi(-f)\frac{\Phi'(-f)}{\Phi(-f)} \right] \\
    &= \Phi'(f)\log(\Phi(f)) + \Phi'(f) - \Phi'(-f)\log(\Phi(-f)) - \Phi'(-f) \\
    &= \Phi'(f) \left[ \log(\Phi(f)) - \log(\Phi(-f)) \right] \\
    &= \Phi'(f) \log\left(\frac{\Phi(f)}{\Phi(-f)}\right).
\end{align*}
Thus, we have
\begin{equation}
    \frac{dg}{df} = \Phi'(f) \log\left(\frac{\Phi(f)}{\Phi(-f)}\right).
\end{equation}
\subsubsection*{Derivation of $\gradp f(\vw)$}
Let $A = \|\vw - \vq\|^2$, $B = \|\vw - \vp\|^2$, $S = \sqrt{A^2 + B^2}$, 
$N = A - B$ and $D = \sigma_0 S$. \\
\\
We have
\begin{equation*}
    \gradp N = 2(\vw-\vp)
\end{equation*}
\begin{equation}\label{eq:gradp_s}
    \gradp S = \frac{1}{2S} \gradp (A^2 + B^2) = \frac{B}{S}\left(-2(\vw-\vp)\right)
\end{equation}
Using the quotient rule $\gradp(N/D) = ((\gradp N) D - N (\gradp D))/D^2$, we get
\begin{align*}
    \gradp f &= \frac{2(\vw-\vp) D - N (B/S)(-2\sigma_0(\vw-\vp))}{D^2} \\
    &= \frac{2(\vw-\vp)}{S D^2} \left[ D S + \sigma_0 N B \right] \\
    &= \frac{2(\vw-\vp)}{S D^2} \left[ \sigma_0S^2 + \sigma_0 NB \right] \\
    &= \frac{2(\vw-\vp)}{\sigma_0 S^3} \left[ A^2 + B^2 + (A - B)B \right] \\
    &= \frac{2(\vw-\vp)}{\sigma_0 S^3} \left[ A(A + B) \right].
\end{align*}
Thus, we have
\begin{equation}
    \gradp f = \frac{2(\vw-\vp)}{\sigma_0 S^3} \left[ A(A + B) \right].
\end{equation}
By symmetry, the gradient with respect to $\vq$ is
\begin{equation}
    \gradq f = \frac{-2(\vw-\vq)}{\sigma_0 S^3} \left[ B(A + B) \right].
\end{equation}

\subsubsection*{Final Expression}
Combining the results, we obtain
\begin{equation*}
    \gradp I(\vp,\vq) = \log \left( \frac{1 - \pi}{\pi} \right) \E_{W} \left[ \Phi'(f) \gradp f(\vw) \right] + \E_{W} \left[ \Phi'(f) \log\left(\frac{\Phi(f)}{\Phi(-f)}\right)\gradp f(\vw) \right]
\end{equation*}
\begin{equation}
    \gradp I(\vp,\vq) = \E_{W} \left[ \left( \log \left( \frac{1 - \pi}{\pi} \right) + \log\left(\frac{\Phi(f)}{\Phi(-f)}\right) \right) \Phi'(f) \gradp f(\vw) \right]
\end{equation}
\begin{equation}
    \gradq I(\vp,\vq) = \E_{W} \left[ \left( \log \left( \frac{1 - \pi}{\pi} \right) + \log\left(\frac{\Phi(f)}{\Phi(-f)}\right) \right) \Phi'(f) \gradq f(\vw) \right].
\end{equation}

\subsection{Hessian Derivation}

The full Hessian of the mutual information $I(\vz)$ with respect to the parameter vector $\vz = (\vp, \vq)$ is a $2D \times 2D$ block matrix
\[
\mH(\vz) = \nabla_{\vz}^2 I = 
\begin{pmatrix} 
    \hessp I & \gradq (\gradp I) \\
    \gradp (\gradq I) & \hessq I 
\end{pmatrix}.
\]

\subsection*{Diagonal Blocks}

We first derive the diagonal blocks - $\hessp I$ and $\hessq I$. \\
\\
We have
\begin{align*}
    \gradp I(\vp,\vq) &= \log \left( \frac{1 - \pi}{\pi} \right) \E_{W} \left[ \Phi'(f) \gradp f(\vw) \right] + \E_{W} \left[ \Phi'(f) \log\left(\frac{\Phi(f)}{\Phi(-f)}\right)\gradp f(\vw) \right] \\
    &= \log \left( \frac{1 - \pi}{\pi} \right) \gradp \pi + \gradp \E_W \left[ g(f(\vw)) \right] \\
    &= G_{\vp1} + G_{\vp2}.
\end{align*}
The Hessian with respect to $\vp$ is found by differentiating the gradient. \\
\\
Using the product rule for vector calculus, we obtain
\begin{equation*}
    \gradp G_{\vp1} = (\gradp \pi) \left( \gradp \log \left( \frac{1 - \pi}{\pi} \right) \right)^\top + \log \left( \frac{1 - \pi}{\pi} \right) \hessp \pi
\end{equation*}
where
\begin{equation*}
    \gradp \log \left( \frac{1 - \pi}{\pi} \right) = - \frac{1}{\pi (1 - \pi)} \gradp \pi
\end{equation*}
and
\begin{equation*}
    \hessp \pi = \E_{W} \left[\Phi''(f) \gradp f \left( \gradp f \right)^\top + \Phi'(f) \hessp f \right].
\end{equation*}
Thus, we have
\begin{equation}
    \gradp G_{\vp1} = - \frac{1}{\pi (1 - \pi)} (\gradp \pi) (\gradp \pi)^\top + \log \left( \frac{1 - \pi}{\pi} \right) \E_{W} \left[\Phi''(f) \gradp f \left( \gradp f \right)^\top + \Phi'(f) \hessp f \right]
\end{equation}
and
\begin{align*}
    \gradp G_{\vp2} &= \hessp \E_W \left[ g(f(\vw)) \right] \\
    \\
    &= \gradp \E_{W} \left[ \frac{dg}{df} \gradp f(\vw) \right]\\
    \\
    &= \E_{W} \left[\gradp \left( \frac{dg}{df} \right) (\gradp f)^T + \frac{dg}{df} \hessp f \right]\\
    \\
    &= \E_{W} \left[ \frac{d^2g}{df^2} (\gradp f)(\gradp f)^T + \frac{dg}{df} \hessp f \right]
\end{align*}

\subsubsection*{Derivation of $\frac{d^2g}{df^2}$}

We differentiate $\frac{dg}{df} = \log\left(\frac{\Phi(f)}{\Phi(-f)}\right) \cdot \Phi'(f)$ using the product rule as below.
\begin{align*}
    \frac{d^2g}{df^2} &= \frac{d}{df}\left(\log\left(\frac{\Phi(f)}{\Phi(-f)}\right) \cdot \Phi'(f) \right) \\
    &= \frac{d}{df}\left(\Phi'(f) \cdot [\log (\Phi(f)) - \log({\Phi(-f)})] \right) \\
    &= \Phi'(f) \left[ \frac{\Phi'(f)}{\Phi(f)} - \frac{-\Phi'(f)}{\Phi(-f)} \right] + \log\left(\frac{\Phi(f)}{\Phi(-f)}\right) \cdot \Phi''(f) \\
    &= \left(\Phi'(f)\right)^2 \left[ \frac{\Phi(-f) + \Phi(f)}{\Phi(f) \Phi(-f)} \right] + \log\left(\frac{\Phi(f)}{\Phi(-f)}\right) \cdot \Phi''(f).
\end{align*}
Thus, we have
\begin{equation}
    \frac{d^2g}{df^2} = \frac{\left(\Phi'(f)\right)^2} {\Phi(f) \Phi(-f)} + \log\left(\frac{\Phi(f)}{\Phi(-f)}\right) \cdot \Phi''(f).
\end{equation}

\subsubsection*{Derivation of $\hessp f(\vw)$}
Let $C_{\vp} = \frac{2A(A+B)}{\sigma_0 S^3}$, so that $\gradp f = C_{\vp}(\vw-\vp)$. \\
\\
We use the product rule $\hessp f = (\vw-\vp)(\gradp C_{\vp})^T + C_{\vp} \gradp(\vw-\vp)$. \\
\\
First, $\gradp(\vw-\vp) = -\bI$, where $\bI$ is the identity matrix. \\
\\ 
Next, we find the gradient of the scalar $C_{\vp}$ with respect to $\vp$. We have
\begin{align*}
    \gradp C_{\vp} = \gradp \left( \frac{2A(A+B)}{\sigma_0 S^3} \right) = \gradp \left( \frac{U}{V} \right).
\end{align*}
Further
\begin{equation*}
    \gradp U = 2A \gradp B = -4 A (\vw - \vp)
\end{equation*}
and
\begin{equation*}
    \gradp V = 3 \sigma_0 S^2 \gradp S.
\end{equation*}
Substituting $\gradp S$ from \eqref{eq:gradp_s}, we get
\begin{equation*}
    \gradp V = 3 \sigma_0 S^2 \frac{B}{S}(-2(\vw - \vp)) = -6 \sigma_0 SB (\vw - \vp).
\end{equation*}
Now, we have
\begin{align*}
    \gradp C_{\vp} &= \frac{\sigma_0 S^3 (-4 A (\vw - \vp)) - 2A(A+B) (-6 \sigma_0 SB (\vw - \vp))}{\left( \sigma_0 S^3 \right)^2} \\
    &= \frac{ 4 \sigma_0 SA (\vw - \vp) \left[ -S^2 + 3B(A+B)\right]}{\sigma_0^2 S^6 } \\
    &= \frac{ 4 A(\vw - \vp) \left[ 3B(A+B) - S^2 \right]}{\sigma_0 S^5 } 
\end{align*}
Combining the terms gives the Hessian of $f$
\begin{align*}
    \hessp f &= (\vw-\vp) \left(\frac{ 4 A(\vw - \vp) \left[ 3B(A+B) - S^2 \right]}{\sigma_0 S^5} \right)^\top - \frac{2A(A+B)}{\sigma_0 S^3} \bI \\
    \\
    &= \frac{2A}{\sigma_0 S^3} \left[ \frac{2 (3B(A+B) - S^2)}{S^2} (\vw-\vp) (\vw-\vp)^\top - (A+B)\bI \right] \\
    \\
    &= \frac{2A}{\sigma_0 S^3} \left[ \frac{2 T_{\vp}}{S^2} (\vw-\vp) (\vw-\vp)^\top - (A+B)\bI \right]
\end{align*}
where $T_{\vp} = (3B(A+B) - S^2)$. \\
\\
Finally, we get
\begin{align*}
    \gradp G_{\vp2} = \E_{W} \left[ \left( \frac{\left(\Phi'(f)\right)^2} {\Phi(f) \Phi(-f)} + \log\left(\frac{\Phi(f)}{\Phi(-f)}\right) \cdot \Phi''(f) \right) (\gradp f)(\gradp f)^T + \left( \Phi'(f) \log\left(\frac{\Phi(f)}{\Phi(-f)}\right) \right) \hessp f \right].
\end{align*}

\subsubsection*{Derivation of $ \gradq G_{\vq2}$}
Similarly, we have
\begin{equation}
    \gradq G_{\vq2} = \E_{W} \left[ \frac{d^2g}{df^2} (\gradq f)(\gradq f)^T + \frac{dg}{df} \hessq f \right]
\end{equation}
with
\begin{align*}
    \hessq f &= \left(-(\vw - \vq)\right)(\gradq C_{\vq})^T + C_{\vq} \gradq\left(-(\vw - \vq)\right) \\
    &=\left(-(\vw - \vq)\right)(\gradq C_{\vq})^T + C_{\vq} \bI
\end{align*}
where
\begin{equation*}
    C_{\vq} = \frac{2B(A+B)}{\sigma_0 S^3}
\end{equation*}
and
\begin{equation*}
    \gradq C_{\vq} = \frac{ 4 B(\vw - \vq) \left[ 3A(A+B) - S^2 \right]}{\sigma_0 S^5 }. 
\end{equation*}
Thus, we have
\begin{align*}
    \hessq f &= \left(-(\vw - \vq)\right)) \left( \frac{ 4 B(\vw - \vq) \left[ 3A(A+B) - S^2 \right]}{\sigma_0 S^5 } \right)^\top + \frac{2B(A+B)}{\sigma_0 S^3} \bI \\
    \\
    &= \frac{2B}{\sigma_0 S^3} \left[(A+B)\bI - \frac{2 (3A(A+B) - S^2)}{S^2} (\vw - \vq) (\vw - \vq)^\top \right] \\
    \\
    &= \frac{2B}{\sigma_0 S^3} \left[(A+B)\bI - \frac{2 T_{\vq}}{S^2} (\vw - \vq) (\vw - \vq)^\top \right]
\end{align*}
where $T_{\vq} = (3A(A+B) - S^2)$.

\subsubsection*{Final Expression}
Assembling all components gives the final expression for the Hessian
\begin{align}
    \hessp I &=  \gradp G_{\vp1} + \gradp G_{\vp2} \\
    &= - \frac{1}{\pi (1 - \pi)} \left( \gradp \pi \right) \left( \gradp \pi \right)^\top + \E_{W} \left[\Psi(\vw) \left( \gradp f \right) \left( \gradp f \right)^\top + \Omega(\vw) \hessp f \right]
\end{align}
and
\begin{equation}
    \hessq I = - \frac{1}{\pi (1 - \pi)} \left( \gradq \pi \right) \left( \gradq \pi \right)^\top + \E_{W} \left[\Psi(\vw) \left( \gradq f \right) \left( \gradq f \right)^\top + \Omega(\vw) \hessq f \right]
\end{equation}
where 
\begin{equation*}
    \Psi(\vw) = \frac{\left(\Phi'(f)\right)^2} {\Phi(f) \Phi(-f)} + \log\left(\frac{ (1 - \pi) \Phi(f)}{\pi \Phi(-f)}\right) \cdot \Phi''(f),
\end{equation*}
\\
\begin{equation*}
    \Omega(\vw) = \log\left(\frac{ (1 - \pi) \Phi(f)}{\pi \Phi(-f)}\right) \Phi'(f)
\end{equation*}
\\
and $\gradp \pi$, $\gradq \pi$, $\gradp f$, $\gradq f$, $\hessp f$ and $\hessq f$ are the expressions derived above.

\subsection*{Off-Diagonal Hessian Blocks}

We will now derive the off-diagonal block $\gradq (\gradp I)$. We have
\begin{equation*}
\gradq \left( \gradp I \right) = \gradq \left( G_{\vp1} + G_{\vp2} \right).
\end{equation*}
Further
\begin{equation*}
    \gradq G_{\vp1} = (\gradp \pi) \left( \gradq \log \left( \frac{1 - \pi}{\pi} \right) \right) ^\top + \log \left( \frac{1 - \pi}{\pi} \right) \gradq (\gradp \pi).
\end{equation*}
where
\begin{equation*}
    \gradq \log \left( \frac{1 - \pi}{\pi} \right) = - \frac{1}{\pi (1 - \pi)} \gradq \pi
\end{equation*}
and
\begin{equation*}
    \gradq (\gradp \pi) = \E_{W} \left[\Phi''(f) \gradp f \left( \gradq f \right)^\top + \Phi'(f) \gradq (\gradp f) \right]
\end{equation*}
Next,
\begin{align*}
\gradq G_{\vp2} &= \gradq \left( \E_{W}\left[\frac{dg}{df}\gradp f\right] \right) \\
&= \E_{W}\left[ \gradq \left( \frac{dg}{df} \cdot \gradp f \right) \right] \\
&= \E_{W}\left[(\gradp f) \left(\gradq \frac{dg}{df}\right) ^T + \left(\frac{dg}{df}\right) \left(\gradq (\gradp f)\right) \right]
\end{align*}

\subsubsection*{Derivation of $\gradq \left(\frac{dg}{df}\right)$}
We use the chain rule. $\frac{dg}{df}$ is a scalar function of $f$, which is a function of $\vq$.
\begin{equation}
    \gradq \left(\frac{dg}{df}\right) = \frac{d}{df}\left(\frac{dg}{df}\right) \cdot \gradq f = \left(\frac{d^2 g}{df^2}\right) \gradq f
\end{equation}
We have already derived $\frac{d^2 g}{df^2}$ and $\gradq f$.

\subsubsection*{Derivation of $\gradq (\gradp f)$}
This is the most complex term. We start with the expression for $\gradp f$:
\begin{equation*}
    \gradp f = C_{\vp}(\vw-\vp), \quad \text{where} \quad C_{\vp}=\frac{2A(A+B)}{\sigma_0 S^3}
\end{equation*}
We differentiate $\gradp f$ with respect to $\vq$ using the product rule:
\begin{align*}
\gradq (\gradp f) &= \gradq \left( C_{\vp} (\vw-\vp) \right) \\
&= (\vw-\vp)(\gradq C_{\vp})^T + C_{\vp} \underbrace{\left( \gradq (\vw-\vp) \right)}_{\text{= } \mzero} \\
&= (\vw-\vp)(\gradq C_{\vp})^T
\end{align*}
The second term is the zero matrix because $(\vw-\vp)$ is not a function of $\vq$. Next, we find the gradient of the scalar $C_{\vp}$ with respect to $\vq$. We have
\begin{align*}
    \gradq C_{\vp} = \gradq \left( \frac{2A(A+B)}{\sigma_0 S^3} \right) = \gradq \left( \frac{U}{V} \right).
\end{align*}
Further
\begin{align*}
    \gradq U &= 2 \gradq (A^2 + AB) \\
    &= 2 (2A + B) \gradq A = -4 (2A + B) (\vw - \vq)
\end{align*}
and
\begin{equation*}
    \gradq V = 3 \sigma_0 S^2 \gradq S.
\end{equation*}
Substituting $\gradq S$, we get
\begin{equation*}
    \gradq V = 3 \sigma_0 S^2 \frac{A}{S}(-2(\vw - \vq)) = -6 \sigma_0 SA (\vw - \vq).
\end{equation*}
Now, we have
\begin{align*}
    \gradq C_{\vp} &= \frac{\sigma_0 S^3 (-4 (2A + B) (\vw - \vq)) - 2A(A+B) (-6 \sigma_0 SA (\vw - \vq))}{\left( \sigma_0 S^3 \right)^2} \\
    &= \frac{ 4 \sigma_0 S (\vw - \vq) \left[ -S^2 (2A + B) + 3A^2(A+B)\right]}{\sigma_0^2 S^6 } \\
    &= \frac{ 4 (\vw - \vq) \left[ 3A^2(A+B) - S^2 (2A + B) \right]}{\sigma_0 S^5 } 
\end{align*}
Substituting this back, the mixed-partial Hessian of $f$ is
\begin{align*}
    \gradq (\gradp f) &= \left(\frac{ 4 \left[ 3A^2(A+B) - S^2 (2A + B) \right]}{\sigma_0 S^5 }\right) (\vw-\vp)(\vw - \vq)^\top \\
    \\
    &= \frac{ 4 T_{\vp \vq}}{\sigma_0 S^5 } (\vw-\vp)(\vw - \vq)^\top
\end{align*}
where $T_{\vp \vq} = 3A^2(A+B) - S^2 (2A + B)$.
This is a scalar multiplied by a $D \times D$ outer product matrix.\\
\\
We plug the expressions back into the main equation for $\gradq (\gradp I)$
\begin{equation}
    \gradq G_{\vp2} = \E_{W}\left[ \left(\frac{d^2 g}{df^2}\right) (\gradp f) (\gradq f)^T + \left(\frac{dg}{df}\right) (\gradq (\gradp f)) \right]
\end{equation}

\subsubsection*{Final Expression}
We plug the expressions back into the main equation for $\gradq (\gradp I)$
\begin{equation}
    \gradq (\gradp I) = - \frac{1}{\pi (1 - \pi)} \left( \gradp \pi \right) \left( \gradq \pi \right)^\top + \E_{W} \left[\Psi(\vw) \left( \gradp f \right) \left( \gradq f \right)^\top + \Omega(\vw) \gradq (\gradp f) \right]
\end{equation}
\\
The other off-diagonal block is its transpose: $\gradp (\gradq I) = (\gradq (\gradp I))^T$, i.e.
\begin{equation}
    \gradp (\gradq I) = \nabla = - \frac{1}{\pi (1 - \pi)} \left( \gradq \pi \right) \left( \gradp \pi \right)^\top + \E_{W} \left[\Psi(\vw) \left( \gradq f \right) \left( \gradp f \right)^\top + \Omega(\vw) \gradp (\gradq f) \right]
\end{equation}

\subsection{Summary}

The terms of the Hessian are as below
\begin{equation*}
    \nabla_{\vz}^2 I = 
\begin{pmatrix} 
    \hessp I & \gradq (\gradp I) \\
    \gradp (\gradq I) &  \hessq I 
\end{pmatrix}.
\end{equation*}
Let $A = \|\vw-\vq\|^2$, $B = \|\vw-\vp\|^2$, and $S = \sqrt{A^2 + B^2}$. We have
\begin{equation*}
    \hessp I(\vz^*) = - \frac{1}{\pi (1 - \pi)} \left( \gradp \pi \right) \left( \gradp \pi \right)^\top + \E_{W} \left[\Psi(\vw) \left( \gradp f \right) \left( \gradp f \right)^\top + \Omega(\vw) \hessp f \right]
\end{equation*}
where $f = f(\vw)$ as defined in \eqref{eq:func_f_def}, 
\begin{equation*}
    \pi = \E_{W} \left[ \P\left( Y = 1|W, (\vp, \vq) \right) \right], \quad
    \gradp \pi = \E_{W} \left[ \Phi'(f) \gradp f(\vw) \right],
\end{equation*}
\begin{equation*}
    \Psi(\vw) = \frac{\left(\Phi'(f)\right)^2} {\Phi(f) \Phi(-f)} + \log\left(\frac{ (1 - \pi) \Phi(f)}{\pi \Phi(-f)}\right) \cdot \Phi''(f),
\end{equation*}
\begin{equation*}
    \Omega(\vw) = \log\left(\frac{ (1 - \pi) \Phi(f)}{\pi \Phi(-f)}\right) \Phi'(f), \quad  \gradp f = \frac{2(\vw-\vp)}{\sigma_0 S^3} \left[ A(A + B) \right],
\end{equation*}
and
\begin{equation*}
     \hessp f = \frac{2A}{\sigma_0 S^3} \left[ \frac{2T_{\vp}}{S^2} (\vw-\vp)(\vw-\vp)^\top - (A + B) \bI \right]
\end{equation*}
where $T_{\vp} = (3B(A + B) - S^2)$.
Next, we have
\begin{equation*}
    \hessq I(\vz^*) = - \frac{1}{\pi (1 - \pi)} \left( \gradq \pi \right) \left( \gradq \pi \right)^\top + \E_{W} \left[\Psi(\vw) \left( \gradq f \right) \left( \gradq f \right)^\top + \Omega(\vw) \hessq f \right]
\end{equation*}
with
\begin{equation*}
    \gradq f = \frac{-2(\vw-\vq)}{\sigma_0 S^3} \left[ B(A + B) \right],
\end{equation*}
and
\begin{equation*}
     \hessq f = \frac{2B}{\sigma_0 S^3} \left[ (A + B) \bI - \frac{2T_{\vq}}{S^2} (\vw-\vq)(\vw-\vq)^\top \right]
\end{equation*}
where $T_{\vq} = (3A(A + B) - S^2)$.
Then
\begin{equation*}
    \gradq (\gradp I) (\vz^*)= - \frac{1}{\pi (1 - \pi)} \left( \gradp \pi \right) \left( \gradq \pi \right)^\top + \E_{W} \left[\Psi(\vw) \left( \gradp f \right) \left( \gradq f \right)^\top + \Omega(\vw) \gradq (\gradp f) \right]
\end{equation*}
with
\begin{equation*}
    \gradq (\gradp f) = \frac{4 T_{\vp \vq} }{\sigma_0 S^5} (\vw-\vp)(\vw-\vq)^\top
\end{equation*} 
where $T_{\vp \vq} = 3A^2(A + B) - S^2(2A + B)$,
and by symmetry
\begin{equation*}
    \gradp (\gradq I) = (\gradq (\gradp I))^\top.
\end{equation*}

\section{\textit{Pair Opt-Dist} heuristic} \label{sec:supp-pair_opt_dist_heuristic}
As derived in Sec.~\ref{sec:query_synthesis}, an optimal synthesized query has the following distinct properties.
\begin{itemize}
    \item \textbf{Midpoint Property:} To ensure equiprobable outcomes, the query midpoint $\vb$ must coincide with the posterior mean $\vmu$. 
    \item \textbf{Alignment Property:} To minimize conditional entropy, the normal vector of the query hyperplane $\va$ must align with $\vv_1$, the principal eigenvector of the covariance matrix $\mSigma$.  
    \item \textbf{Magnitude Property}: To maintain high response confidence (avoiding pairs that are too similar or too disparate), the magnitude of the difference must be close to $\tilde{r}$.
\end{itemize}
We can combine these properties into a new heuristic for any candidate pair $(\vp, \vq, \lambda)$ as below
\begin{equation*}
    \eta(\vp,\vq, \lambda) = \sum_{i=1}^{d} \frac{(\vb_i - \vmu_i)^2}{\mSigma_{ii}} + \lambda ||(\vp - \vq) - 2\tilde{r} \vv_1||_2^2.
\end{equation*}
The first term is a Diagonal Mahalanobis Midpoint Penalty. By dividing by the diagonal elements of the covariance matrix ($\mSigma_{ii}$), the algorithm penalizes midpoint deviations in dimensions where we are already highly certain, but forgive deviations in dimensions where the user's preference is still unknown. The second term enforces both the Alignment and Magnitude Properties simultaneously by comparing the candidate difference vector $(\vp - \vq)$ directly to the scaled principal eigenvector $2\tilde{r} \vv_1$. $\lambda$ is a scalar parameter to balance the two properties.

\paragraph{Setting the scalar parameter $\lambda$.} The heuristic has a fundamental unit matching problem and setting $\lambda = 1$ will break the filter. 

The first term (the Midpoint Penalty) divides a squared distance by a variance ($\mSigma_{ii}$), making it a dimensionless quantity. If a candidate midpoint $\vb$ deviates from the mean $\vmu$ by exactly one standard deviation in every dimension, that sum will perfectly equal the dimension $d$. 

The second term (the Directional Penalty) is a raw, unscaled squared Euclidean distance. Because it is not divided by a variance, it will output massive, arbitrary values that completely drown out the Midpoint Penalty.

We need to calibrate $\lambda$ to force both terms to be on the same scale. Because the posterior distribution in our experiments is approximated using MCMC sampling, we can extract the natural scale of the uncertainty essentially for free. The sum of all the marginal variances is the trace of the covariance matrix,
\begin{equation*}
    \Tr(\mSigma) = \sum_{i=1}^{d} \mSigma_{ii}.
\end{equation*}
The trace represents the total expected squared Euclidean distance of a random draw from the posterior. Therefore, to convert the raw Euclidean distance of the second term into a dimensionless scale that matches the first term, we divide by $\Tr(\Sigma)$ and multiply by $d$,
\begin{equation*}
    \lambda = \frac{\zeta \cdot d}{\Tr(\Sigma)}.
\end{equation*}
This automatic calibration does three critical things for the selection strategy.
\begin{itemize}
    \item \textbf{Dimensional Parity:} By multiplying by $d$ and dividing by $\Tr(\mSigma)$, it guarantees that a $1\sigma$ error in the difference vector generates the exact same numerical penalty ($d$) as a $1\sigma$ error in the midpoint vector. The two halves of the heuristic are now perfectly balanced.
    \item \textbf{Dynamic Tightening:} As the MCMC sampler utilizes more pairwise responses and converges, the posterior shrinks and $\Tr(\mSigma)$ will drop rapidly. Because $\Tr(\mSigma)$ is in the denominator, $\lambda$ will automatically grow larger. This dynamically tightens the algorithm's strictness regarding the direction of the query as it becomes more confident in its estimates.
    \item \textbf{Noise Model Accommodation:} Since we are using a heteroskedastic noise model where the variance scales heavily with query distances, enforcing the exact difference magnitude ($2\tilde{r} \vv_1$) might become too restrictive, starving the algorithm of valid candidate pairs. Thus, we use a lower value of $\zeta = 0.1$ so that the filter prioritizes finding the correct midpoint, and is more forgiving about the exact distance between the items.
\end{itemize}

\section{Experiment Details}
In this section, we provide the experimental details and highlight additional experimental results.

\subsection{Synthetic Dataset Experiments}  \label{sec:supp-synth_data_exp}
\subsubsection{Experimental setup}

The synthetic dataset is constructed by creating an embedding of $N$ items in $\R^D$ by drawing samples from $U[-4, 4]^D$ and a user point was drawn from $U[-1, 1]^D$. Responses were obtained for the paired comparison queries selected according to the corresponding AL method and the estimate of the user point was updated. The experiments were conducted for $5$ trials and for every trial, a new user point was sampled. Results for the user point estimation error were averaged over trials for $100$ queries.

\subsubsection{$k$-NN Approx method}

Considering $k$ nearest neighbors of the points gives us access to all the surrounding points that may approximate the optimal pair of points' geometry more closely than the immediate neighbors. In this method, we perform a $k$-NN search for the individual points and then evaluate each of the combination of points for optimality. \\
\\
Let $\mathbb{P} ~=~k$-NN$(\tilde{\vp})$ and $\mathbb{Q} ~=~k$-NN$(\tilde{\vq})$. The points are selected as below:
\begin{equation*}
    (\vp, \vq) = \underset{\vp^* \in \mathbb{P}, ~\vq^* \in \mathbb{Q}}{\argmax} ~I(W;Y | (\vp^*, \vq^*)).
\end{equation*}
This method again uses a filter and refine strategy. The nearest points are computed using Euclidean distances and the top $\beta$ fraction of the pairs undergo MI evaluation to select the best approximation.

\subsubsection{Additional Results}

\paragraph{Synthesis comparison with discrete methods}

\begin{figure}[htbp!]
    \hfill
    \begin{minipage}{0.24\linewidth}
        \centering
        \includegraphics[width=\textwidth]{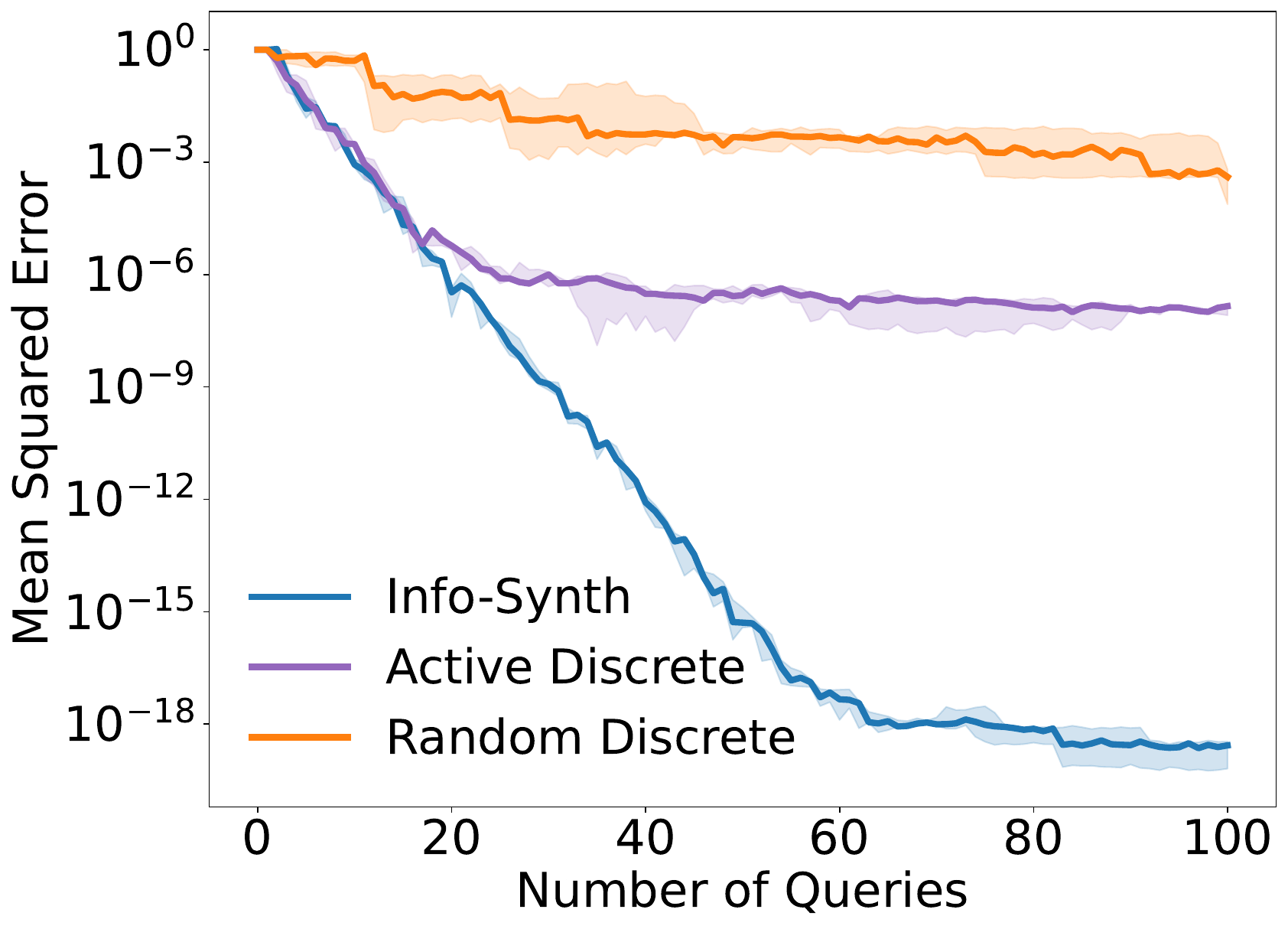}
    \end{minipage}
    \hfill
    \begin{minipage}{0.24\linewidth}
        \centering
        \includegraphics[width=\textwidth]{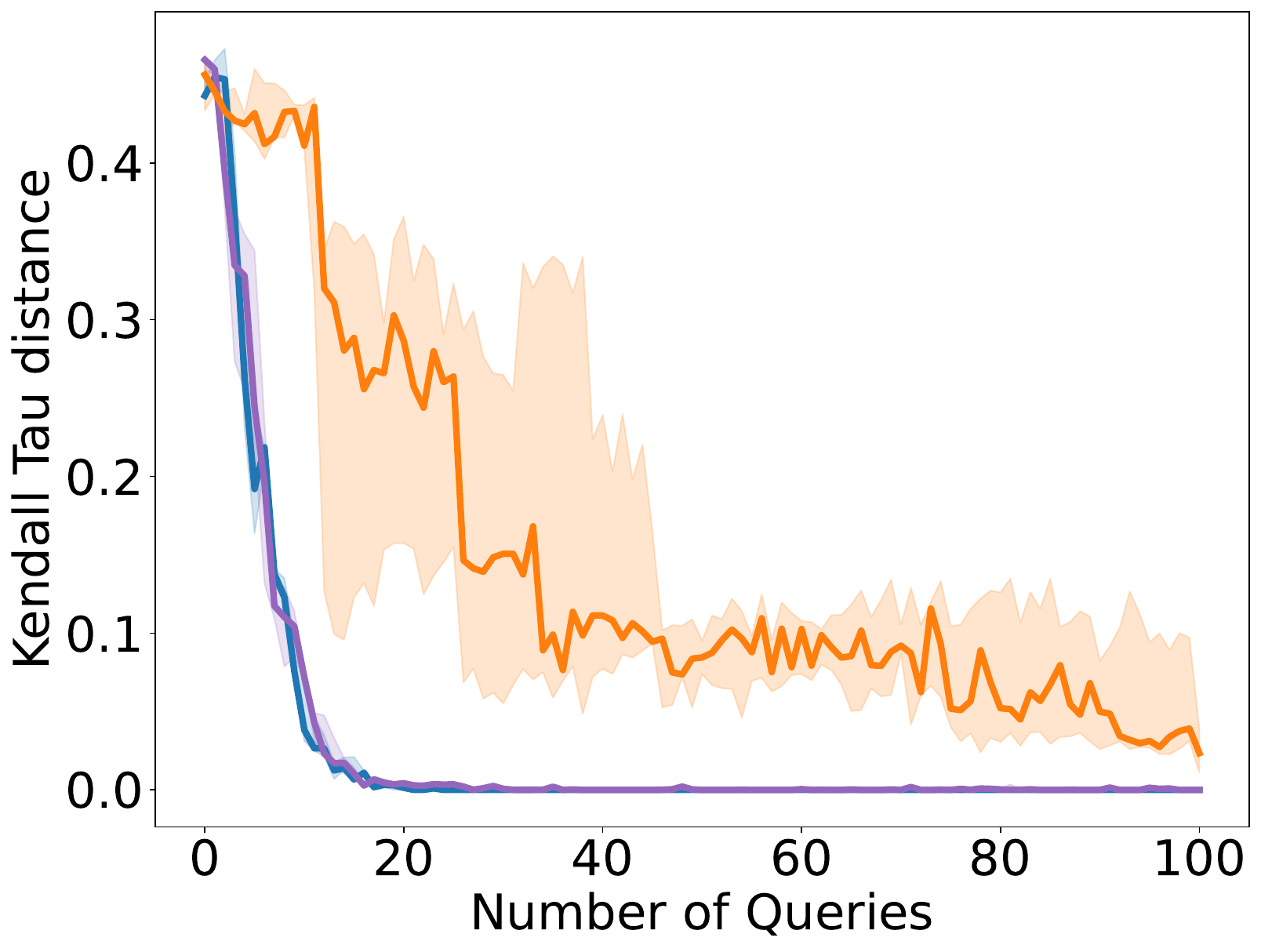}
    \end{minipage}
    \hfill
    \begin{minipage}{0.24\linewidth}
        \centering
        \includegraphics[width=\textwidth]{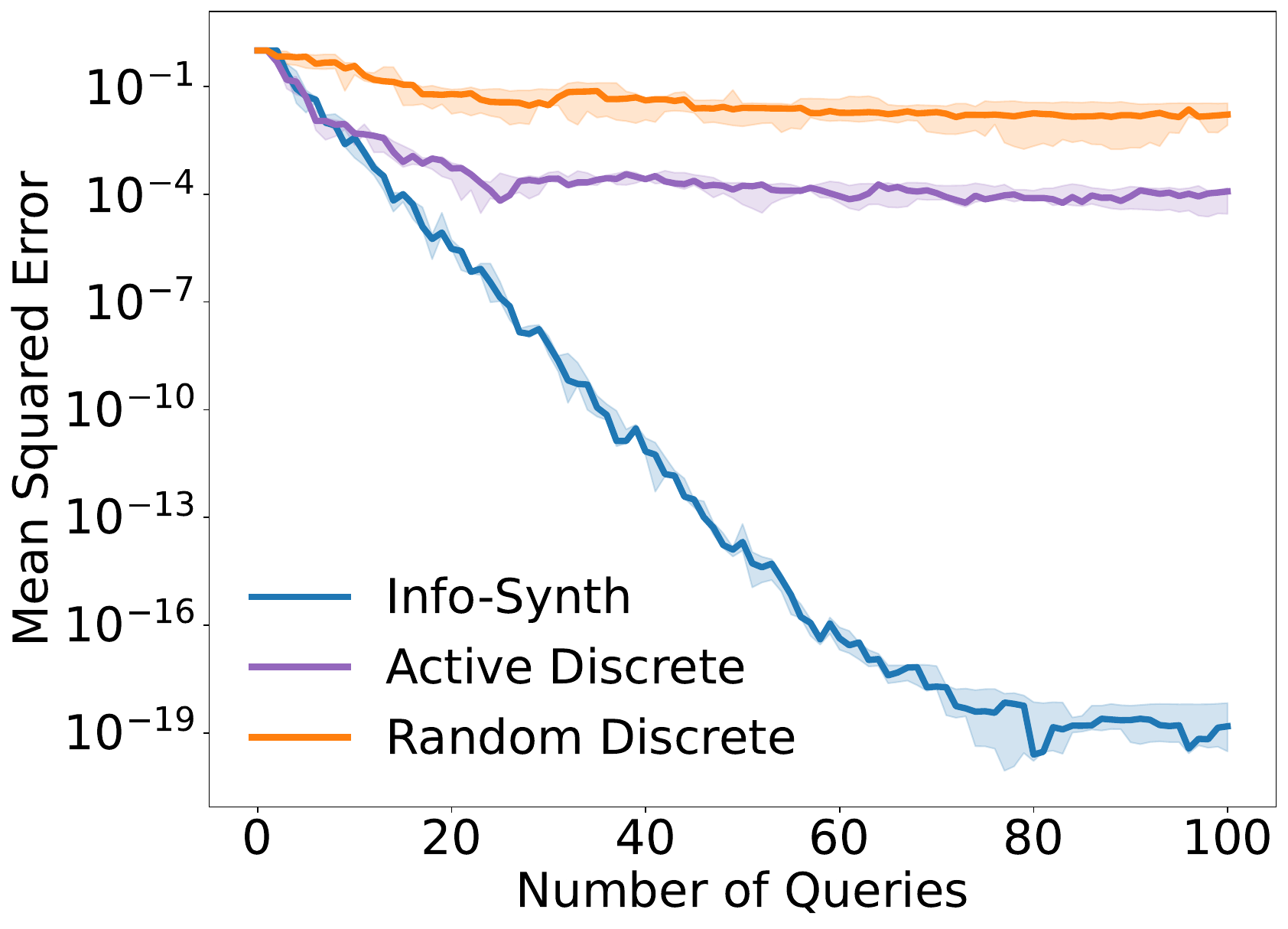}
    \end{minipage}
    \hfill
    \begin{minipage}{0.24\linewidth}
        \centering
        \includegraphics[width=\textwidth]{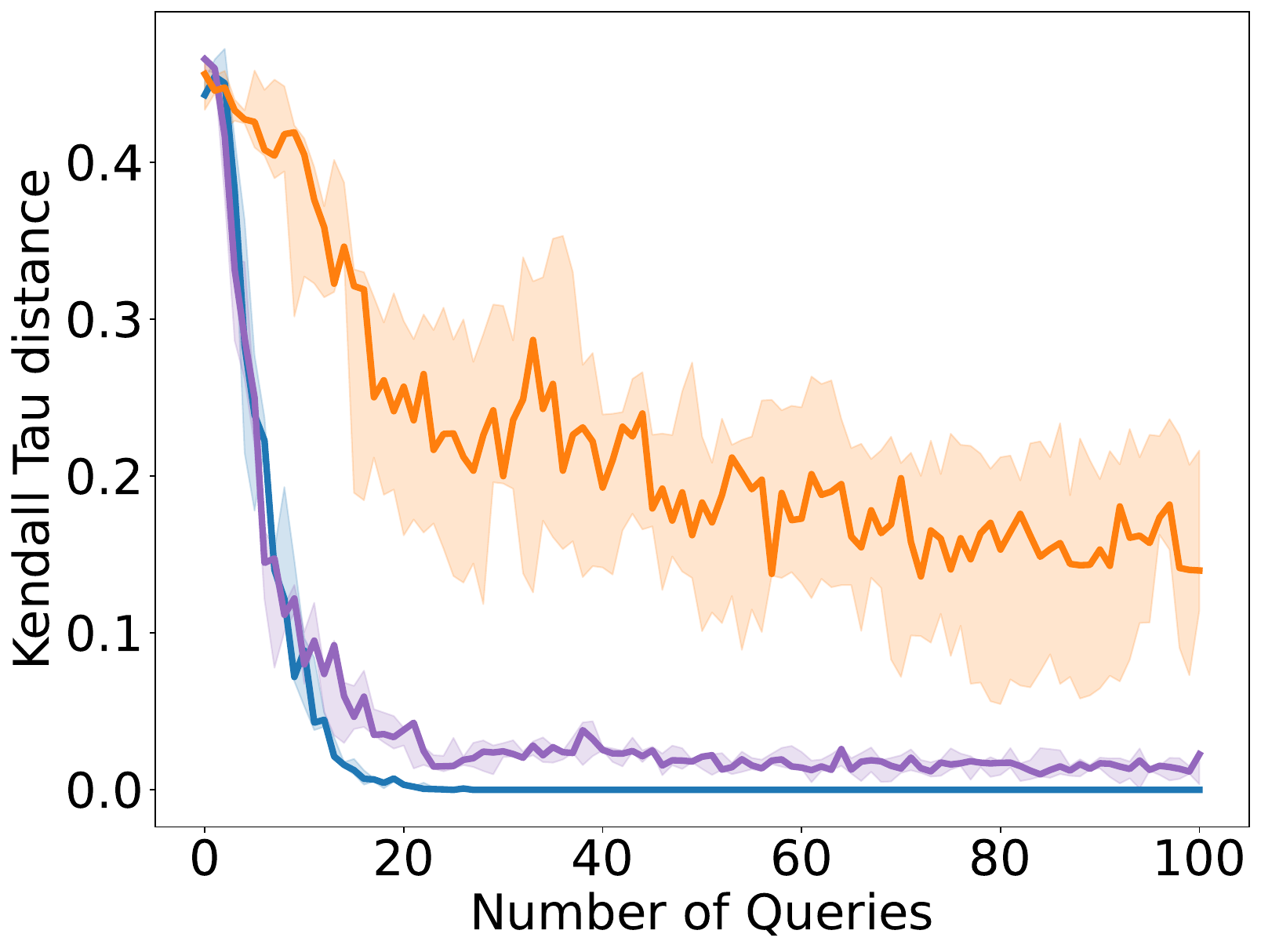}
    \end{minipage}
    \hfill
    \vspace{-2mm}
      \caption{\footnotesize Results for $D=2$ and $N=500$ for $\sigma_0 = 0.001$ (left) and $\sigma_0 = 0.1$ (right).}
      \vspace{-1\baselineskip}
\end{figure}

\begin{figure}[htbp!]
    \hfill
    \begin{minipage}{0.24\linewidth}
        \centering
        \includegraphics[width=\textwidth]{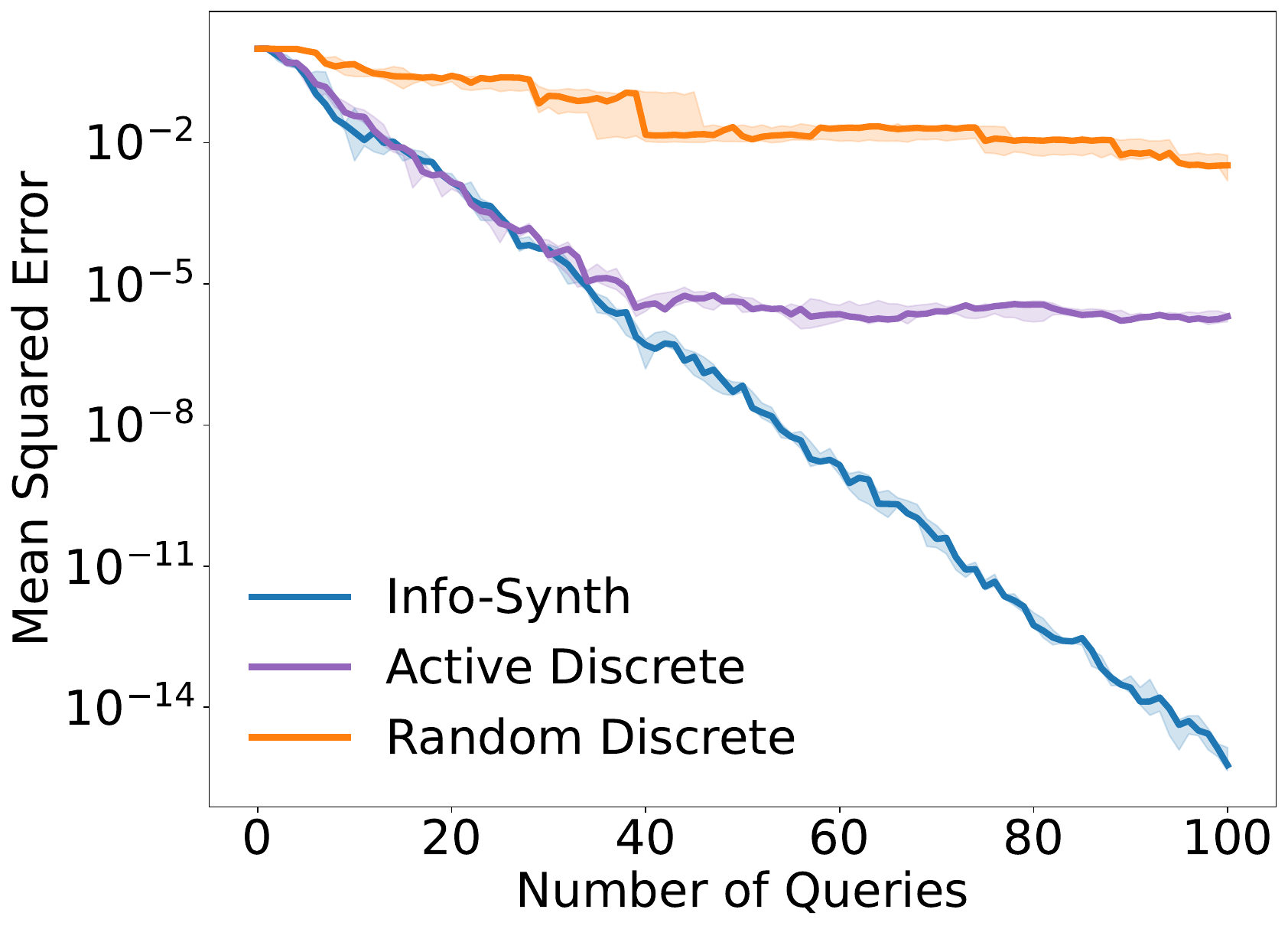}
    \end{minipage}
    \hfill
    \begin{minipage}{0.24\linewidth}
        \centering
        \includegraphics[width=\textwidth]{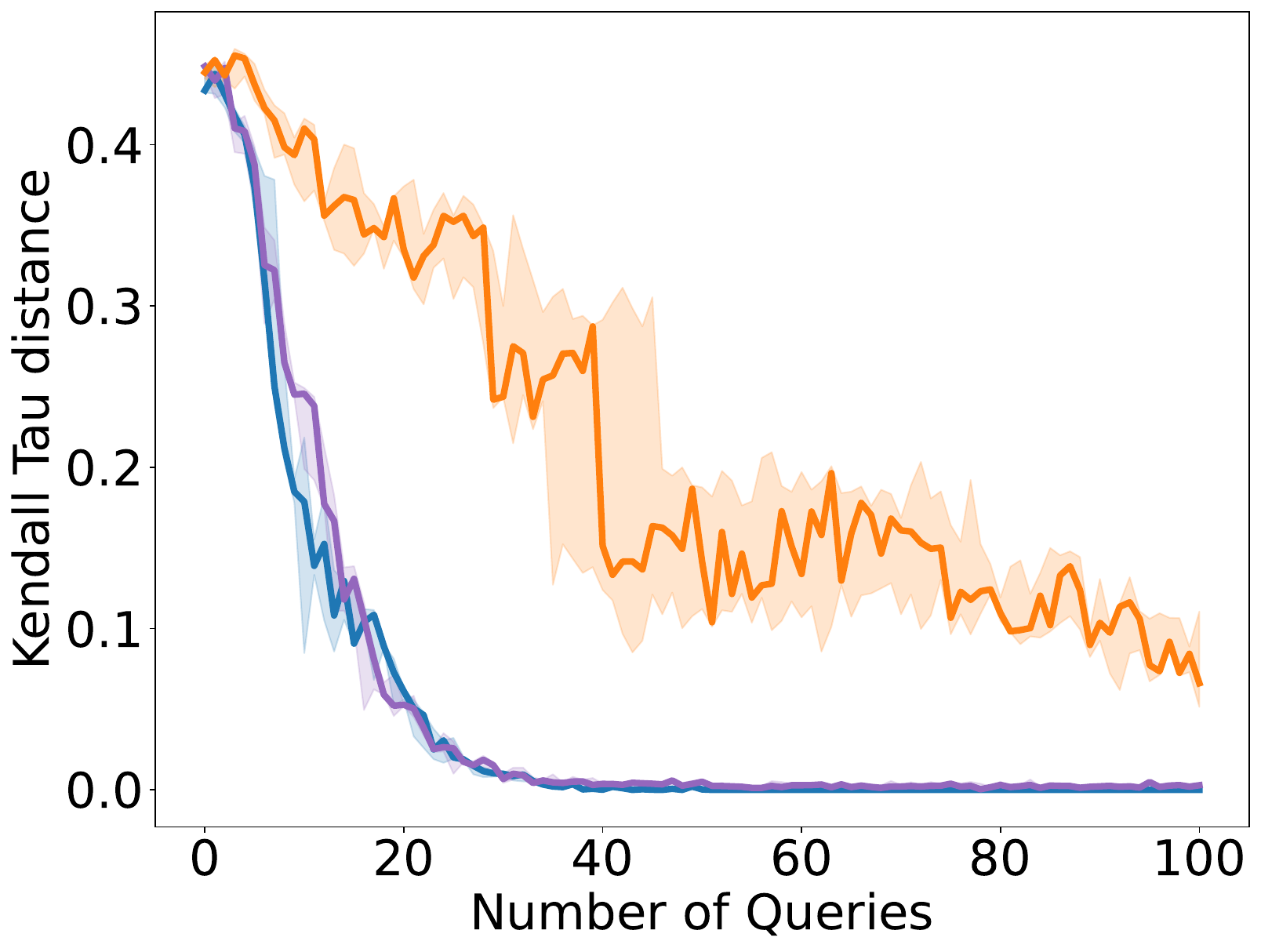}
    \end{minipage}
    \hfill
    \begin{minipage}{0.24\linewidth}
        \centering
        \includegraphics[width=\textwidth]{figures/synth_data_exp/synthesis_vs_discrete/mse_d4_n500_s0.1_r1.pdf}
    \end{minipage}
    \hfill
    \begin{minipage}{0.24\linewidth}
        \centering
        \includegraphics[width=\textwidth]{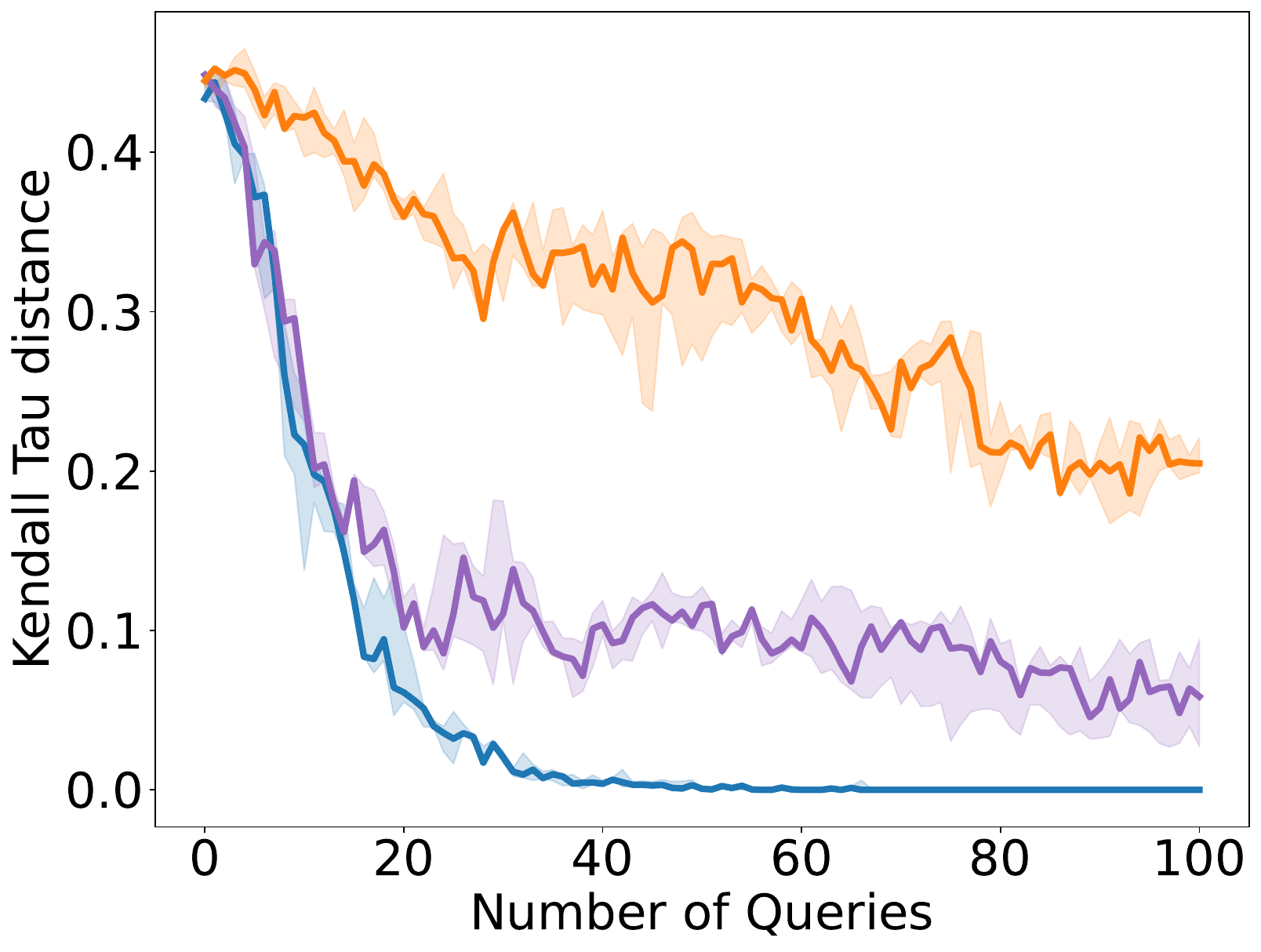}
    \end{minipage}
    \hfill
    \vspace{-2mm}
      \caption{\footnotesize Results for $D=4$ and $N=500$ for $\sigma_0 = 0.001$ (left) and $\sigma_0 = 0.1$ (right).}
\end{figure}

\paragraph{Discrete Comparison}

\begin{figure*}[htbp!]
    \hfill
    \begin{minipage}{0.31\linewidth}
        \centering
        \includegraphics[width=\textwidth]{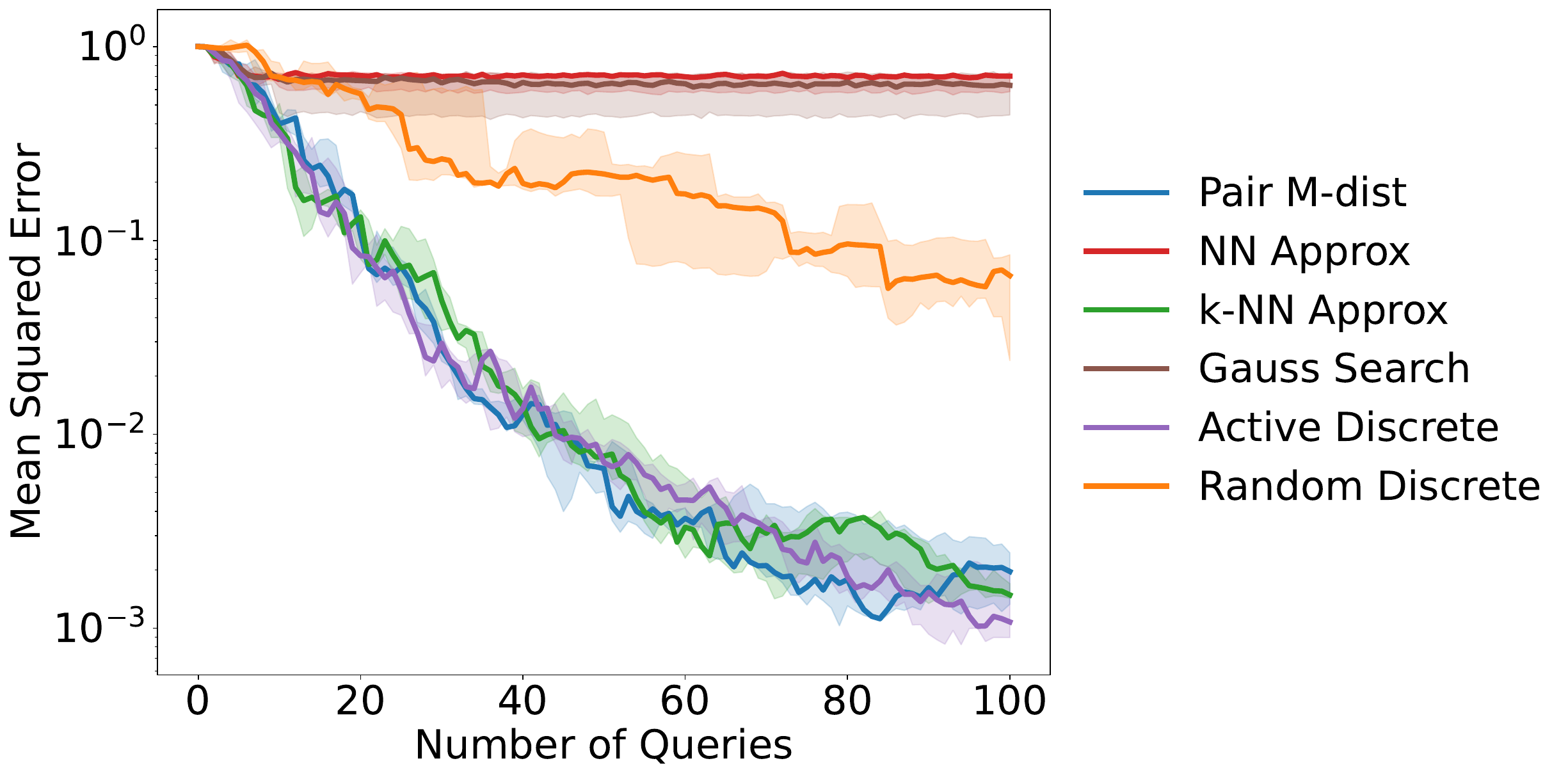}
    \end{minipage}
    \hfill
    \begin{minipage}{0.31\linewidth}
        \centering
        \includegraphics[width=0.8\textwidth]{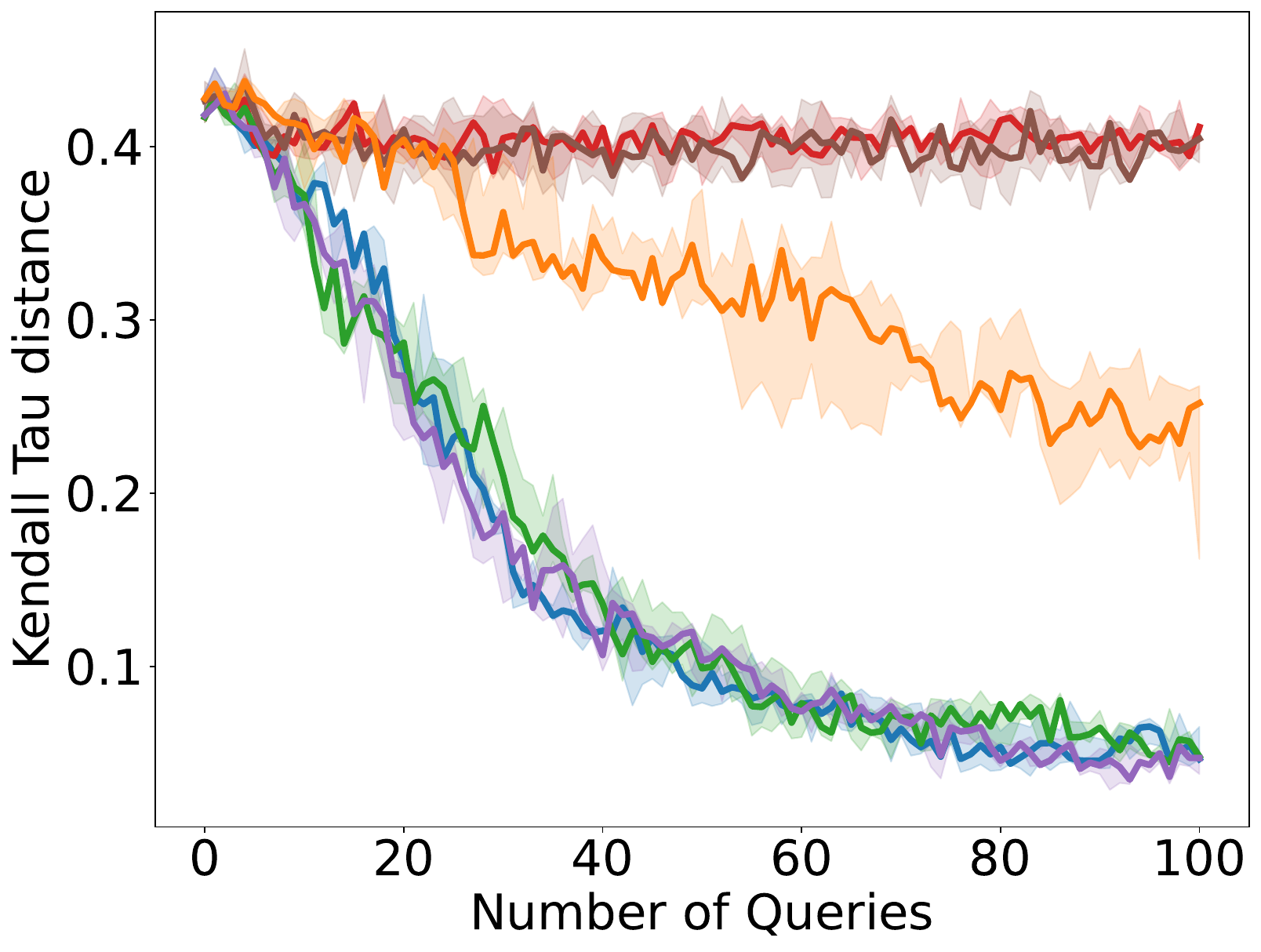}
    \end{minipage}
    \hfill
    \begin{minipage}{0.31\linewidth}
        \centering
        \includegraphics[width=0.8\textwidth]{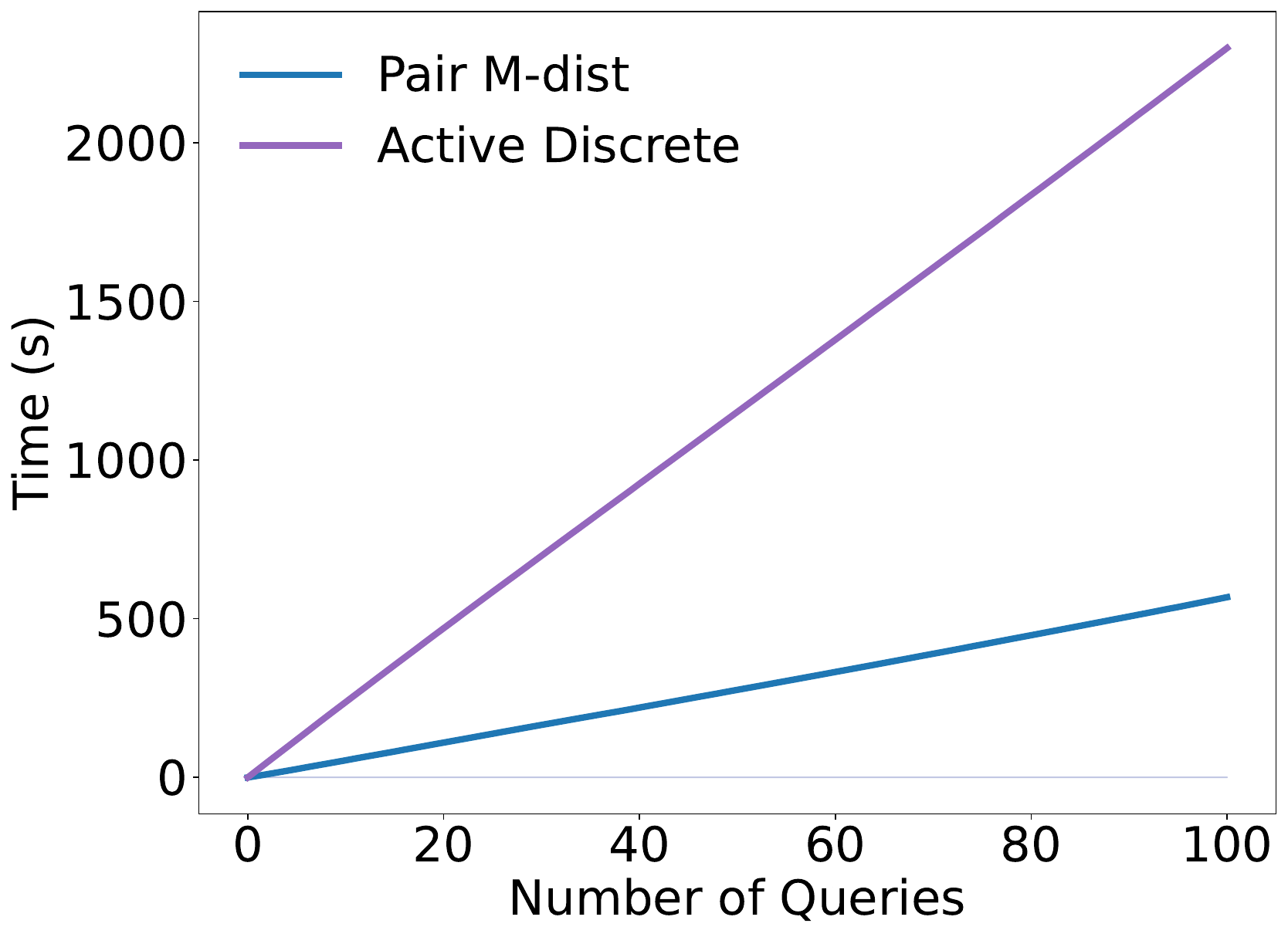}
    \end{minipage}
    \hfill
\caption{\footnotesize Results for $D=10$, $N=500$ and $\sigma_0 = 0.01$. }
\end{figure*}

\begin{figure*}[htbp!]
    \hfill
    \begin{minipage}{0.31\linewidth}
        \centering
        \includegraphics[width=\textwidth]{figures/synth_data_exp/discrete_comparison/mse_d10_n100_s0.01_r1.pdf}
    \end{minipage}
    \hfill
    \begin{minipage}{0.31\linewidth}
        \centering
        \includegraphics[width=0.8\textwidth]{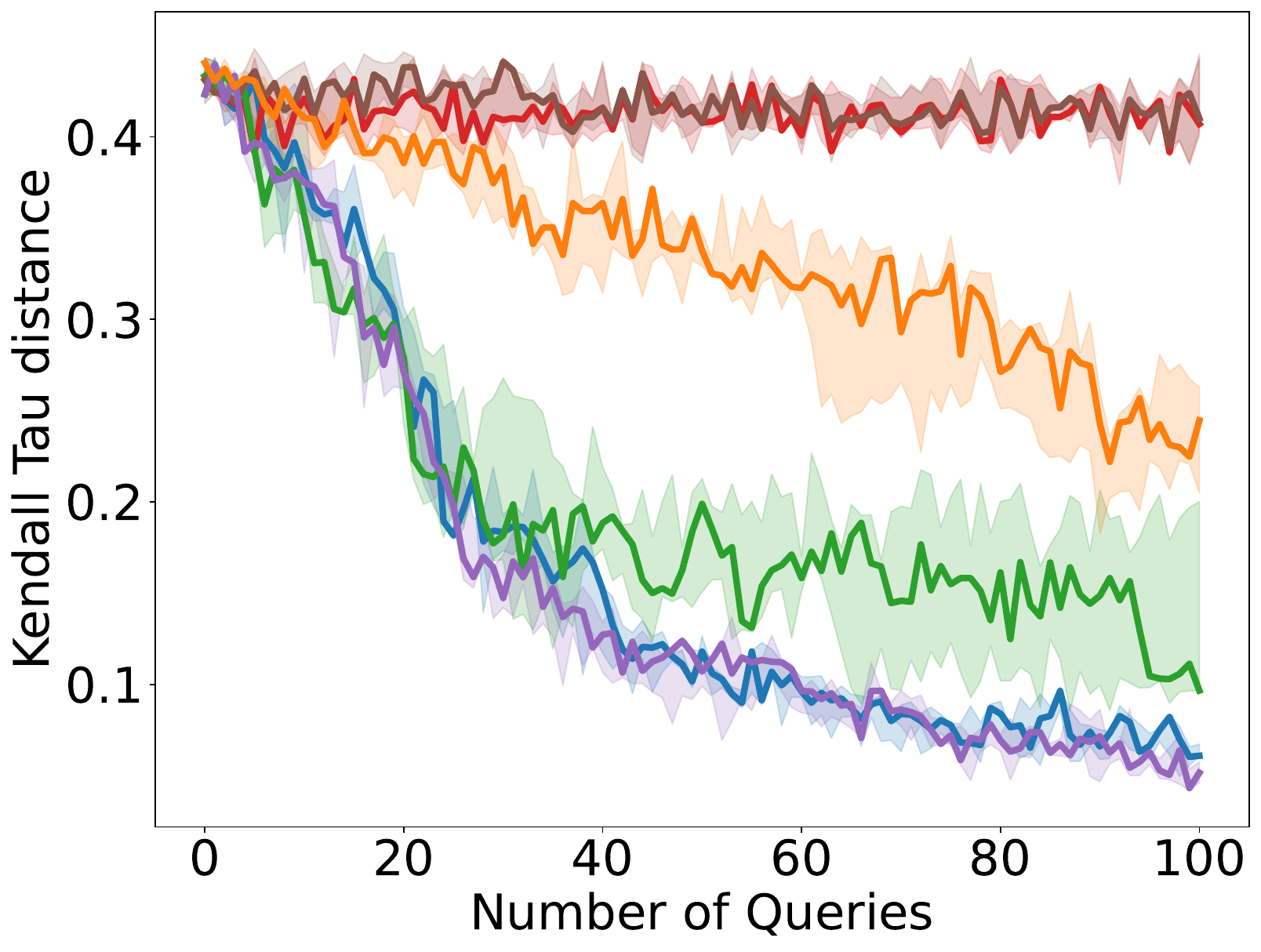}
    \end{minipage}
    \hfill
    \begin{minipage}{0.31\linewidth}
        \centering
        \includegraphics[width=0.8\textwidth]{figures/synth_data_exp/discrete_comparison/time_d10_n100_s0.01_r1.pdf}
    \end{minipage}
    \hfill
\caption{\footnotesize Results for $D=10$, $N=100$ and $\sigma_0 = 0.01$. }
\end{figure*}

\subsection{Reddit Summary Dataset Experiments} \label{sec:supp-reddit_summary_exp}
\subsubsection{Experimental setup}

The chosen user for results in  Figure \ref{fig:reddit_summary_expt} had $1874$ queries in total. For the users available in this dataset, this user had a relatively larger number of paired query responses. The original embedding size for these representations was $512$. To speed up these active learning experiments, particularly for the \emph{Active Discrete} baselines, we used PCA to project embedding dimensions to $128$.  

\subsubsection{Additional Results}

Here, we discuss the results for two additional users in Figure \ref{fig:reddit_summary_expt_user_2_and_3}. The results again show how our method can match the active learning performance of the \emph{Active Discrete} method with lesser computation cost. Decreasing $\gamma$ allows our method to aggressively filter queries for faster per-query computation, seamlessly matching the performance of \textit{Active Discrete} over a slightly longer sequence of queries. The total number of queries available for these two users are $1010$ and $837$ respectively. 

\begin{figure}[!htbp]
    \hfill
    \begin{minipage}{0.24\linewidth}
        \centering
        \includegraphics[width=\textwidth]{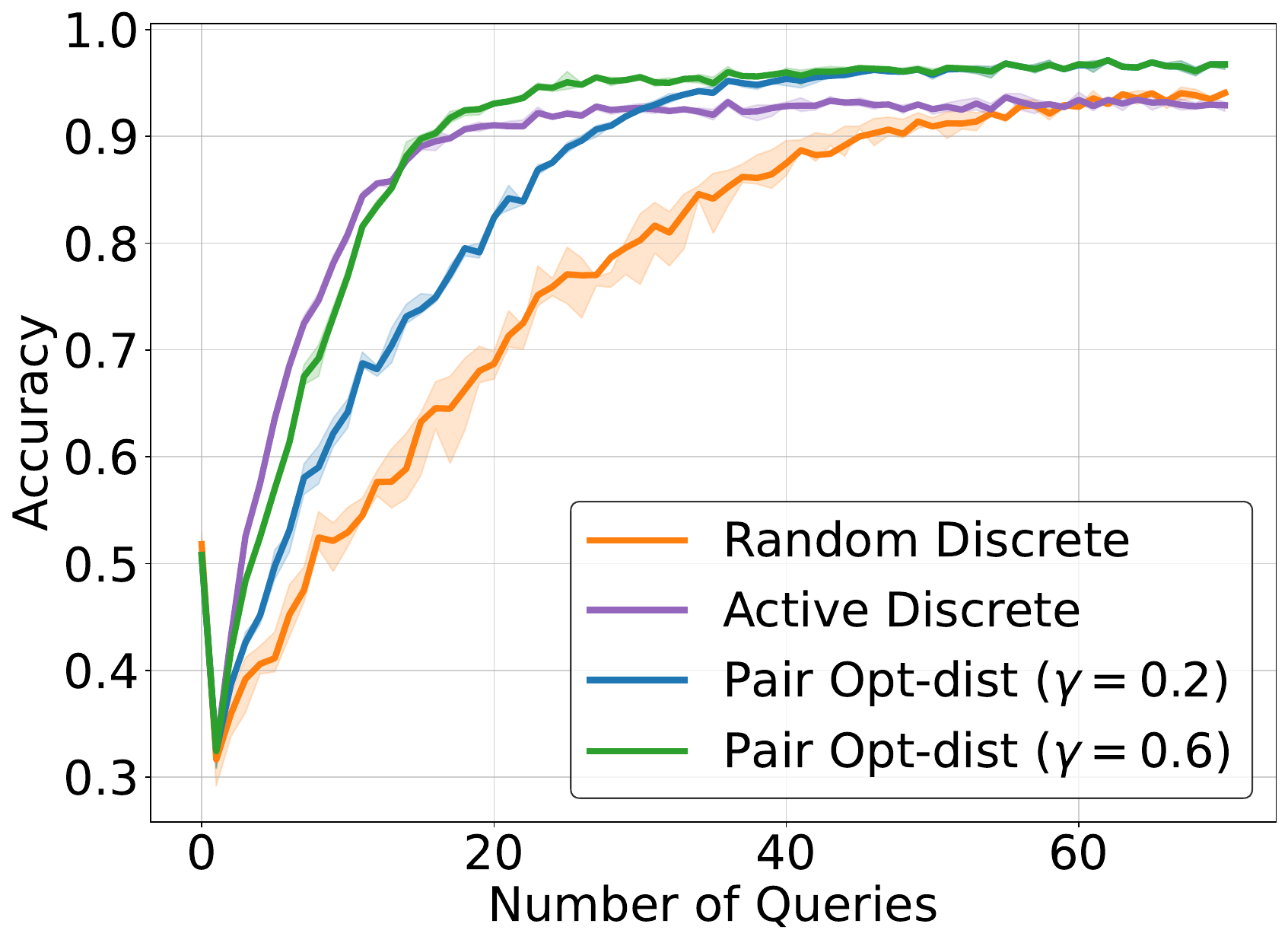}
          \subcaption{\footnotesize User 2}
          \label{fig:reddit-performance_accuracy_comparison_sigma_0p1_user2}
    \end{minipage}
    \hfill
    \begin{minipage}{0.24\linewidth}
        \centering
        \includegraphics[width=\textwidth]{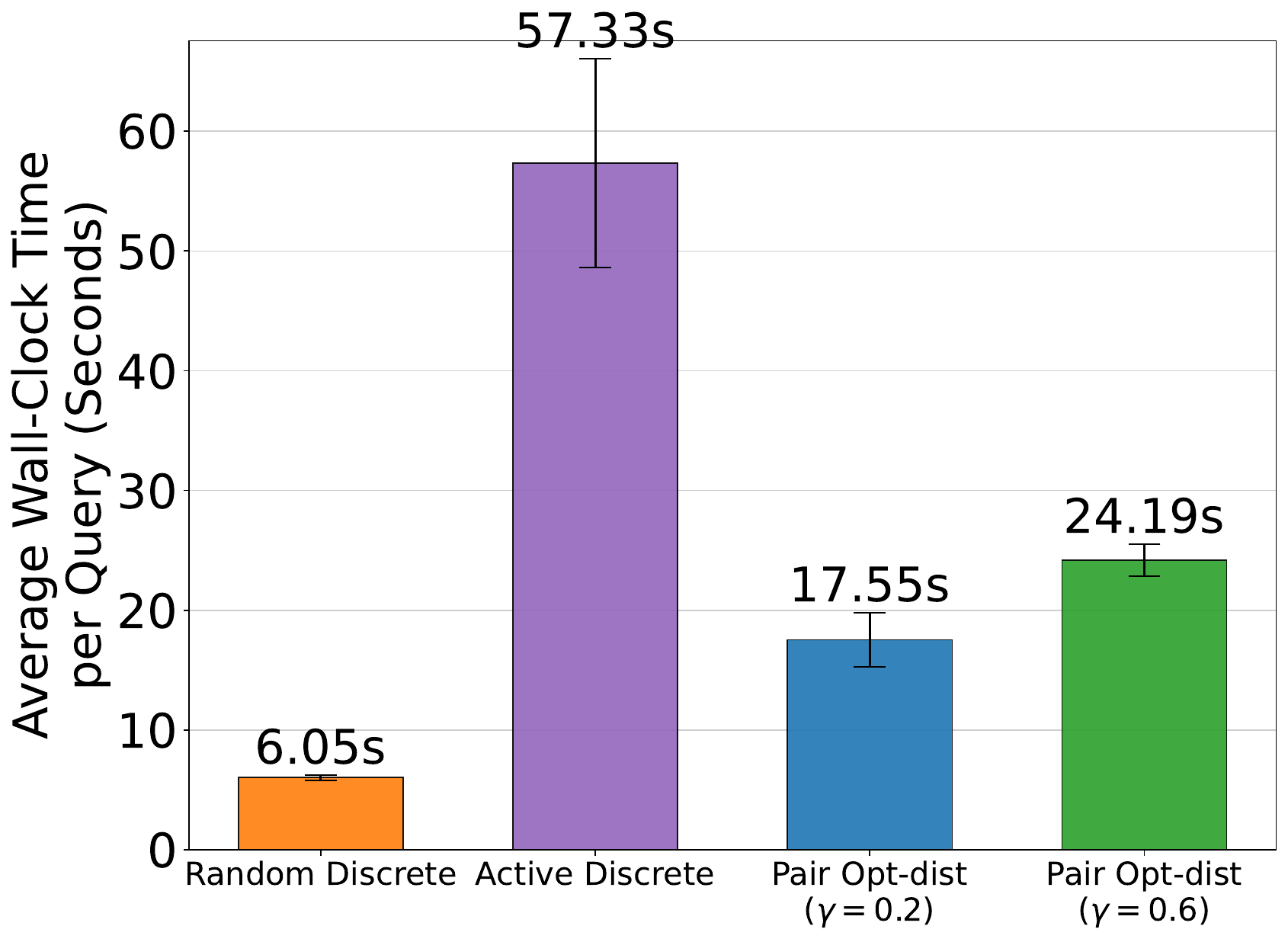}
      \subcaption{\footnotesize User 2}
      \label{fig:reddit-performance_time_comparison_sigma0p1_user2}
    \end{minipage}
    \begin{minipage}{0.24\linewidth}
        \centering
        \includegraphics[width=\textwidth]{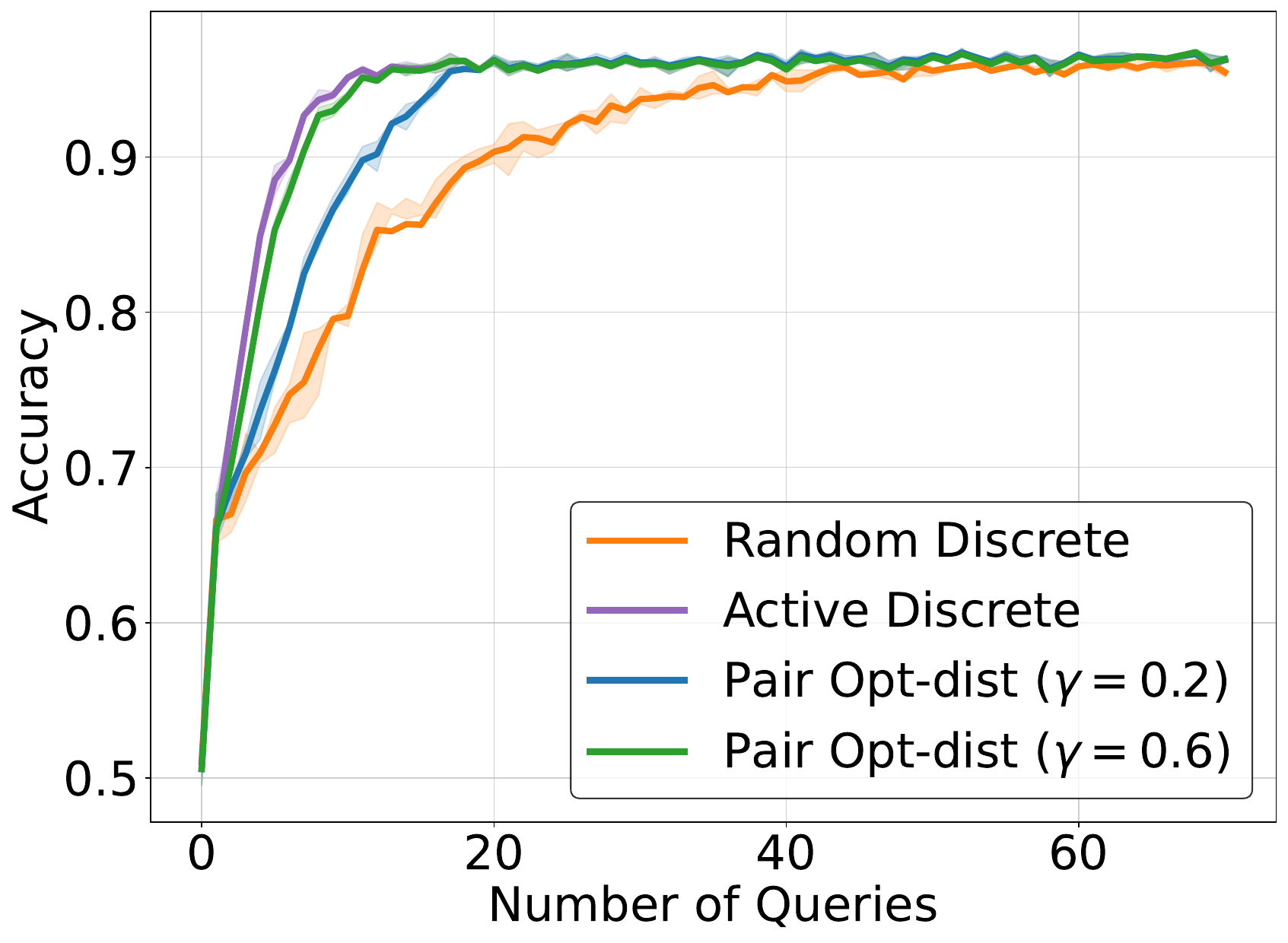}
      \subcaption{\footnotesize User 3}
      \label{fig:reddit-performance_accuracy_comparison_sigma0p3_user3}
    \end{minipage}
    \begin{minipage}{0.24\linewidth}
        \centering
        \includegraphics[width=\textwidth]{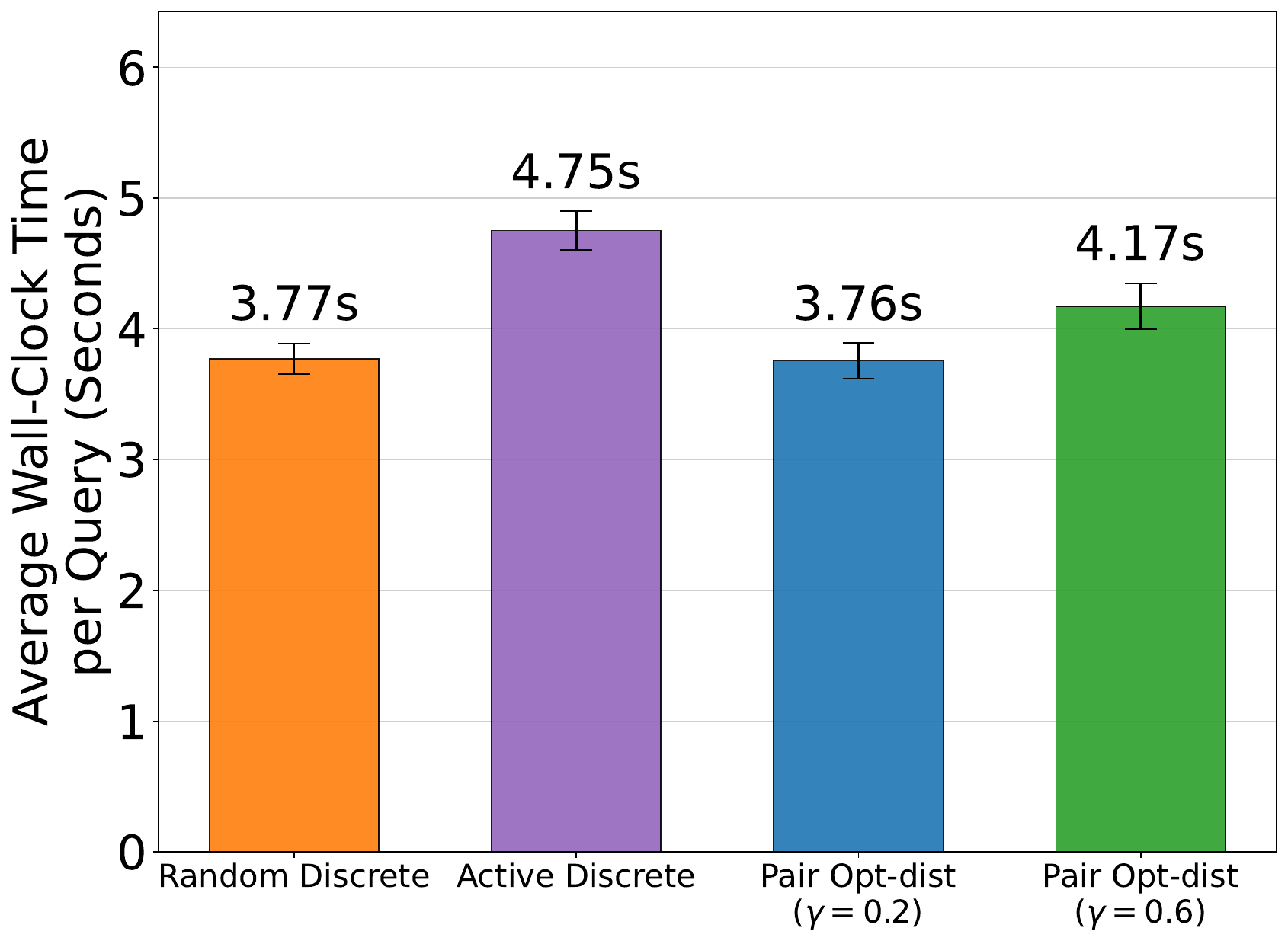}
      \subcaption{\footnotesize User 3}
      \label{fig:reddit-performance_time_comparison_sigma0p3_user3}
    \end{minipage}
    \hfill
    \vspace{-2mm}
      \caption{\footnotesize  Performance analysis on the Reddit Summary TL;DR dataset for two additional users. Our proposed approximation method, \textit{Pair Opt-dist} is shown for two different filtering levels of $\gamma=0.6$ (green)  and $\gamma=0.2$ (blue). Here $\gamma$ represents total fraction of queries used for selection. (a) and (b) show the accuracy  and average query selection time at  $\sigma=0.1$ for user $2$ while (c) and (d) show these results for user $3$.}
      \label{fig:reddit_summary_expt_user_2_and_3}
      \vspace{-1\baselineskip}
\end{figure}

\subsection{Gain Tuning} \label{sec:supp-gain_tuning_exp}
\subsubsection{Controller Design and Equations}

The experiments involve a trajectory-tracking controller applied to a unicycle model with three parameters $k_x, k_y, k_\theta$. The unicycle system is defined by its state and the hyperparameters governing its response. Let $\mathbf{q}(t) = [x, y, \theta]^\top \in \mathbb{R}^2 \times \mathbb{S}^1$ denote the state of the unicycle at time $t$ in the global frame. The dynamics are driven by a control input vector $\mathbf{u} = [v, \omega]^\top$, where $v$ (linear velocity) and $\omega$ (angular velocity) are scalars defined in the robot's local frame.

We define the global position error as $\Delta x = x_{des} - x$ and $\Delta y = y_{des} - y$. These are transformed into the robot's local longitudinal and lateral errors as follows

\begin{itemize}
    \item \textbf{Longitudinal Error ($e_x$):} Represents displacement along the robot's current heading.
    \begin{equation*}
        e_x = \cos(\theta) \Delta x + \sin(\theta) \Delta y
    \end{equation*}
    \item \textbf{Lateral Error ($e_y$):} Represents perpendicular displacement from the robot's current heading.
    \begin{equation*}
        e_y = -\sin(\theta) \Delta x + \cos(\theta) \Delta y
    \end{equation*}
    \item \textbf{Heading Error ($e_{\theta}$):} Represents the angular misalignment with the trajectory tangent.
    \begin{equation*}
        e_{\theta} = (\theta_{des} - \theta) \pmod{2\pi}
    \end{equation*}
\end{itemize}
The control law translates these local errors into velocity and angular velocity commands to ensure the robot stays on the intended path. The velocity command, $v_{cmd}$, is calculated by
\begin{equation*}
    v_{cmd} = v_{des} \cos(e_{\theta}) + k_x e_x
\end{equation*}
where $v_{des}$ represents the desired trajectory velocity and $k_x$ is the longitudinal gain. Similarly, the angular velocity command, $\omega_{cmd}$, is defined as
\begin{equation*}
    \omega_{cmd} = \omega_{des} + k_y v_{des} e_y + k_{\theta} \sin(e_{\theta})
\end{equation*}
where $\omega_{des}$ is the desired curvature and $k_y, k_{\theta}$ are the lateral and heading gains respectively \cite{kanayama1990stable}.

\subsubsection{B\'{e}zier Curve Trajectories}

The reference paths are generated using B\'{e}zier curves, providing smooth, continuous trajectories defined by a set of control points. A B\'{e}zier curve is defined by $n+1$ control points, and the robot's target position at any phase $t \in [0, 1]$ is computed using the De Casteljau algorithm \cite{farouki2012bernstein}. Beyond simple position, the evaluation logic calculates the first derivative for velocity and the second derivative for acceleration. These derivatives are essential for determining the desired velocity ($v_{des}$) and curvature ($\omega_{des}$) required for the feedforward component of the controller.

\subsubsection{Experimental setup}

The objective of the experiments is for the unicycle to smoothly track the trajectory. As with all active preference learning, we seek to autonomously identify the optimal control parameters that accomplish this objective in as few experiments as possible, relying only on comparative feedback (``A is better than B") rather than numerical gradient values.

We simulate a human-in-the-loop tuning process where the learner maintains a probabilistic belief about where the optimal parameters lie in the log-parameter space. The active learning framework operates within a logarithmic parameter space to improve the efficiency and stability of the estimation process. By representing the parameters as $W = \log(H)$, where $H$ is the parameter vector in linear space, the estimator can more effectively handle gain values spanning several orders of magnitude. This ensures that multiplicative changes in the physical gains are treated as additive changes in the search space, which significantly benefits the MCMC sampling process conducted via the Stan probabilistic programming language \cite{carpenter2017stan}. Furthermore, log-space parameterization often results in a more Gaussian-like posterior distribution, allowing the optimizer to identify informative query pairs more reliably.

We use different acquisition functions to select the next query pair - two different sets of gains ($K_A, K_B$) that will yield the most information. For each query pair, the system simulates a unicycle robot tracking a specific path starting from various initial states. Instead of a real human, we use an automated oracle that evaluates the simulated trajectories based on the trajectory tracking error (computed using Hausdorff distance). To simulate real human inconsistency, the oracle uses a sigmoid function to generate a preference probability: if Trajectory A is much better than B, the probability of picking A is high (near 1.0); if they are similar, it approaches a coin flip (0.5). To ensure the learned gains are robust, the oracle averages the error across multiple trajectories or multiple initial states before deciding which set of gains is superior.

\subsection{Trajectory Scenarios and Initial Conditions}
The performance of the tracking controller and the convergence of the active learning estimator are evaluated across several trajectory scenarios. These paths are generated using B\'{e}zier curves, defined by a set of control points $\mathbf{P} = \{P_0, P_1, \dots, P_n\}$.

\begin{figure}[htbp]
     \centering
     \begin{subfigure}[b]{0.24\textwidth}
         \centering
         \includegraphics[width=\textwidth]{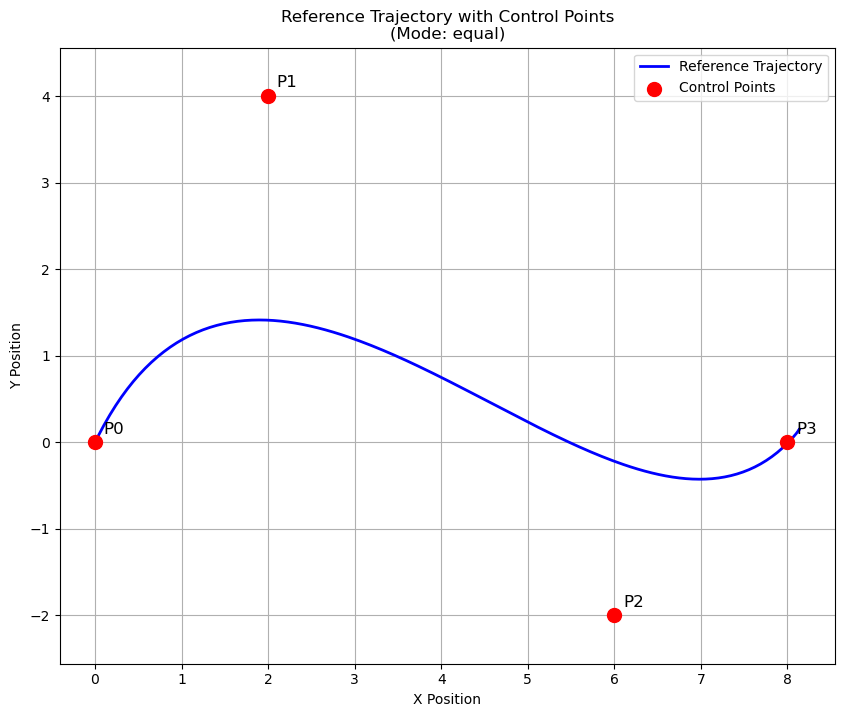}
         \caption{Standard Path}
     \end{subfigure}
     \hfill
     \begin{subfigure}[b]{0.24\textwidth}
         \centering
         \includegraphics[width=\textwidth]{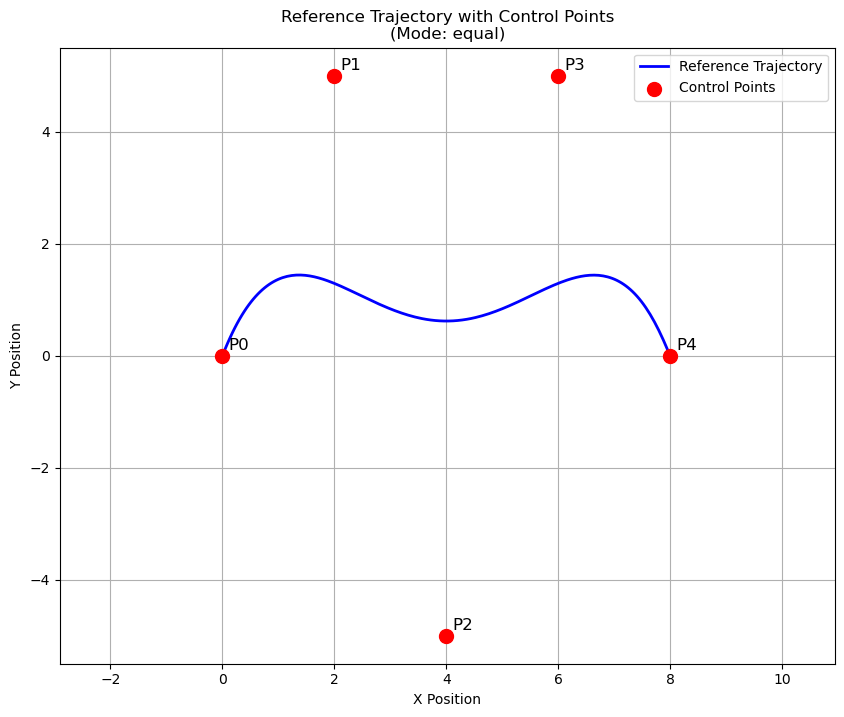}
         \caption{Hairpin Path}
     \end{subfigure}
     \hfill
     \begin{subfigure}[b]{0.24\textwidth}
         \centering
         \includegraphics[width=\textwidth]{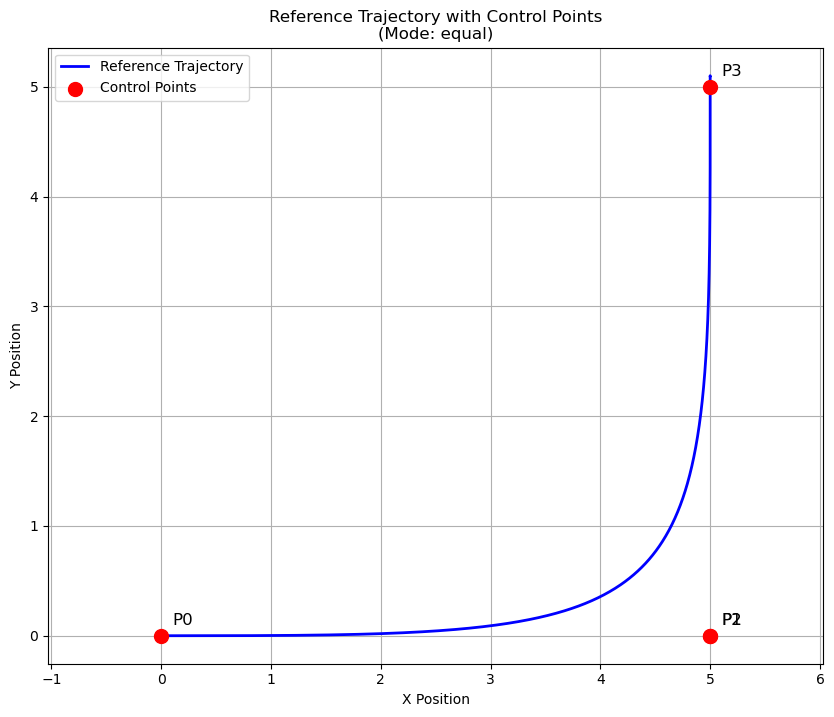}
         \caption{High Curvature}
     \end{subfigure}
     \hfill
     \begin{subfigure}[b]{0.24\textwidth}
         \centering
         \includegraphics[width=\textwidth]{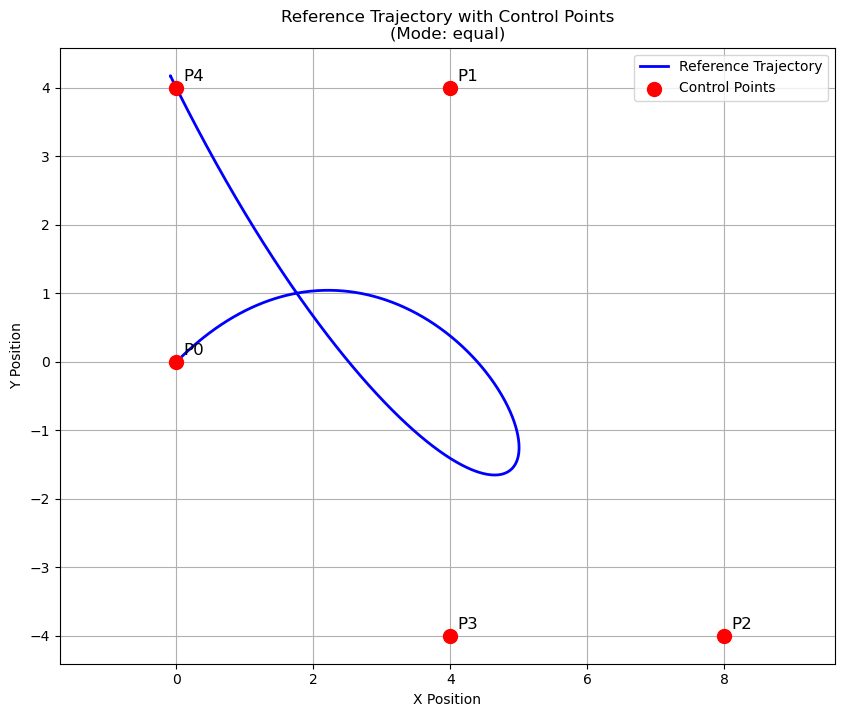}
         \caption{Figure-8 Path}
     \end{subfigure}
     \caption{\small Different trajectories considered in the experiments}
     \label{fig:diff_trajectories}
\end{figure}

\paragraph{Path Geometry.}
The following trajectories (illustrated in Fig.~\ref{fig:diff_trajectories}) were implemented to test various aspects of the unicycle dynamics:
\begin{itemize}
    \item \textbf{Trajectory 1 (Standard):} A basic Bézier curve defined by control points $\mathbf{P} = [[0,0], [2,4], [6,-2], [8,0]]$. This serves as a baseline for smooth, multi-directional tracking.
    \item \textbf{Trajectory 2 (Hairpin/Aggressive):} A high-amplitude path designed to force sharp, fast turns. With control points $\mathbf{P} = [[0, 0], [2, 5], [4, -5], [6, 5], [8, 0]]$, this scenario makes robot dynamics more observable as it challenges the lateral correction limits of the controller.
    \item \textbf{Trajectory 3 (High Curvature):} A sharp turn trajectory defined by $\mathbf{P} = [[0, 0], [5, 0], [5, 0], [5, 5]]$. This path tests the controller's ability to handle sudden changes in heading while maintaining velocity.
    \item \textbf{Trajectory 4 (Figure-8):} A complex shape defined by $\mathbf{P} = [[0, 0], [4, 4], [8, -4], [4, -4], [0, 4]]$. This trajectory provides a continuous loop with alternating lateral stresses, testing the symmetry and consistency of the learned gains.
\end{itemize}

\paragraph{Starting Scenarios and Initial Errors.}
The experiments also explore how the robot handles different initial states relative to these paths. These scenarios test the controller's recovery performance.
\begin{itemize}
    \item \textbf{Perfect Start:} The robot begins exactly at the first control point ($P_0$) with a heading perfectly aligned to the path tangent.
    \item \textbf{Lateral Error:} The robot starts 1 unit to the left of the path, forcing the lateral gain ($k_{lat}$) to immediately correct the side-slip.
    \item \textbf{Heading Error:} The robot begins at the correct position but with a 45-degree misalignment, testing the heading gain's ability to steer the robot back onto the trajectory.
\end{itemize}

\paragraph{Simulation Parameters.} 
Each trajectory is executed within a specific simulation window.
\begin{itemize}
    \item \textbf{Simulation Period ($T_{period}$):} Defines the time the robot has to complete the curve phase ($t \in [0, 1]$).
    \item \textbf{Total Final Time ($T_{final}$):} Determines the total duration the robot is allowed to run.
    \item \textbf{Time Step ($dt$):} Set consistently at 0.02 seconds for numerical integration stability.
\end{itemize}

\end{document}